\documentclass{article}

\newif\ifappendix
\appendixtrue
\newcommand{\appendixref}[1]{\ifappendix \autoref{#1}\else the appendix\fi}

\usepackage[margin=1.15in]{geometry}
\usepackage{palatino}
\usepackage[longnamesfirst]{natbib}

\usepackage[utf8]{inputenc} 
\usepackage[T1]{fontenc}    
\usepackage[pagebackref,colorlinks,linkcolor=purple,citecolor=teal,urlcolor=violet]{hyperref}       
\usepackage{url}            
\usepackage{booktabs}       
\usepackage{amsfonts}       
\usepackage{nicefrac}       
\usepackage{microtype}      
\usepackage{xcolor}         

\usepackage{amssymb}
\usepackage{amsmath,amsthm,thmtools,thm-restate}
\usepackage{cleveref}
\usepackage{doi}
\usepackage[shortlabels]{enumitem}
\SetLabelAlign{parright}{\parbox[t]{\labelwidth}{\raggedleft #1}}
\usepackage{graphicx}
\usepackage{mathtools}
\usepackage{mismath}

\usepackage{siunitx}

\usepackage{titlesec}
\titleformat*{\subparagraph}{\itshape}

\usepackage{tocloft}
\usepackage{parskip}

\allowdisplaybreaks[1]
\renewcommand{\angle}{\measuredangle}
\renewcommand{\vec}[1]{\boldsymbol{#1}}
\newcommand{\iid}{\stackrel{\text{i.i.d.}}{\scalebox{1.75}[1]{$\sim$}}}
\newcommand{\ppp}{{\:\!\!+\:\!\!}}
\newcommand{\mmm}{{\:\!\!-\:\!\!}}
\DeclareMathOperator*{\argmin}{arg\,min}
\DeclareMathOperator*{\clos}{cl}
\DeclareMathOperator*{\cone}{cone}
\DeclareMathOperator*{\intr}{int}
\DeclareMathOperator*{\lspn}{span}

\declaretheorem{assumption}
\declaretheorem{proposition}
\declaretheorem[sibling=proposition]{lemma}
\declaretheorem[sibling=proposition]{theorem}
\declaretheorem[sibling=proposition]{corollary}
\declaretheorem[sibling=proposition,style=remark]{claim}
\declaretheorem[sibling=proposition,style=remark]{remark}
\declaretheorem[sibling=proposition,style=remark,qed=\qedsymbol]{example}

\title{Learning a Neuron by a Shallow ReLU Network: \\ Dynamics and Implicit Bias for Correlated Inputs}

\author{
Dmitry Chistikov\thanks{Equal contribution.} \\
University of Warwick \\
\texttt{d.chistikov@warwick.ac.uk}
\and
Matthias Englert\footnotemark[\value{footnote}] \\
University of Warwick \\
\texttt{m.englert@warwick.ac.uk}
\and
Ranko Lazi\'c\footnotemark[\value{footnote}] \\
University of Warwick \\
\texttt{r.s.lazic@warwick.ac.uk}
}

\date{}

\begin{document}

\maketitle

\begin{abstract}
We prove that, for the fundamental regression task of learning a single neuron, training a one-hidden layer ReLU network of any width by gradient flow from a small initialisation converges to zero loss and is implicitly biased to minimise the rank of network parameters.  By assuming that the training points are correlated with the teacher neuron, we complement previous work that considered orthogonal datasets.  Our results are based on a detailed non-asymptotic analysis of the dynamics of each hidden neuron throughout the training.  We also show and characterise a surprising distinction in this setting between interpolator networks of minimal rank and those of minimal Euclidean norm.  Finally we perform a range of numerical experiments, which corroborate our theoretical findings.
\end{abstract}

\section{Introduction}

One of the grand challenges for machine learning research is to understand how overparameterised neural networks are able to fit perfectly the training examples and simultaneously to generalise well to unseen data~\citep{ZhangBHRV21}.  The double-descent phenomenon~\citep{BelkinHMM19}, where increasing the neural network capacity beyond the interpolation threshold can eventually reduce the test loss much further than could be achieved around the underparameterised ``sweet spot'', is a mystery from the standpoint of classical machine learning theory.  This has been observed to happen even for training without explicit regularisers.

\paragraph{Implicit bias of gradient-based algorithms.}

A key hypothesis towards explaining the double-descent phenomenon is that the gradient-based algorithms that are used for training are \emph{implicitly biased} (or~\emph{implicitly regularised})~\citep{NeyshaburBMS17} to converge to solutions that in addition to fitting the training examples have certain properties which cause them to generalise well.  It has attracted much attention in recent years from the research community, which has made substantial progress in uncovering implicit biases of training algorithms in many important settings~\citep{Vardi23c}.  For example, for classification tasks, and for homogeneous networks (which is a wide class that includes ReLU networks provided they contain neither biases at levels deeper than the first nor residual connections), \citet{LyuL20} and \citet{JiT20} established that gradient flow is biased towards maximising the classification margin in parameter space, in the sense that once the training loss gets sufficiently small, the direction of the parameters subsequently converges to a Karush-Kuhn-Tucker point of the margin maximisation problem.

Insights gained in this foundational research direction have not only shed light on overparameterised generalisation, but have been applied to tackle other central problems, such as the susceptibility of networks trained by gradient-based algorithms to adversarial examples~\citep{VardiYS22} and the possibility of extracting training data from network parameters~\citep{HaimVYSI22}.

\paragraph{Regression tasks and initialisation scale.}

Showing the implicit bias for regression tasks, where the loss function is commonly mean square, has turned out to be more challenging than for classification tasks, where loss functions typically have exponential tails.  A major difference is that, whereas most of the results for classification do not depend on how the network parameters are initialised, the scale of the initialisation has been observed to affect decisively the implicit bias of gradient-based algorithms for regression~\citep{WoodworthGLMSGS20}.  When it is large so that the training follows the \emph{lazy regime}, we tend to have fast convergence to a global minimum of the loss, however without an implicit bias towards sparsity and with limited generalisation~\citep{JacotGSHG21}.  The focus, albeit at the price of uncertain convergence and lengthier training, has therefore been on the \emph{rich regime} where the initialisation scale is small.

Considerable advances have been achieved for linear networks.  For example, \citet{AzulayMNWSGS21} and \citet{YunKM21} proved that gradient flow is biased to minimise the Euclidean norm of the predictor for one-hidden layer linear networks with infinitesimally small initialisation, and that the same holds also for deeper linear networks under an additional assumption on their initialisation.  A related extensive line of work is on implicit bias of gradient-based algorithms for matrix factorisation and reconstruction, which has been a fruitful test-bed for regression using multi-layer networks.  For example, \citet{GunasekarWBNS17} proved that, under a commutativity restriction and starting from a small initialisation, gradient flow is biased to minimise the nuclear norm of the solution matrix; they also conjectured that the restriction can be dropped, which after a number of subsequent works was refuted by \citet{LiLL21}, leading to a detailed analysis of both underparameterised and overparameterised regimes by \cite{JinLLDL23}.

For non-linear networks, such as those with the popular ReLU activation, progress has been difficult.  Indeed, \citet{VardiS21} showed that precisely characterising the implicit bias via a non-trivial regularisation function is impossible already for single-neuron one-hidden layer ReLU networks, and \citet{TimorVS23} showed that gradient flow is not biased towards low-rank parameter matrices for multiple-output ReLU networks already with one hidden layer and small training datasets.

\paragraph{ReLU networks and training dynamics.}

We suggest that, in order to further substantially our knowledge of convergence, implicit bias, and generalisation for regression tasks using non-linear networks, we need to understand more thoroughly the dynamics throughout the gradient-based training.  This is because of the observed strong influence that initialisation has on solutions, but is challenging due to the highly non-convex optimisation landscape.  To this end, evidence and intuition were provided by \citet{MaennelBG18}, \citet{LiLL21}, and \citet{JacotGSHG21}, who conjectured that, from sufficiently small initialisations, after an initial phase where the neurons get aligned to a number of directions that depend only on the dataset, training causes the parameters to pass close to a sequence of saddle points, during which their rank increases gradually but stays low.

The first comprehensive analysis in this vein was accomplished by \citet{BoursierPF22}, who focused on orthogonal datasets (which are therefore of cardinality less than or equal to the input dimension), and established that, for one-hidden layer ReLU networks, gradient flow from an infinitesimal initialisation converges to zero loss and is implicitly biased to minimise the Euclidean norm of the network parameters.  They also showed that, per sign class of the training labels (positive or negative), minimising the Euclidean norm of the interpolator networks coincides with minimising their rank.

\paragraph{Our contributions.}

We tackle the main challenge posed by \citet{BoursierPF22}, namely handling datasets that are not orthogonal.  A major obstacle to doing so is that, whereas the analysis of the training dynamics in the orthogonal case made extensive use of an almost complete separation between a turning phase and a growth phase for all hidden neurons, non-orthogonal datasets cause considerably more complex dynamics in which hidden neurons follow training trajectories that simultaneously evolve their directions and norms~\citep[Appendix~A]{BoursierPF22}.

To analyse this involved dynamics in a reasonably clean setting, we consider the training of one-hidden layer ReLU networks by gradient flow from a small balanced initialisation on datasets that are labelled by a teacher ReLU neuron with which all the training points are correlated.  More precisely, we assume that the angles between the training points and the teacher neuron are less than $\pi / 4$, which implies that all angles between training points are less than $\pi / 2$.  The latter restriction has featured per label class in many works in the literature (such as by \citet{PhuongL21} and \cite{WangP22}), and the former is satisfied for example if the training points can be obtained by summing the teacher neuron~$\vec{v}^*$ with arbitrary vectors of length less than $\|\vec{v}^*\| / \:\!\! \sqrt{2}$.  All our other assumptions are very mild, either satisfied with probability exponentially close to~$1$ by any standard random initialisation, or excluding corner cases of Lebesgue measure zero.

Our contributions can be summarised as follows.
\begin{itemize}[itemsep=.25ex,leftmargin=1.5em]
\item
We provide a detailed {\bf non-asymptotic analysis} of the dynamics of each hidden neuron throughout the training, and show that it applies whenever the initialisation scale~$\lambda$ is below a {\bf precise bound} which is polynomial in the network width~$m$ and exponential in the training dataset cardinality~$n$.  Moreover, our analysis applies for any input dimension $d > 1$, for any $n \geq d$ (otherwise exact learning of the teacher neuron may not be possible), for any~$m$, and without assuming any specific random distribution for the initialisation.  In particular, we demonstrate that the role of the overparameterisation in this setting is to ensure that initially at least one hidden neuron with a positive last-layer weight has in its active half-space at least one training point.
\item
We show that, during a first phase of the training, all active hidden neurons with a positive last-layer weight {\bf get aligned} to a single direction which is positively correlated with all training points, whereas all active hidden neurons with a negative last-layer weight get turned away from all training points so that they deactivate.  In contrast to the orthogonal dataset case where the sets of training points that are in the active half-spaces of the neurons are essentially constant during the training, in our correlated setting this first phase in general consists, for each neuron, of a different {\bf sequence of stages} during which the cardinality of the set of training points in its active half-space gradually increases or decreases, respectively.
\item
We show that, during the rest of the training, the bundle of aligned hidden neurons with their last-layer weights, formed by the end of the first phase, grows and turns as it travels from near the origin to near the teacher neuron, and {\bf does not separate}.  To establish the latter property, which is the most involved part of this work, we identify a set in predictor space that depends only on~$\lambda$ and the training dataset, and prove: first, that the trajectory of the bundle {\bf stays inside the set}; and second, that this implies that the directional gradients of the individual neurons are such that the angles between them are non-increasing.
\item
We prove that, after the training departs from the initial saddle, which takes time logarithmic in~$\lambda$ and linear in~$d$, the gradient satisfies a Polyak-{\L}ojasiewicz inequality and consequently the loss {\bf converges to zero exponentially fast}.
\item
We prove that, although for any fixed~$\lambda$ the angles in the bundle of active hidden neurons do not in general converge to zero as the training time tends to infinity, if we let~$\lambda$ tend to zero then the networks to which the training converges have a limit: a network of rank~$1$, in which all non-zero hidden neurons are positive scalings of the teacher neuron and have positive last-layer weights.  This establishes that gradient flow from an infinitesimal initialisation is {\bf implicitly biased} to select interpolator networks of {\bf minimal rank}.  Note also that the limit network is identical in predictor space to the teacher neuron.
\item
We show that, surprisingly, among all networks with zero loss, there may exist some whose Euclidean norm is smaller than that of any network of rank~$1$.  Moreover, we prove that this is the case if and only if a certain condition on angles determined by the training dataset is satisfied.  This result might be seen as {\bf refuting the conjecture} of \citet[section~3.2]{BoursierPF22} that the implicit bias to minimise Euclidean parameter norm holds beyond the orthogonal setting, and adding some weight to the hypothesis of \citet{RazinC20}.  The counterexample networks in our proof have rank~$2$ and make essential use of the ReLU non-linearity.
\item
We perform numerical experiments that indicate that the training dynamics and the implicit bias we theoretically established occur in practical settings in which some of our assumptions are relaxed.  In particular, gradient flow is replaced by gradient descent with a realistic learning rate, the initialisation scales are small but not nearly as small as in the theory, and the angles between the teacher neuron and the training points are distributed around~$\pi / 4$.
\end{itemize}

We further discuss related work, prove all theoretical results, and provide additional material on our experiments, in the appendix.

\section{Preliminaries}
\label{s:prelim}

\paragraph{Notation.}

We write:
$[n]$~for the set $\{1, \ldots, n\}$,
$\|\vec{v}\|$~for the Euclidean length of a vector~$\vec{v}$,
$\overline{\vec{v}} \coloneqq \vec{v} / \|\vec{v}\|$ for the normalised vector,
$\angle(\vec{v}, \vec{v}') \coloneqq
 \arccos(\overline{\vec{v}}^\top \overline{\vec{v}}')$
for the angle between~$\vec{v}$ and~$\vec{v}'$, and
$\cone\{\vec{v}_1, \ldots, \vec{v}_n\} \coloneqq
 \left\{
 \sum_{i = 1}^n \beta_i \vec{v}_i
 \:\middle\vert\:
 \beta_1, \ldots, \beta_n \geq 0
 \right\}$
for the cone generated by vectors $\vec{v}_1$,~\dots,~$\vec{v}_n$.

\paragraph{One-hidden layer ReLU network.}

For an input $\vec{x} \in \mathbb{R}^d$, the output of the network is
\[h_{\vec{\theta}}(\vec{x}) \coloneqq
  {\textstyle \sum_{j = 1}^m} a_j \, \sigma(\vec{w}_j^\top \vec{x}) \;,\]
where $m$~is the width, the parameters $\vec{\theta} = (\vec{a}, \vec{W}) \in \mathbb{R}^m \times \mathbb{R}^{m \times d}$ consist of last-layer weights $\vec{a} = [a_1, \ldots, a_m]$ and hidden-layer weights $\vec{W}^\top = [\vec{w}_1, \ldots, \vec{w}_m]$, and $\sigma(u) \coloneqq \max \{u, 0\}$ is the ReLU function.

\paragraph{Balanced initialisation.}

For all $j \in [m]$ let
\begin{align*}
\vec{w}_j^0 & \coloneqq \lambda \, \vec{z}_j &
a_j^0       & \coloneqq s_j \|\vec{w}_j^0\|
\end{align*}
where $\lambda > 0$ is the initialisation scale, $\vec{z}_j \in \mathbb{R}^d \setminus \{\vec{0}\}$, and $s_j \in \{\pm 1\}$.

A precise upper bound on~$\lambda$ will be stated in \autoref{ass:lambda}.

We regard the initial unscaled hidden-layer weights~$\vec{z}_j$ and last-layer signs~$s_j$ as given, without assuming any specific random distributions for them.  For example, we might have that each~$\vec{z}_j$ consists of $d$~independent centred Gaussians with variance~$\frac{1}{d \, m}$ and each~$s_j$ is uniform over~$\{\pm 1\}$.

We consider only initialisations for which the layers are balanced, i.e.~$|a_j^0| = \|\vec{w}_j^0\|$ for all $j \in [m]$.  Since more generally each difference $(a_j^t)^2 - \|\vec{w}_j^t\|^2$ is constant throughout training \citep[Theorem~2.1]{DuHL18} and we focus on small initialisation scales that tend to zero, this restriction (which is also present in \citet{BoursierPF22}) is minor but simplifies our analysis.

\paragraph{Neuron-labelled correlated inputs.}

The teacher neuron $\vec{v}^* \in \mathbb{R}^d$ and the training dataset $\{(\vec{x}_i, y_i)\}_{i = 1}^n \subseteq (\mathbb{R}^d \setminus \{\vec{0}\}) \times \mathbb{R}$ are such that for all~$i$ we have
\begin{align*}
y_i & = \sigma({\vec{v}^*}^\top \vec{x}_i) &
\angle(\vec{v}^*, \vec{x}_i) & < \pi / 4 \;.
\end{align*}
In particular, since the angles between~$\vec{v}^*$ and the training points~$\vec{x}_i$ are acute, each label~$y_i$ is positive.

To apply our results to a network with biases in the hidden layer and to a teacher neuron with a bias, one can work in dimension $d + 1$ and extend the training points to $\left[\genfrac{}{}{0pt}{}{\vec{x}_i}{1}\right]$.

\paragraph{Mean square loss gradient flow.}

For the regression task of learning the teacher neuron by the one-hidden layer ReLU network, we use the standard mean square empirical loss
\[L(\vec{\theta}) \coloneqq
  \tfrac{1}{2 n} {\textstyle \sum_{i = 1}^n} (y_i - h_{\vec{\theta}}(\vec{x}_i))^2 \;.\]

Our theoretical analysis concentrates on training by gradient flow, which from an initialisation as above evolves the network parameters by descending along the gradient of the loss by infinitesimal steps in continuous time~\citep{LiTE19}.  Formally, we consider any parameter trajectory $\vec{\theta}^t \colon [0, \infty) \to \mathbb{R}^m \times \mathbb{R}^{m \times d}$ that is absolutely continuous on every compact subinterval, and that satisfies the differential inclusion
\[\mathrm{d} \vec{\theta}^t / \mathrm{d} t \in -\partial L(\vec{\theta}^t)
  \quad\text{for almost all } t \in [0, \infty) \;,\]
where $\partial L$ denotes the \citet{clarke1975generalized} subdifferential of the loss function (which is locally Lipschitz).

We work with the Clarke subdifferential, which is a generalisation of the gradient, because the ReLU activation is not differentiable at~$0$, which causes non-differentiability of the loss function~\citep{bolte2010characterizations}.  Although it follows from our results that, in our setting, the derivative of the ReLU can be fixed as $\sigma'(0) \coloneqq 0$ like in the orthogonal case~\citep[Appendix~D]{BoursierPF22}, and the gradient flow trajectories are uniquely defined, that is not a priori clear; hence we work with the unrestricted Clarke subdifferential of the ReLU.  We also remark that, in other settings, $\sigma'(0)$ cannot be fixed in this way due to gradient flow subtrajectories that correspond to gradient descent zig-zagging along a ReLU boundary (cf.~e.g.~\citet[section~9.4]{MaennelBG18}).

\paragraph{Basic observations.}

We establish the formulas for the derivatives of the last-layer weights and the hidden neurons; and that throughout the training, the signs of the last-layer weights do not change, and their absolute values track the norms of the corresponding hidden neurons.  The latter property holds for all times~$t$ by continuity and enables us to focus the analysis on the hidden neurons.

\begin{restatable}{proposition}{prprelim}
\label{pr:prelim}
For all $j \in [m]$ and almost all $t \in [0, \infty)$ we have:
\begin{enumerate}[(i),itemsep=0ex,leftmargin=3em]
\item
\label{pr:prelim.gj}
$\mathrm{d} a_j^t / \mathrm{d} t
= {\vec{w}_j^t}^\top \vec{g}_j^t$
and
$\mathrm{d} \vec{w}_j^t / \mathrm{d} t
= a_j^t \, \vec{g}_j^t$,
where
$\vec{g}_j^t \in
 \tfrac{1}{n}
 {\textstyle \sum_{i = 1}^n}
 (y_i - h_{\vec{\theta}^t}(\vec{x}_i)) \,
 \partial \sigma({\vec{w}_j^t}^\top \vec{x}_i) \,
 \vec{x}_i$;
\item
\label{pr:prelim.balanced}
$a_j^t = s_j \|\vec{w}_j^t\| \neq 0$.
\end{enumerate}
\end{restatable}

The definition in part~\ref{pr:prelim.gj} of the vectors~$\vec{g}_j^t$ that govern the dynamics is a membership because the subdifferential of the ReLU at~$0$ is the set of all values between~$0$ and~$1$, i.e.~$\partial \sigma(0) = [0, 1]$.

\section{Assumptions}
\label{s:ass}

To state our assumptions precisely, we introduce some additional notation.
Let
\begin{align*}
I_\ppp(\vec{v}) & \coloneqq
\{i \in [n] \,\vert\, \vec{v}^\top \vec{x}_i > 0\}
&
I_0(\vec{v}) & \coloneqq
\{i \in [n] \,\vert\, \vec{v}^\top \vec{x}_i = 0\}
&
I_\mmm(\vec{v}) & \coloneqq
\{i \in [n] \,\vert\, \vec{v}^\top \vec{x}_i < 0\}
\end{align*}
denote the sets of indices of training points that are, respectively, either inside or on the boundary or outside of the non-negative half-space of a vector~$\vec{v}$.
Then let
\begin{align*}
J_\ppp & \coloneqq
\{j \in [m] \;\;\vert\;\;
  I_\ppp(\vec{z}_j) \neq \emptyset \,\wedge\,
  s_j = +1\}
&
J_\mmm & \coloneqq
\{j \in [m] \;\;\vert\;\;
  I_\ppp(\vec{z}_j) \neq \emptyset \,\wedge\,
  s_j = -1\}
\end{align*}
be the sets of indices of hidden neurons that are initially active on at least one training point and whose last-layer signs are, respectively, positive or negative.
Also let
\begin{align*}
\vec{X} & \coloneqq
[\vec{x}_1, \ldots, \vec{x}_n]
&
\vec{\gamma}_I & \coloneqq
\tfrac{1}{n} {\textstyle \sum_{i \in I}} y_i \vec{x}_i
\end{align*}
denote the matrix whose columns are all the training points, and the sum of all training points whose indices are in a set~$I$, weighted by the corresponding labels and divided by~$n$.

Moreover we define, for each $j \in J_\ppp \cup J_\mmm$, a continuous trajectory~$\vec{\alpha}_j^t$ in~$\mathbb{R}^d$ by
\begin{align*}
\vec{\alpha}_j^0 & \coloneqq
\vec{z}_j
&
\mathrm{d} \vec{\alpha}_j^t / \mathrm{d} t & \coloneqq
s_j \|\vec{\alpha}_j^t\| \, \vec{\gamma}_{I_\ppp(\vec{\alpha}_j^t)}
\quad\text{for all } t \in (0, \infty) \;.
\end{align*}
Thus, starting from the unscaled initialisation~$\vec{z}_j$ of the corresponding hidden neuron, $\vec{\alpha}_j^t$~follows a dynamics obtained from that of~$\vec{w}_j^t$ in \autoref{pr:prelim}~\ref{pr:prelim.gj} and~\ref{pr:prelim.balanced} by replacing the vector~$\vec{g}_j^t$ by~$\vec{\gamma}_{I_\ppp(\vec{\alpha}_j^t)}$, which amounts to removing from~$\vec{g}_j^t$ the network output terms and the activation boundary summands.  These trajectories will be useful as \emph{yardsticks} in our analysis of the first phase of the training.

\begin{assumption}
\label{ass:enum}
\begin{enumerate}[(i),itemsep=0ex,leftmargin=3em]
\item
\label{ass:enum.span}
$d > 1$, $\lspn\{\vec{x}_1, \ldots, \vec{x}_n\} = \mathbb{R}^d$, and $\|\vec{v}^*\| = 1$.
\item
\label{ass:enum.init}
$J_\ppp \neq \emptyset$,
$I_0(\vec{z}_j) = \emptyset$ for all $j \in [m]$, and
$\angle(\vec{z}_j, \vec{\gamma}_{[n]}) > 0$ for all $j \in J_\mmm$.
\item
\label{ass:enum.eigen}
$\overline{\vec{x}}_1$,~\ldots,~$\overline{\vec{x}}_n$ are distinct, the eigenvalues of $\frac{1}{n} \vec{X} \vec{X}^\top$ are distinct, and $\vec{v}^*$~does not belong to a span of fewer than~$d$ eigenvectors of $\frac{1}{n} \vec{X} \vec{X}^\top$.
\item
\label{ass:enum.omega}
$|I_0(\vec{\alpha}_j^t)| \leq 1$ for all $j \in J_\ppp \cup J_\mmm$ and all $t \in [0, \infty)$.
\item
\label{ass:enum.death}
For all $j \in [m]$ and all $0 \leq T < T'$, if for all $t \in (T, T')$ we have $I_\ppp(\vec{w}_j^t) = I_0\bigl(\vec{w}_j^{T'}\bigr) \neq \emptyset$ and $I_0(\vec{w}_j^t) = I_\ppp\bigl(\vec{w}_j^{T'}\bigr) = \emptyset$, then for all $t \geq T'$ we have $\vec{w}_j^t = \vec{w}_j^{T'}$.
\end{enumerate}
\end{assumption}

This assumption is very mild.
Part~\ref{ass:enum.span} excludes the trivial univariate case without biases (for univariate inputs with biases one needs $d = 2$), ensures that exact learning is possible, and fixes the length of the teacher neuron to streamline the presentation.
Part~\ref{ass:enum.init} assumes that, initially: at least one hidden neuron with a positive last-layer weight has in its active half-space at least one training point, no training point is at a ReLU boundary, and no hidden neuron with a negative last-layer weight is perfectly aligned with the~$\vec{\gamma}_{[n]}$ vector; this holds with probability at least $1 - (3 / 4)^m$ for any continuous symmetric distribution of the unscaled hidden-neuron initialisations, e.g.~$\vec{z}_j \!\iid \mathcal{N}(\vec{0}, \frac{1}{d \, m} \vec{I}_d)$, and the uniform distribution of the last-layer signs $s_j \!\iid \mathcal{U}\{\pm 1\}$.
Parts~\ref{ass:enum.eigen} and~\ref{ass:enum.omega} exclude corner cases of Lebesgue measure zero; observe that $\frac{1}{n} \vec{X} \vec{X}^\top$ is positive-definite, and that \ref{ass:enum.omega}~rules out a yardstick trajectory encountering two or more training points in its half-space boundary at exactly the same time.
Part~\ref{ass:enum.death} excludes some unrealistic gradient flows that might otherwise be possible due to the use of the subdifferential: it specifies that, whenever a neuron deactivates (i.e.~all training points exit its positive half-space), then it stays deactivated for the remainder of the training.

Before our next assumption, we define several further quantities.
Let $\eta_1 > \cdots > \eta_d > 0$ denote the eigenvalues of $\frac{1}{n} \vec{X} \vec{X}^\top$, and let $\vec{u}_1$,~\dots,~$\vec{u}_d$ denote the corresponding unit-length eigenvectors such that $\vec{v}^* = \sum_{k = 1}^d \nu^*_k \vec{u}_k$ for some $\nu^*_1, \ldots, \nu^*_d > 0$.
Also, for each $j \in J_\ppp \cup J_\mmm$, let $n_j \coloneqq |I_{-s_j}(\vec{z}_j)|$ be the number of training points that should enter into or exit from the non-negative half-space along the trajectory~$\vec{\alpha}_j^t$ depending on whether the sign~$s_j$ is positive or negative (respectively), and let
\[\varphi_j^t \coloneqq
  \angle(\vec{\alpha}_j^t, \vec{\gamma}_{I_\ppp(\vec{\alpha}_j^t)})
  \quad\text{for all } t \in [0, \infty)
       \text{ such that } I_\ppp(\vec{\alpha}_j^t) \neq \emptyset\]
be the evolving angle between~$\vec{\alpha}_j^t$ and the vector governing its dynamics (if any).
Then the existence of the times at which the entries or the exits occur is confirmed in the following.

\begin{proposition}
\label{main:pr:omega}
For all $j \in J_\ppp \cup J_\mmm$ there exist a unique enumeration $i_j^1, \ldots, i_j^{n_j}$ of $I_{-s_j}(\vec{z}_j)$ and unique $0 = \tau_j^0 < \tau_j^1 < \cdots < \tau_j^{n_j}$ such that for all $\ell \in [n_j]$:
\begin{enumerate}[(i),itemsep=0ex,leftmargin=3em]
\item
$I_{s_j}(\vec{\alpha}_j^t) =
 I_{s_j}(\vec{z}_j) \cup \{i_j^1, \ldots, i_j^{\ell - 1}\}$
for all $t \in (\tau_j^{\ell - 1}, \tau_j^\ell)$;
\item
$I_0(\vec{\alpha}_j^t) = \emptyset$
for all $t \in (\tau_j^{\ell - 1}, \tau_j^\ell)$, and
$I_0\Bigl(\vec{\alpha}_j^{\tau_j^\ell}\Bigr) = \{i_j^\ell\}$.
\end{enumerate}
\end{proposition}

Finally we define two measurements of the unscaled initialisation and the training dataset, which are positive thanks to \autoref{ass:enum}, and which will simplify the presentation of our results.
\begin{align*}
\delta & \coloneqq
\min\!\left\{\!\!
\begin{array}{c}
\min_{i \in [n]}
\|\vec{x}_i\|, \;
\min_{i, i' \in [n]}
{\overline{\vec{x}}_i\!}^\top \, \overline{\vec{x}}_{i'}, \;
\min_{k \in [d - 1]}
(\!\sqrt{\eta_k} - \sqrt{\eta_{k + 1}}) (d - 1), \;
\sqrt{\eta_d},
\\[1.33ex]
\min_{k \in [d]}
\nu^*_k \sqrt{d}, \;
\min_{j \in [m]}
\|\vec{z}_j\|, \;
\min_{j \in J_\ppp}
\cos \varphi_j^0, \;
\min_{j \in J_\mmm}
\sin \varphi_j^0,
\\[1.33ex]
\min\!\left\{
|{\overline{\vec{\alpha}}_j^t}^\top \, \overline{\vec{x}}_i|
\,\,\,\middle\vert
\begin{array}{c}
j \in J_\ppp \cup J_\mmm \,\wedge\,
\ell \in [n_j]
\\[.67ex] \wedge\;
t \in [\tau_j^{\ell - 1}, \tau_j^\ell] \,\wedge\,
i \in [n]
\\[.67ex] \wedge\;
i \neq i_j^\ell \,\wedge\,
(\ell \neq 1 \Rightarrow i \neq i_j^{\ell - 1})
\end{array}
\!\!\right\}\!, \;
\min_{j \in J_\mmm}
{\overline{\vec{\alpha}}_j^0}^\top \, \overline{\vec{x}}_{i_j^1},
\\[4ex]
\min
\{\tau_j^\ell - \tau_j^{\ell - 1}
  \,\,\,\vert\,\,\,
  j \in J_\ppp \cup J_\mmm \,\wedge\,
  \ell \in [n_j]\}
\end{array}
\!\!\right\}\!
\\[.33ex]
\Delta & \coloneqq
\max
\{{\textstyle \max_{i \in [n]}}
  \|\vec{x}_i\|, \;
  {\textstyle \max_{j \in [m]}}
  \|\vec{z}_j\|, \;
  1\} \;.
\end{align*}

\begin{assumption}
\label{ass:lambda}
$0 < \varepsilon \leq \frac{1}{4}$ and
$\lambda \leq {\Bigl(m \, n^{9 n {\Delta\!}^2 / \delta^3}\Bigr)\!}^{-3 / \varepsilon}$.
\end{assumption}

The quantity~$\varepsilon$ introduced here has no effect on the network training, but is a parameter of our analysis, so that varying it within the assumed range tightens some of the resulting bounds while loosening others.  The assumed bound on the initialisation scale~$\lambda$ is polynomial in the network width~$m$ and exponential in the dataset cardinality~$n$.  The latter is also the case in \citet{BoursierPF22}, where the bound was stated informally and without its dependence on parameters other than~$m$ and~$n$.

\section{First phase: alignment or deactivation}
\label{s:first}

We show that, for each initially active hidden neuron, if its last-layer sign is positive then it turns to include in its active half-space all training points that were initially outside, whereas if its last-layer sign is negative then it turns to remove from its active half-space all training points that were initially inside.  Moreover, those training points cross the activation boundary in the same order as they cross the half-space boundary of the corresponding yardstick trajectory~$\vec{\alpha}_j^t$, and at approximately the same times (cf.~\autoref{main:pr:omega}).

\begin{lemma}
\label{main:l:0}
For all $j \in J_\ppp \cup J_\mmm$ there exist unique $0 = t_j^0 < t_j^1 < \ldots < t_j^{n_j}$ such that for all $\ell \in [n_j]$:
\begin{enumerate}[(i),itemsep=0ex,leftmargin=3em]
\item
$I_{s_j}(\vec{w}_j^t) =
 I_{s_j}(\vec{z}_j) \cup \{i_j^1, \ldots, i_j^{\ell - 1}\}$
for all $t \in (t_j^{\ell - 1}, t_j^\ell)$;
\item
$I_0(\vec{w}_j^t) = \emptyset$ for all $t \in (t_j^{\ell - 1}, t_j^\ell)$, and
$I_0\Bigl(\vec{w}_j^{t_j^\ell}\Bigr) = \{i_j^\ell\}$;
\item
$|\tau_j^\ell - t_j^\ell|
 \leq \lambda^{1 - \left(\!1 + \frac{3 \ell - 1}{3 n_j}\!\right) \varepsilon}$.
\end{enumerate}
\end{lemma}

The preceding lemma is proved by establishing, for this first phase of the training, non-asymptotic upper bounds on the Euclidean norms of the hidden neurons and hence on the absolute values of the network outputs, and inductively over the stage index~$\ell$, on the distances between the unit-sphere normalisations of~$\vec{\alpha}_j^t$ and~$\vec{w}_j^t$.  Based on that analysis, we then obtain that each negative-sign hidden neuron does not grow from its initial length and deactivates by time $T_0 \coloneqq \max_{j \in J_\ppp \cup J_\mmm} \tau_j^{n_j} + 1$.

\begin{lemma}
\label{main:l:len.w}
For all $j \in J_\mmm$ we have:
\begin{align*}
\|\vec{w}_j^{T_0}\|
& \leq \lambda \|\vec{z}_j\|
&
\vec{w}_j^t
& = \vec{w}_j^{T_0}
\quad\text{for all } t \geq T_0 \;.
\end{align*}
\end{lemma}

We also obtain that, up to a later time $T_1 \coloneqq \varepsilon \ln(1 / \lambda) / \|\vec{\gamma}_{[n]}\|$, each positive-sign hidden neuron: grows but keeps its length below $2 \|\vec{z}_j\| \lambda^{1 - \varepsilon}$, continues to align to the vector~$\vec{\gamma}_{[n]}$ up to a cosine of at least $1 - \lambda^\varepsilon$, and maintains bounded by~$\lambda^{1 - 3 \varepsilon}$ the difference between the logarithm of its length divided by the initialisation scale and the logarithm of the corresponding yardstick vector length.

\begin{lemma}
\label{main:l:1}
For all $j \in J_\ppp$ we have:
\begin{align*}
\|\vec{w}_j^{T_1}\|
& < 2 \|\vec{z}_j\| \lambda^{1 - \varepsilon}
&
{\overline{\vec{w}}_j^{T_1}\!}^\top \, \overline{\vec{\gamma}}_{[n]}
& \geq 1 - \lambda^\varepsilon
&
|\ln \|\vec{\alpha}_j^{T_1}\| - \ln \|\vec{w}_j^{T_1} \! / \lambda\||
& \leq \lambda^{1 - 3 \varepsilon} \;.
\end{align*}
\end{lemma}

\section{Second phase: growth and convergence}
\label{s:second}

We next analyse the gradient flow subsequent to the deactivation of the negative-sign hidden neurons by time~$T_0$ and the alignment of the positive-sign ones up to time~$T_1$, and establish that the loss converges to zero at a rate which is exponential and does not depend on the initialisation scale~$\lambda$.

\begin{theorem}
\label{th:2}
Under Assumptions~\ref{ass:enum} and~\ref{ass:lambda}, there exists a time
$T_2 <
 \ln(1 / \lambda)
 (4 + \varepsilon) d {\Delta\!}^2 / \delta^6$
such that for all $t \geq 0$ we have
$L(\vec{\theta}^{T_2 + t}) <
 0.5 \, {\Delta\!}^2 \,
 \e^{-t \,\cdot\, 0.4 \, \delta^4 / {\Delta\!}^2}$.
\end{theorem}

In particular, for
$\varepsilon = 1 / 4$ and
$\lambda = {\Bigl((m \, n^n)^{9 {\Delta\!}^2 / \delta^3}\Bigr)\!}^{-3 / \varepsilon}$
(cf.~\autoref{ass:lambda}), the first bound in \autoref{th:2} becomes
$T_2 < (\ln m + n \ln n) \, d \cdot 17 \cdot 27 \, {\Delta\!}^4 / \delta^9$.

The proof of \autoref{th:2} is in large part geometric, with a key role played by a set $\mathcal{S} \coloneqq \mathcal{S}_1 \cup \cdots \cup \mathcal{S}_d$ in predictor space, whose constituent subsets are defined as
\[\mathcal{S}_\ell \coloneqq
\left\{
\vec{v} = \sum_{k = 1}^d \nu_k \vec{u}_k
\,\,\,\middle\vert\,\,\,
\bigwedge_{1 \leq k < \ell}
\Omega_k
\,\,\wedge\,\,
\Phi_\ell
\,\,\wedge
\bigwedge_{\ell \leq k < k' \leq d}
(\Psi_{k, k'}^\downarrow \wedge
 \Psi_{k, k'}^\uparrow)
\,\,\wedge\,\,
\Xi
\right\} \:,\]
where the individual constraints are as follows (here $\eta_0 \coloneqq \infty$ so that e.g.~$\frac{\eta_1}{2 \eta_0} = 0$):
\begin{align*}
\Omega_k & \colon\;
1 < \frac{\nu_k}{\nu^*_k}
&
\Phi_\ell & \colon\;
\frac{\eta_\ell}{2 \eta_{\ell - 1}} < \frac{\nu_\ell}{\nu^*_\ell} \leq 1
&
\Psi_{k, k'}^\downarrow & \colon\;
\frac{\eta_{k'}}{2 \eta_k}
\frac{\nu_k}{\nu^*_k}
< \frac{\nu_{k'}}{\nu^*_{k'}}
\\
\Xi & \colon\;
\mathrlap{
\overline{\vec{v}}^\top \,
\overline{\vec{X} \vec{X}^\top (\vec{v}^* - \vec{v})} >
\lambda^{\varepsilon / 3}}
& & &
\Psi_{k, k'}^\uparrow & \colon\;
\frac{\nu_{k'}}{\nu^*_{k'}} <
1 - {\!\left(\!1 - \frac{\nu_k}{\nu^*_k}\!\right)\!}
    ^{\frac{1}{2} + \frac{\eta_{k'}}{2 \eta_k}} \;.
\end{align*}
Thus $\mathcal{S}$~is connected, open, and constrained by~$\Xi$ to be within the ellipsoid $\vec{v}^\top \vec{X} \vec{X}^\top (\vec{v}^* - \vec{v}) = 0$ which is centred at~$\frac{\vec{v}^*}{2}$, with the remaining constraints slicing off further regions by straight or curved boundary surfaces.

In the most complex component of this work, we show that, for all $t \geq T_1$, the trajectory of the sum $\vec{v}^t \coloneqq \sum_{j \in J_\ppp} a_j^t \vec{w}_j^t$ of the active hidden neurons weighted by the last layer stays inside~$\mathcal{S}$, and the cosines of the angles between the neurons remain above $1 - 4 \lambda^\varepsilon$.  This involves proving that each face of the boundary of~$\mathcal{S}$ is repelling for the training dynamics when approached from the inside; we remark that, although that is in general false for the entire boundary of the constraint~$\Xi$, it is in particular true for its remainder after the slicing off by the other constraints.  We also show that all points in~$\mathcal{S}$ are positively correlated with all training points, which together with the preceding facts implies that, during this second phase of the training, the network behaves approximately like a linear one-hidden layer one-neuron network.  Then, as the cornerstone of the rest of the proof, we show that, for all $t \geq T_2$, the gradient of the loss satisfies a \citeauthor{POLYAK1963864}-{\L}ojasiewicz inequality
$\|\nabla L(\vec{\theta}^t)\|^2 >
 \frac{2 \eta_d \|\vec{\gamma}_{[n]}\|}{5 \eta_1}
 L(\vec{\theta}^t)$.
Here
$T_2 \coloneqq
 \inf \{t \geq T_1 \,\,\vert\,\,
        \nu_1^t / \nu^*_1 \geq 1 / 2\}$
is a time by which the network has departed from the initial saddle, more precisely when the first coordinate~$\nu_1^t$ of the bundle vector~$\vec{v}^t$ with respect to the basis consisting of the eigenvectors of the matrix $\frac{1}{n} \vec{X} \vec{X}^\top$ crosses the half-way threshold to the first coordinate~$\nu^*_1$ of the teacher neuron.

The interior of the ellipsoid in the constraint~$\Xi$ actually consists of all vectors that have an acute angle with the derivative of the training dynamics in predictor space, and the ``padding'' of~$\lambda^{\varepsilon / 3}$ is present because the derivative of the bundle vector~$\vec{v}^t$ is ``noisy'' due to the latter being made up of the approximately aligned neurons.  The remaining constraints delimit the subsets $\mathcal{S}_1$,~\dots,~$\mathcal{S}_d$ of the set~$\mathcal{S}$, through which the bundle vector~$\vec{v}_t$ passes in that order, with each unique ``handover'' from~$\mathcal{S}_\ell$ to~$\mathcal{S}_{\ell + 1}$ happening exactly when the corresponding coordinate~$\nu_\ell^t$ exceeds its target~$\nu^*_\ell$.  The non-linearity of the constraints~$\Psi_{k, k'}^\uparrow$ is needed to ensure the repelling for the training dynamics.

\section{Implicit bias of gradient flow}
\label{s:bias}

Let us denote the set of all balanced networks by
\[\Theta \coloneqq
  \{(\vec{a}, \vec{W}) \in \mathbb{R}^m \times \mathbb{R}^{m \times d} \,\,\,\vert\,\,\,
    \forall j \in [m] \colon |a_j| = \|\vec{w}_j\|\}\]
and the subset in which all non-zero hidden neurons are positive scalings of $\vec{v}^*$, have positive last-layer weights, and have lengths whose squares sum up to $\|\vec{v}^*\| = 1$, by
\[\Theta_{\vec{v}^*} \coloneqq
  \{(\vec{a}, \vec{W}) \in \Theta \,\,\,\vert\,\,\,
    {\textstyle \sum_{j = 1}^m} \|\vec{w}_j\|^2 = 1 \;\wedge\;
    \forall j \in [m] \colon \vec{w}_j \neq \vec{0} \:\Rightarrow\:
    (\overline{\vec{w}}_j = \vec{v}^* \,\wedge\, a_j > 0)\} \;.\]

Our main result establishes that, as the initialisation scale~$\lambda$ tends to zero, the networks with zero loss to which the gradient flow converges tend to a network in~$\Theta_{\vec{v}^*}$.  The explicit subscripts indicate the dependence on~$\lambda$ of the parameter vectors.  The proof builds on the preceding results and involves a careful control of accumulations of approximation errors over lengthy time intervals.

\begin{restatable}{theorem}{thbias}
\label{th:bias}
Under Assumptions~\ref{ass:enum} and~\ref{ass:lambda},
$L\Bigl(\,{\displaystyle \lim_{t \to \infty\mathstrut}}
          \vec{\theta}_\lambda^t\Bigr) = 0$ and
${\displaystyle \lim_{\lambda \to 0^+}} \,
 {\displaystyle \lim_{t \to \infty\mathstrut}}
 \vec{\theta}_\lambda^t
 \in \Theta_{\vec{v}^*}$.
\end{restatable}

\section{Interpolators with minimum norm}
\label{s:inter}

To compare the set~$\Theta_{\vec{v}^*}$ of balanced rank-$1$ interpolator networks with the set of all minimum-norm interpolator networks, in this section we focus on training datasets of cardinality~$d$, we assume the network width is greater than~$1$ (otherwise the rank is necessarily~$1$), and we exclude the threshold case of Lebesgue measure zero where $\mathcal{M} = 0$.  The latter measurement of the training dataset is defined below in terms of angles between the teacher neuron and vectors in any two cones generated by different generators of the dual of the cone of all training points.

Let $[\vec{\chi}_1, \ldots, \vec{\chi}_d]^\top \coloneqq \vec{X}^{-1}$ and
\[\mathcal{M} \coloneqq
\max\left\{
\cos \angle(\vec{p}, \vec{q}) -
\sin \angle(\vec{p}, \vec{v}^*)
\,\,\,\middle\vert
\begin{array}{c}
\emptyset \subsetneq K \subsetneq [d]
\\[.33ex] \wedge \;
\vec{0} \neq \vec{p} \in \cone\{\vec{\chi}_k \:\vert\; k \in K\}
\\[.33ex] \wedge \;
\vec{0} \neq \vec{q} \in \cone\{\vec{\chi}_k \:\vert\; k \notin K\}
\end{array}
\!\!\!\right\} \:.\]

\begin{assumption}
\label{ass:M}
$n = d$, $m > 1$, and $\mathcal{M} \neq 0$.
\end{assumption}

We obtain that, surprisingly, $\Theta_{\vec{v}^*}$~equals the set of all interpolators with minimum Euclidean norm if $\mathcal{M} < 0$, but otherwise they are disjoint.

\begin{restatable}{theorem}{thmin}
\label{th:min}
Under Assumptions~\ref{ass:enum} and~\ref{ass:M}:
\begin{enumerate}[(i),itemsep=0ex,leftmargin=3em]
\item
if $\mathcal{M} < 0$ then $\Theta_{\vec{v}^*}$~is the set of all global minimisers of~$\|\vec{\theta}\|^2$ subject to $L(\vec{\theta}) = 0$;
\item
if $\mathcal{M} > 0$ then no point in~$\Theta_{\vec{v}^*}$ is a global minimiser of~$\|\vec{\theta}\|^2$ subject to $L(\vec{\theta}) = 0$.
\end{enumerate}
\end{restatable}

For each of the two cases, we provide a family of example datasets in \appendixref{app:inter}.  We remark that a sufficient condition for $\mathcal{M} < 0$ to hold is that the inner product of any two distinct rows~$\vec{\chi}_k$ of the inverse of the dataset matrix~$\vec{X}$ is non-positive, i.e.~that the inverse of the Gram matrix of the dataset (in our setting this Gram matrix is positive) is a Z-matrix (cf.~e.g.~\citet{Fiedler1962}).  Also, if the training points were orthogonal then all the $\cos \angle(\vec{p}, \vec{q})$ terms in the definition of~$\mathcal{M}$ would be zero and consequently we would have $\mathcal{M} < 0$; this is consistent with the result that, per sign class of the training labels in the orthogonal setting, minimising the Euclidean norm of interpolators coincides with minimising their rank \citep[Appendix~C]{BoursierPF22}.

\section{Experiments}
\label{s:exp}

We consider two schemes for generating the training dataset, where $\mathbb{S}^{d - 1}$~is the unit sphere in~$\mathbb{R}^d$.
\begin{description}[itemsep=.25ex,style=unboxed,wide]
\item[Centred:]
We sample~$\vec{\mu}$ from $\mathcal{U}(\mathbb{S}^{d - 1})$, then sample $\vec{x}_1, \ldots, \vec{x}_d$ from $\mathcal{N}(\vec{\mu}, \frac{\rho}{d} \vec{I}_d)$ where $\rho = 1$, and finally set $\vec{v}^* = \vec{\mu}$.  This distribution has the property that, in high dimensions, the angles between the teacher neuron~$\vec{v}^*$ and the training points~$\vec{x}_i$ concentrate around~$\pi / 4$.  We exclude rare cases where some of these angles exceed~$\pi / 2$.
\item[Uncentred:]
This is the same, except that we use $\rho = \sqrt{2} - 1$, sample one extra point $\vec{x}_0$, and finally set $\vec{v}^* = \overline{\vec{x}}_0$.  Here the angles between~$\vec{v}^*$ and~$\vec{x}_i$ also concentrate around~$\pi / 4$ in high dimensions, but the expected distance between~$\vec{v}^*$ and~$\vec{\mu}$ is~$\sqrt{\rho}$.
\end{description}

For each of the two dataset schemes, we train a one-hidden layer ReLU network of width $m = 200$ by gradient descent with learning rate~$0.01$, from a balanced initialisation such that $\vec{z}_j \!\iid \mathcal{N}(\vec{0}, \frac{1}{d \, m} \vec{I}_d)$ and $s_j \!\iid \mathcal{U}\{\pm 1\}$, and for a range of initialisation scales~$\lambda$ and input dimensions~$d$.

We present in \autoref{f:exp} some results from considering initialisation scales $\lambda = 4^2, 4^1, \ldots, 4^{-12}, 4^{-13}$ and input dimensions $d = 4, 16, 64, 256, 1024$, where we train until the number of iterations reaches~$2 \cdot 10^7$ or the loss drops below~$10^{-9}$.  The plots are in line with \autoref{th:bias}, showing how the maximum angle between active hidden neurons at the end of the training decreases with~$\lambda$.

\newcommand{\foreachd}[2]{
\foreach \d in {4,16,64,256,1024} {
\addplot+ [error bars/.cd, y dir=both] table [
  col sep=comma,
  x expr={and(\thisrow{init scale}!=9.31322574615478E-10,
              \thisrow{init scale}!=3.72529029846191E-09)==1?
          \thisrow{init scale}:nan},
  y=#2-median-\d,
] {table_w_#1.csv}; }}

\begin{figure}
\centering
\includegraphics{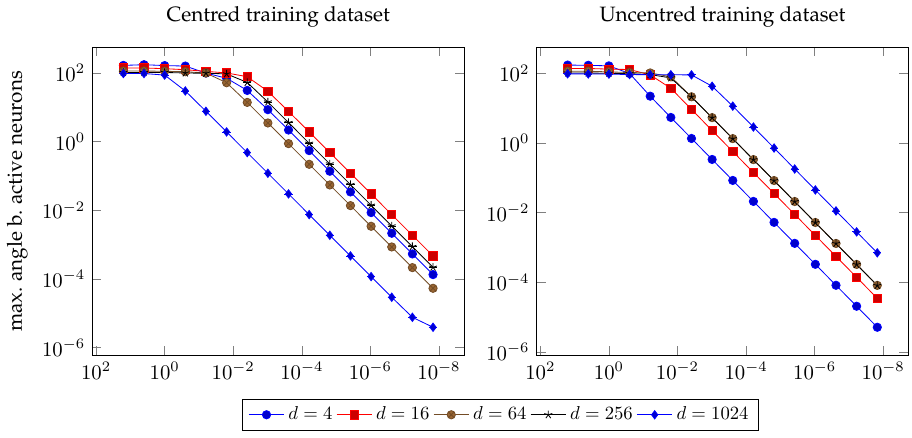}
\caption{Dependence of the maximum angle between active hidden neurons on the initialisation scale~$\lambda$, for two generation schemes of the training dataset and a range of input dimensions, at the end of the training.  Both axes are logarithmic, and each point plotted shows the median over five trials.}
\label{f:exp}
\end{figure}

\autoref{f:loss.centred} on the left illustrates the exponential convergence of the training loss (cf.~\autoref{th:2}), and on the right how the implicit bias can result in good generalisation.  The test loss is computed over an input distribution which is different from that of the training points, namely we sample $64$~test inputs from the standard multivariate $\mathcal{N}(\vec{0}, \vec{I}_d)$.  These plots are for initialisation scales $\lambda = 4^{-2}, 4^{-3}, \ldots, 4^{-7}, 4^{-8}$.

\newcommand{\foreachb}[3]{
\foreach \b in {-2,-3,-4,-5,-6,-7,-8} {
\addplot+ table [
  col sep=comma,
  x expr={and(\thisrow{init scale base}==\b,
              \thisrow{iteration}<#3)?
          \thisrow{iteration}:nan},
  y=#1,
] {splitting_#2.csv}; }}

\begin{figure}
\centering
\includegraphics{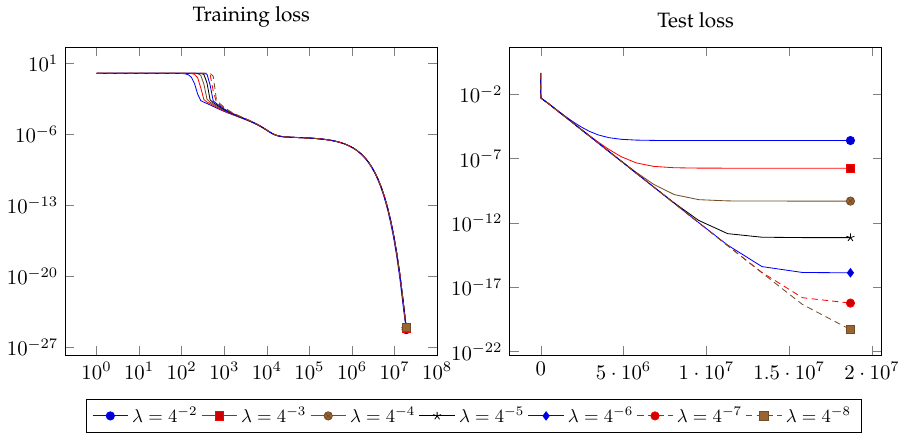}
\caption{Evolution of the training loss, and of an outside distribution test loss, during training for an example centred training dataset in dimension~$16$ and width~$200$.  The horizontal axes, logarithmic for the training loss and linear for the test loss, show iterations.  The vertical axes are logarithmic.}
\label{f:loss.centred}
\end{figure}

\section{Conclusion}
\label{s.concl}

We provided a detailed analysis of the dynamics of training a shallow ReLU network by gradient flow from a small initialisation for learning a single neuron which is correlated with the training points, establishing convergence to zero loss and implicit bias to rank minimisation in parameter space.  We believe that in particular the geometric insights we obtained in order to deal with the complexities of the multi-stage alignment of hidden neurons followed by the simultaneous evolution of their norms and directions, will be useful to the community in the ongoing quest to understand implicit bias of gradient-based algorithms for regression tasks using non-linear networks.

A major direction for future work is to bridge the gap between, on one hand, our assumption that the angles between the teacher neuron and the training points are less than $\pi / 4$, and the other, the assumption of \citet{BoursierPF22} that the training points are orthogonal, while keeping a fine granularity of description.  We expect this to be difficult because it seems to require handling bundles of approximately aligned neurons which may have changing sets of training points in their active half-spaces and which may separate during the training.  However, it should be straightforward to extend our results to orthogonally separable datasets and two teacher ReLU neurons, where each of the latter has an arbitrary sign, labels one of the two classes of training points, and has angles less than $\pi / 4$ with them; the gradient flow would then pass close to a second saddle point, where the labels of one of the classes have been nearly fitted but the hidden neurons that will fit the labels of the other class are still small.  We report on related numerical experiments in \appendixref{app:further}.

We also obtained a condition on the dataset that determines whether rank minimisation and Euclidean norm minimisation for interpolator networks coincide or are distinct.  Although this dichotomy remains true if the $\pi / 4$ correlation bound is relaxed to $\pi / 2$, the implicit bias of gradient flow in that extended setting is an open question.

Other directions for future work include considering multi-neuron teacher networks, student networks with more than one hidden layer, further non-linear activation functions, and gradient descent instead of gradient flow; also refining the bounds on the initialisation scale and the convergence time.

\newcommand{\ackdf}{We acknowledge the Centre for Discrete Mathematics and Its Applications at the University of Warwick for partial support, and the Scientific Computing Research Technology Platform at the University of Warwick for providing the compute cluster on which the experiments presented in this paper were run.}

\phantomsection
\addcontentsline{toc}{section}{Acknowledgments and Disclosure of Funding}
\section*{Acknowledgments and Disclosure of Funding}

\ackdf

\phantomsection
\addcontentsline{toc}{section}{References}
\bibliographystyle{plainnat}
{\small \bibliography{main}}

\begin{thebibliography}{51}
\providecommand{\natexlab}[1]{#1}
\providecommand{\url}[1]{\texttt{#1}}
\expandafter\ifx\csname urlstyle\endcsname\relax
  \providecommand{\doi}[1]{doi: #1}\else
  \providecommand{\doi}{doi: \begingroup \urlstyle{rm}\Url}\fi

\bibitem[Azulay et~al.(2021)Azulay, Moroshko, Nacson, Woodworth, Srebro,
  Globerson, and Soudry]{AzulayMNWSGS21}
Shahar Azulay, Edward Moroshko, Mor~Shpigel Nacson, Blake~E. Woodworth, Nathan
  Srebro, Amir Globerson, and Daniel Soudry.
\newblock \href{http://proceedings.mlr.press/v139/azulay21a.html}{On the
  Implicit Bias of Initialization Shape: Beyond Infinitesimal Mirror Descent}.
\newblock In \emph{ICML}, pages 468--477, 2021.

\bibitem[Belkin et~al.(2019)Belkin, Hsu, Ma, and Mandal]{BelkinHMM19}
Mikhail Belkin, Daniel Hsu, Siyuan Ma, and Soumik Mandal.
\newblock \href{\doiurl 10.1073/pnas.1903070116}{Reconciling modern
  machine-learning practice and the classical bias–variance trade-off}.
\newblock \emph{Proc. Natl. Acad. Sci.}, 116\penalty0 (32):\penalty0
  15849--15854, 2019.

\bibitem[Bolte et~al.(2010)Bolte, Daniilidis, Ley, and
  Mazet]{bolte2010characterizations}
J{\'e}r{\^o}me Bolte, Aris Daniilidis, Olivier Ley, and Laurent Mazet.
\newblock Characterizations of {{\L}ojasiewicz} inequalities: subgradient
  flows, talweg, convexity.
\newblock \emph{Trans. Amer. Math. Soc.}, 362\penalty0 (6):\penalty0
  3319--3363, 2010.

\bibitem[Boursier and Flammarion(2023)]{BoursierF23}
Etienne Boursier and Nicolas Flammarion.
\newblock \href{\doiurl 10.48550/arxiv.2303.01353}{Penalising the biases in
  norm regularisation enforces sparsity}.
\newblock \emph{CoRR}, abs/2303.01353, 2023.
\newblock Accepted to NeurIPS 2023.

\bibitem[Boursier et~al.(2022)Boursier, Pillaud{-}Vivien, and
  Flammarion]{BoursierPF22}
Etienne Boursier, Loucas Pillaud{-}Vivien, and Nicolas Flammarion.
\newblock
  \href{https://proceedings.neurips.cc/paper_files/paper/2022/hash/7eeb9af3eb1f48e29c05e8dd3342b286-Abstract-Conference.html}{Gradient
  flow dynamics of shallow ReLU networks for square loss and orthogonal
  inputs}.
\newblock In \emph{NeurIPS}, 2022.

\bibitem[Clarke(1975)]{clarke1975generalized}
Frank~H. Clarke.
\newblock \href{\doiurl 10.1090/s0002-9947-1975-0367131-6}{Generalized
  gradients and applications}.
\newblock \emph{Trans. Amer. Math. Soc.}, 205:\penalty0 247--262, 1975.

\bibitem[Davis et~al.(2020)Davis, Drusvyatskiy, Kakade, and Lee]{DavisDKL20}
Damek Davis, Dmitriy Drusvyatskiy, Sham~M. Kakade, and Jason~D. Lee.
\newblock \href{\doiurl 10.1007/s10208-018-09409-5}{Stochastic Subgradient
  Method Converges on Tame Functions}.
\newblock \emph{Found. Comput. Math.}, 20\penalty0 (1):\penalty0 119--154,
  2020.

\bibitem[Du et~al.(2018)Du, Hu, and Lee]{DuHL18}
Simon~S. Du, Wei Hu, and Jason~D. Lee.
\newblock
  \href{https://proceedings.neurips.cc/paper/2018/hash/fe131d7f5a6b38b23cc967316c13dae2-Abstract.html}{Algorithmic
  Regularization in Learning Deep Homogeneous Models: Layers are Automatically
  Balanced}.
\newblock In \emph{NeurIPS}, pages 382--393, 2018.

\bibitem[Englert and Lazi\'c(2022)]{Englert022}
Matthias Englert and Ranko Lazi\'c.
\newblock
  \href{https://proceedings.neurips.cc//paper_files/paper/2022/hash/b734c30b9c955c535e333f0301f5e45c-Abstract-Conference.html}{Adversarial
  Reprogramming Revisited}.
\newblock In \emph{NeurIPS}, 2022.

\bibitem[Ergen and Pilanci(2021)]{ErgenP21}
Tolga Ergen and Mert Pilanci.
\newblock \href{http://jmlr.org/papers/v22/20-1447.html}{Convex Geometry and
  Duality of Over-parameterized Neural Networks}.
\newblock \emph{J. Mach. Learn. Res.}, 22\penalty0 (212):\penalty0 1--63, 2021.

\bibitem[Even et~al.(2023)Even, Pesme, Gunasekar, and Flammarion]{EvenPGF23}
Mathieu Even, Scott Pesme, Suriya Gunasekar, and Nicolas Flammarion.
\newblock \href{\doiurl 10.48550/arxiv.2302.08982}{{(S)GD} over Diagonal Linear
  Networks: Implicit Regularisation, Large Stepsizes and Edge of Stability}.
\newblock \emph{CoRR}, abs/2302.08982, 2023.
\newblock Accepted to NeurIPS 2023.

\bibitem[Fiedler and Pt\'ak(1962)]{Fiedler1962}
Miroslav Fiedler and Vlastimil Pt\'ak.
\newblock \href{\doiurl 10.21136/cmj.1962.100526}{On matrices with non-positive
  off-diagonal elements and positive principal minors}.
\newblock \emph{Czechoslovak Mathematical Journal}, 12\penalty0 (3):\penalty0
  382--400, 1962.

\bibitem[Frei et~al.(2020)Frei, Cao, and Gu]{FreiCG20}
Spencer Frei, Yuan Cao, and Quanquan Gu.
\newblock
  \href{https://proceedings.neurips.cc/paper/2020/hash/3a37abdeefe1dab1b30f7c5c7e581b93-Abstract.html}{Agnostic
  Learning of a Single Neuron with Gradient Descent}.
\newblock In \emph{NeurIPS}, 2020.

\bibitem[Frei et~al.(2023{\natexlab{a}})Frei, Vardi, Bartlett, and
  Srebro]{FreiVBS23Benign}
Spencer Frei, Gal Vardi, Peter~L. Bartlett, and Nathan Srebro.
\newblock \href{https://proceedings.mlr.press/v195/frei23a.html}{Benign
  Overfitting in Linear Classifiers and Leaky ReLU Networks from {KKT}
  Conditions for Margin Maximization}.
\newblock In \emph{COLT}, pages 3173--3228, 2023{\natexlab{a}}.

\bibitem[Frei et~al.(2023{\natexlab{b}})Frei, Vardi, Bartlett, and
  Srebro]{FreiVBS23Double}
Spencer Frei, Gal Vardi, Peter~L. Bartlett, and Nathan Srebro.
\newblock \href{\doiurl 10.48550/arxiv.2303.01456}{The Double-Edged Sword of
  Implicit Bias: Generalization vs. Robustness in ReLU Networks}.
\newblock \emph{CoRR}, abs/2303.01456, 2023{\natexlab{b}}.
\newblock Accepted to NeurIPS 2023.

\bibitem[Frei et~al.(2023{\natexlab{c}})Frei, Vardi, Bartlett, Srebro, and
  Hu]{frei2023implicit}
Spencer Frei, Gal Vardi, Peter~L. Bartlett, Nathan Srebro, and Wei Hu.
\newblock \href{https://openreview.net/forum?id=JpbLyEI5EwW}{Implicit Bias in
  Leaky Re{LU} Networks Trained on High-Dimensional Data}.
\newblock In \emph{ICLR}, 2023{\natexlab{c}}.

\bibitem[Gunasekar et~al.(2017)Gunasekar, Woodworth, Bhojanapalli, Neyshabur,
  and Srebro]{GunasekarWBNS17}
Suriya Gunasekar, Blake~E. Woodworth, Srinadh Bhojanapalli, Behnam Neyshabur,
  and Nati Srebro.
\newblock
  \href{https://proceedings.neurips.cc/paper/2017/hash/58191d2a914c6dae66371c9dcdc91b41-Abstract.html}{Implicit
  Regularization in Matrix Factorization}.
\newblock In \emph{NeurIPS}, pages 6151--6159, 2017.

\bibitem[Haim et~al.(2022)Haim, Vardi, Yehudai, Shamir, and Irani]{HaimVYSI22}
Niv Haim, Gal Vardi, Gilad Yehudai, Ohad Shamir, and Michal Irani.
\newblock
  \href{https://proceedings.neurips.cc/paper_files/paper/2022/hash/906927370cbeb537781100623cca6fa6-Abstract-Conference.html}{Reconstructing
  Training Data From Trained Neural Networks}.
\newblock In \emph{NeurIPS}, 2022.

\bibitem[Jacot et~al.(2021)Jacot, Ged, \c{S}im\c{s}ek, Hongler, and
  Gabriel]{JacotGSHG21}
Arthur Jacot, Fran{\c{c}}ois Ged, Berfin \c{S}im\c{s}ek, Cl{\'{e}}ment Hongler,
  and Franck Gabriel.
\newblock \href{https://arxiv.org/abs/2106.15933}{Saddle-to-Saddle Dynamics in
  Deep Linear Networks: Small Initialization Training, Symmetry, and Sparsity}.
\newblock \emph{CoRR}, abs/2106.15933, 2021.

\bibitem[Jentzen and Riekert(2023)]{JENTZEN2023126601}
Arnulf Jentzen and Adrian Riekert.
\newblock \href{\doiurl 10.1016/j.jmaa.2022.126601}{Convergence analysis for
  gradient flows in the training of artificial neural networks with {ReLU}
  activation}.
\newblock \emph{J. Math. Anal. Appl.}, 517\penalty0 (2):\penalty0 126601, 2023.

\bibitem[Ji and Telgarsky(2020)]{JiT20}
Ziwei Ji and Matus Telgarsky.
\newblock
  \href{https://proceedings.neurips.cc/paper/2020/hash/c76e4b2fa54f8506719a5c0dc14c2eb9-Abstract.html}{Directional
  convergence and alignment in deep learning}.
\newblock In \emph{NeurIPS}, 2020.

\bibitem[Jin et~al.(2023)Jin, Li, Lyu, Du, and Lee]{JinLLDL23}
Jikai Jin, Zhiyuan Li, Kaifeng Lyu, Simon~S. Du, and Jason~D. Lee.
\newblock \href{https://proceedings.mlr.press/v202/jin23a.html}{Understanding
  Incremental Learning of Gradient Descent: {A} Fine-grained Analysis of Matrix
  Sensing}.
\newblock In \emph{ICML}, pages 15200--15238, 2023.

\bibitem[Lee et~al.(2022)Lee, Sim, and Ye]{LeeSY22}
Sangmin Lee, Byeongsu Sim, and Jong~Chul Ye.
\newblock \href{\doiurl 10.48550/arxiv.2209.13394}{Magnitude and Angle Dynamics
  in Training Single ReLU Neurons}.
\newblock \emph{CoRR}, abs/2209.13394, 2022.

\bibitem[Li et~al.(2019)Li, Tai, and E]{LiTE19}
Qianxiao Li, Cheng Tai, and Weinan E.
\newblock \href{http://jmlr.org/papers/v20/17-526.html}{Stochastic Modified
  Equations and Dynamics of Stochastic Gradient Algorithms {I:} Mathematical
  Foundations}.
\newblock \emph{J. Mach. Learn. Res.}, 20\penalty0 (40):\penalty0 1--47, 2019.

\bibitem[Li et~al.(2021)Li, Luo, and Lyu]{LiLL21}
Zhiyuan Li, Yuping Luo, and Kaifeng Lyu.
\newblock \href{https://openreview.net/forum?id=AHOs7Sm5H7R}{Towards Resolving
  the Implicit Bias of Gradient Descent for Matrix Factorization: Greedy
  Low-Rank Learning}.
\newblock In \emph{ICLR}, 2021.

\bibitem[Lyu and Li(2020)]{LyuL20}
Kaifeng Lyu and Jian Li.
\newblock \href{https://openreview.net/forum?id=SJeLIgBKPS}{Gradient Descent
  Maximizes the Margin of Homogeneous Neural Networks}.
\newblock In \emph{ICLR}, 2020.

\bibitem[Lyu et~al.(2021)Lyu, Li, Wang, and Arora]{LyuLWA21}
Kaifeng Lyu, Zhiyuan Li, Runzhe Wang, and Sanjeev Arora.
\newblock
  \href{https://proceedings.neurips.cc/paper/2021/hash/6c351da15b5e8a743a21ee96a86e25df-Abstract.html}{Gradient
  Descent on Two-layer Nets: Margin Maximization and Simplicity Bias}.
\newblock In \emph{NeurIPS}, pages 12978--12991, 2021.

\bibitem[Maennel et~al.(2018)Maennel, Bousquet, and Gelly]{MaennelBG18}
Hartmut Maennel, Olivier Bousquet, and Sylvain Gelly.
\newblock \href{http://arxiv.org/abs/1803.08367}{Gradient Descent Quantizes
  ReLU Network Features}.
\newblock \emph{CoRR}, abs/1803.08367, 2018.

\bibitem[Melamed et~al.(2023)Melamed, Yehudai, and Vardi]{MelamedYV23}
Odelia Melamed, Gilad Yehudai, and Gal Vardi.
\newblock \href{\doiurl 10.48550/arxiv.2303.00783}{Adversarial Examples Exist
  in Two-Layer ReLU Networks for Low Dimensional Data Manifolds}.
\newblock \emph{CoRR}, abs/2303.00783, 2023.
\newblock Accepted to NeurIPS 2023.

\bibitem[Min et~al.(2023)Min, Vidal, and Mallada]{MinVM23}
Hancheng Min, Ren{\'{e}} Vidal, and Enrique Mallada.
\newblock \href{\doiurl 10.48550/arxiv.2307.12851}{Early Neuron Alignment in
  Two-layer ReLU Networks with Small Initialization}.
\newblock \emph{CoRR}, abs/2307.12851, 2023.

\bibitem[Neyshabur et~al.(2017)Neyshabur, Bhojanapalli, McAllester, and
  Srebro]{NeyshaburBMS17}
Behnam Neyshabur, Srinadh Bhojanapalli, David McAllester, and Nati Srebro.
\newblock
  \href{https://proceedings.neurips.cc/paper/2017/hash/10ce03a1ed01077e3e289f3e53c72813-Abstract.html}{Exploring
  Generalization in Deep Learning}.
\newblock In \emph{NeurIPS}, pages 5947--5956, 2017.

\bibitem[Ongie et~al.(2020)Ongie, Willett, Soudry, and Srebro]{OngieWSS20}
Greg Ongie, Rebecca Willett, Daniel Soudry, and Nathan Srebro.
\newblock \href{https://openreview.net/forum?id=H1lNPxHKDH}{A Function Space
  View of Bounded Norm Infinite Width ReLU Nets: The Multivariate Case}.
\newblock In \emph{ICLR}, 2020.

\bibitem[Pesme and Flammarion(2023)]{PesmeF23}
Scott Pesme and Nicolas Flammarion.
\newblock \href{\doiurl 10.48550/arxiv.2304.00488}{Saddle-to-Saddle Dynamics in
  Diagonal Linear Networks}.
\newblock \emph{CoRR}, abs/2304.00488, 2023.
\newblock Accepted to NeurIPS 2023.

\bibitem[Phuong and Lampert(2021)]{PhuongL21}
Mary Phuong and Christoph~H. Lampert.
\newblock \href{https://openreview.net/forum?id=krz7T0xU9Z\_}{The inductive
  bias of {ReLU} networks on orthogonally separable data}.
\newblock In \emph{ICLR}, 2021.

\bibitem[Polyak(1963)]{POLYAK1963864}
B.T. Polyak.
\newblock \href{\doiurl 10.1016/0041-5553(63)90382-3}{Gradient methods for the
  minimisation of functionals}.
\newblock \emph{USSR Computational Mathematics and Mathematical Physics},
  3\penalty0 (4):\penalty0 864--878, 1963.

\bibitem[Razin and Cohen(2020)]{RazinC20}
Noam Razin and Nadav Cohen.
\newblock
  \href{https://proceedings.neurips.cc/paper/2020/hash/f21e255f89e0f258accbe4e984eef486-Abstract.html}{Implicit
  Regularization in Deep Learning May Not Be Explainable by Norms}.
\newblock In \emph{NeurIPS}, 2020.

\bibitem[Sarussi et~al.(2021)Sarussi, Brutzkus, and Globerson]{SarussiBG21}
Roei Sarussi, Alon Brutzkus, and Amir Globerson.
\newblock \href{http://proceedings.mlr.press/v139/sarussi21a.html}{Towards
  Understanding Learning in Neural Networks with Linear Teachers}.
\newblock In \emph{ICML}, pages 9313--9322, 2021.

\bibitem[Savarese et~al.(2019)Savarese, Evron, Soudry, and
  Srebro]{SavareseESS19}
Pedro Savarese, Itay Evron, Daniel Soudry, and Nathan Srebro.
\newblock \href{http://proceedings.mlr.press/v99/savarese19a.html}{How do
  infinite width bounded norm networks look in function space?}
\newblock In \emph{COLT}, pages 2667--2690, 2019.

\bibitem[Stewart et~al.(2023)Stewart, Bach, Berthet, and Vert]{stewart23a}
Lawrence Stewart, Francis Bach, Quentin Berthet, and Jean-Philippe Vert.
\newblock \href{https://proceedings.mlr.press/v206/stewart23a.html}{Regression
  as Classification: Influence of Task Formulation on Neural Network Features}.
\newblock In \emph{AISTATS}, pages 11563--11582, 2023.

\bibitem[Timor et~al.(2023)Timor, Vardi, and Shamir]{TimorVS23}
Nadav Timor, Gal Vardi, and Ohad Shamir.
\newblock \href{https://proceedings.mlr.press/v201/timor23a.html}{Implicit
  Regularization Towards Rank Minimization in ReLU Networks}.
\newblock In \emph{ALT}, pages 1429--1459, 2023.

\bibitem[Vardi(2023)]{Vardi23c}
Gal Vardi.
\newblock \href{\doiurl 10.1145/3571070}{On the Implicit Bias in Deep-Learning
  Algorithms}.
\newblock \emph{Commun. {ACM}}, 66\penalty0 (6):\penalty0 86--93, 2023.

\bibitem[Vardi and Shamir(2021)]{VardiS21}
Gal Vardi and Ohad Shamir.
\newblock \href{http://proceedings.mlr.press/v134/vardi21b.html}{Implicit
  Regularization in ReLU Networks with the Square Loss}.
\newblock In \emph{COLT}, pages 4224--4258, 2021.

\bibitem[Vardi et~al.(2021)Vardi, Yehudai, and Shamir]{VardiYS21}
Gal Vardi, Gilad Yehudai, and Ohad Shamir.
\newblock
  \href{https://proceedings.neurips.cc/paper/2021/hash/f0f6cc51dacebe556699ccb45e2d43a8-Abstract.html}{Learning
  a Single Neuron with Bias Using Gradient Descent}.
\newblock In \emph{NeurIPS}, pages 28690--28700, 2021.

\bibitem[Vardi et~al.(2022)Vardi, Yehudai, and Shamir]{VardiYS22}
Gal Vardi, Gilad Yehudai, and Ohad Shamir.
\newblock
  \href{https://proceedings.neurips.cc/paper_files/paper/2022/hash/83e6913572ba09b0ab53c64c016c7d1a-Abstract-Conference.html}{Gradient
  Methods Provably Converge to Non-Robust Networks}.
\newblock In \emph{NeurIPS}, 2022.

\bibitem[Wang and Ma(2022)]{WangM22}
Mingze Wang and Chao Ma.
\newblock
  \href{https://proceedings.neurips.cc//paper_files/paper/2022/hash/04cda3a5ef307978cb5dbef6ab649380-Abstract-Conference.html}{Early
  Stage Convergence and Global Convergence of Training Mildly Parameterized
  Neural Networks}.
\newblock In \emph{NeurIPS}, 2022.

\bibitem[Wang and Pilanci(2022)]{WangP22}
Yifei Wang and Mert Pilanci.
\newblock \href{https://openreview.net/forum?id=5QhUE1qiVC6}{The Convex
  Geometry of Backpropagation: Neural Network Gradient Flows Converge to
  Extreme Points of the Dual Convex Program}.
\newblock In \emph{ICLR}, 2022.

\bibitem[Woodworth et~al.(2020)Woodworth, Gunasekar, Lee, Moroshko, Savarese,
  Golan, Soudry, and Srebro]{WoodworthGLMSGS20}
Blake~E. Woodworth, Suriya Gunasekar, Jason~D. Lee, Edward Moroshko, Pedro
  Savarese, Itay Golan, Daniel Soudry, and Nathan Srebro.
\newblock \href{http://proceedings.mlr.press/v125/woodworth20a.html}{Kernel and
  Rich Regimes in Overparametrized Models}.
\newblock In \emph{COLT}, pages 3635--3673, 2020.

\bibitem[Xu and Du(2023)]{xu23a}
Weihang Xu and Simon Du.
\newblock
  \href{https://proceedings.mlr.press/v195/xu23a.html}{Over-Parameterization
  Exponentially Slows Down Gradient Descent for Learning a Single Neuron}.
\newblock In \emph{COLT}, pages 1155--1198, 2023.

\bibitem[Yehudai and Shamir(2020)]{YehudaiS20}
Gilad Yehudai and Ohad Shamir.
\newblock \href{http://proceedings.mlr.press/v125/yehudai20a.html}{Learning a
  Single Neuron with Gradient Methods}.
\newblock In \emph{COLT}, pages 3756--3786, 2020.

\bibitem[Yun et~al.(2021)Yun, Krishnan, and Mobahi]{YunKM21}
Chulhee Yun, Shankar Krishnan, and Hossein Mobahi.
\newblock \href{https://openreview.net/forum?id=ZsZM-4iMQkH}{A unifying view on
  implicit bias in training linear neural networks}.
\newblock In \emph{ICLR}, 2021.

\bibitem[Zhang et~al.(2021)Zhang, Bengio, Hardt, Recht, and
  Vinyals]{ZhangBHRV21}
Chiyuan Zhang, Samy Bengio, Moritz Hardt, Benjamin Recht, and Oriol Vinyals.
\newblock \href{\doiurl 10.1145/3446776}{Understanding deep learning (still)
  requires rethinking generalization}.
\newblock \emph{Commun. {ACM}}, 64\penalty0 (3):\penalty0 107--115, 2021.

\end{thebibliography}

\ifappendix
\appendix
\clearpage

\tableofcontents

\section{Related work}

Here we further discuss a selection of related work.

\paragraph{Early training phase.}

Our analysis of the first phase of the training builds on those of \citet{MaennelBG18}, who considered unrestricted datasets but focused on asymptotic results; and of \citet{BoursierPF22}, who restricted the datasets to orthogonal but provided detailed non-asymptotic bounds.  In particular, our normalised yardstick trajectories~$\overline{\vec{\alpha}}_j^t$ are analogous to ${\stackrel{{}_{\scriptscriptstyle \rightarrow}}{U}\:\!\!}_i(t)$~in \citet[section~9.6]{MaennelBG18} and $\widetilde{\mathsf{w}}_j^t$~in \citet[Appendix~B.6]{BoursierPF22}.  Our contribution in this part, in relation to these two works, is to extend the fine-grained description for the orthogonal case to the more involved correlated case, which exhibits a sequence of intermediate stages, obtaining detailed non-asymptotic bounds including for the initialisation scale~$\lambda$; the latter are essential for our analysis of the subsequent second phase of the training and our proof of the implicit bias when $\lambda$~tends to zero.

\else\nocite{LyuLWA21,MinVM23,WangM22,xu23a}\fi\ifappendix
Non-asymptotic bounds for early training of one-hidden layer networks were shown by \citet{LyuLWA21} with the Leaky-ReLU non-linearity, logistic loss, and linearly separable data which are either symmetric or have a principal direction and uniformly labelled support vectors; and by \citet{MinVM23} with the ReLU non-linearity, exponential loss, and orthogonally separable data.  \citet{WangM22} studied early training by gradient descent of one-hidden layer ReLU networks, focusing on a non-balanced random initialisation and on obtaining a lower bound for the Euclidean norm of the gradient.  Another related work is by \citet{xu23a}, where the hidden layer is trained, every last-layer weight is fixed to~$1$, and the loss is over a Gaussian data population; consequently some aspects of the training dynamics are simpler, and already the first phase aligns the neurons to the teacher.

\paragraph{Learning a single neuron.}

\else\nocite{YehudaiS20,FreiCG20,VardiYS21,LeeSY22}\fi\ifappendix
A number of previous works studied learning a single neuron by gradient-based algorithms, in settings including realisable without bias~\citep{YehudaiS20}, agnostic and noisy~\citep{FreiCG20}, realisable with bias~\citep{VardiYS21}, multi layer~\citep{LeeSY22}, and overparameterised~\citep{xu23a}.  In particular, \citet{VardiYS21} proved exponentially fast convergence of gradient descent to the global minimum by geometric and algebraic arguments, for two complementary sets of assumptions on the data distribution and the student initialisation; \citet{xu23a} determined that, when the student network has width at least~$2$ and only its hidden layer is trained, the speed of convergence drops to cubic; and \citet{LeeSY22} studied how neuron depth and initialisation scale affect the speed of convergence.  In contrast to the settings in those works, our student network has arbitrary width and both its layers are trained.

\paragraph{Convergence for one-hidden layer ReLU networks and mean square loss.}

\else\nocite{JENTZEN2023126601}\fi\ifappendix
Further related convergence results were obtained by \citet{JENTZEN2023126601}, who proved that, for one-hidden layer ReLU networks trained by gradient flow with respect to a mean square population loss, if the data is one-dimensional, the target function is affine, the initial loss is sufficiently small, and the training trajectory is bounded, then it converges to zero loss; and that if moreover the network is width-one then the boundedness assumption is not needed.

\else\nocite{stewart23a}\fi\ifappendix
Also with univariate data, \citet{stewart23a} compared the features learnt by one-hidden layer ReLU networks for the square loss and the cross-entropy loss, postulating that sparseness in the regression case may cause optimisation difficulties, and reporting synthetic experiments that support the claim.

\paragraph{Properties transferred from parameter space to function space.}

\else\nocite{SavareseESS19,ErgenP21,OngieWSS20,BoursierF23}\fi\ifappendix
The implicit bias in parameter space that we established, namely to interpolator networks of rank~$1$, has a clear implication in function space, namely the resulting function is identical to that defined by the teacher neuron.  However, we also compared that set of interpolator networks with the one obtained by minimising the Euclidean norm.  The question of what functions the latter networks define is in general non-trivial; for one-hidden layer ReLU networks, it was studied in the univariate case by \citet{SavareseESS19} and \citet{ErgenP21}, and in the multivariate case by \citet{OngieWSS20}.  More recently, \citet{BoursierF23} investigated further the univariate case, elucidating the consequences of whether the norm takes into account the bias terms.

\paragraph{Regression using diagonal linear networks.}

\else\nocite{EvenPGF23,PesmeF23}\fi\ifappendix
\citet{EvenPGF23} and \cite{PesmeF23} considered implicit bias for regression tasks, focusing on diagonal networks with linear activation.  The former proved convergence and compared implicit biases of gradient descent and stochastic gradient descent with large learning rates.  The latter studied gradient flow from a vanishing initialisation and provided a full description of training trajectories, showing that they jump from saddle to saddle until reaching the minimum $\ell_1$-norm solution.

\paragraph{Classification using Leaky-ReLU networks.}

\else\nocite{SarussiBG21,frei2023implicit,FreiVBS23Benign}\fi\ifappendix
Building on the implicit bias to margin maximisation in parameter space for homogenous networks~\citep{LyuL20,JiT20}, convergence to a linear classifier was shown by \citet{LyuLWA21}, \citet{SarussiBG21}, and \citet{frei2023implicit} for one-hidden layer networks with the Leaky-ReLU activation and several sets of assumptions that include linear separability of the data.  \citet{FreiVBS23Benign} subsequently established that in two kinds of distributional settings benign overfitting occurs, namely the predictors interpolate noisy training data and simultaneously generalise well to unseen test data.

\paragraph{Implicit bias and adversarial examples.}

\else\nocite{Englert022,FreiVBS23Double,MelamedYV23}\fi\ifappendix
In addition to \citet{VardiYS21}, likewise in the context of one-hidden layer ReLU networks and exponentially-tailed loss functions, consequences for adversarial examples of the implicit bias to margin maximisation in parameter space were investigated by \citet{Englert022}, who showed that for orthogonally separable training datasets it may prevent adversarial reprogrammability; and by \citet{FreiVBS23Double}, who showed that for clustered data it leads to non-robust solutions even though robust networks that fit the data exist.  In contrast to those works, which focus on gradient flow, \citet{MelamedYV23} considered possibly stochastic gradient descent and established that, when data belongs to a low-dimensional linear subspace, the training produces non-robust solutions, but decreasing the initialisation scale or adding a Euclidean norm regulariser increases robustness to orthogonal adversarial perturbations.

\section{Proof for the preliminaries}

First we note that the gradient flow ensures that the loss monotonically decreases, at the rate equal to the square of the Euclidean norm of the gradient.

\else\nocite{DavisDKL20}\fi\ifappendix
\begin{proposition}[by {\citet[Lemma~5.2]{DavisDKL20}}]
\label{pr:loss}
For almost all $t \in [0, \infty)$ we have
$\mathrm{d} L(\vec{\theta}^t) / \mathrm{d} t
= -\|\mathrm{d} \vec{\theta}^t / \mathrm{d} t\|^2$.
\end{proposition}

Then we show the following proposition from \autoref{s:prelim}.

\prprelim*

\begin{proof}
Part~\ref{pr:prelim.gj} follows by straightforward calculations.  For part~\ref{pr:prelim.balanced}, by~\ref{pr:prelim.gj}, for all $j \in [m]$ and almost all $t \in [0, \infty)$ we have
\[\mathrm{d} (a_j^t)^2 / \mathrm{d} t =
  2 a_j^t {\vec{w}_j^t}^\top \vec{g}_j^t =
  \mathrm{d} \|\vec{w}_j^t\|^2 / \mathrm{d} t \;,\]
and so $(a_j^t)^2 - \|\vec{w}_j^t\|^2$ is constant and therefore zero for all~$t$ by the initialisation and continuity.  It remains to show that $a_j^t \neq 0$ for all~$t$.  Observe that, by \autoref{pr:loss}, for almost all~$t$, provided $a_j^t \neq 0$ we have
\[\mathrm{d} \ln (a_j^t)^2 / \mathrm{d} t
  \geq -2 \|\vec{g}_j^t\|
  \geq -2 \sqrt{2 L(\vec{\theta}^t)} \max_{i = 1}^n \|\vec{x}_i\|
  \geq -2 \sqrt{2 L(\vec{\theta}^0)} \max_{i = 1}^n \|\vec{x}_i\| \;.\]
Hence, by continuity, for all~$t$ we have
\[(a_j^t)^2 \geq
  (a_j^0)^2
  \exp\!\left(\!-2 t \sqrt{2 L(\vec{\theta}^0)}
                     \max_{i = 1}^n \|\vec{x}_i\|\!\right)\! \;.
  \qedhere\]
\end{proof}

Also from \autoref{pr:prelim} we obtain the formulas for the derivatives of the spherical coordinates of the hidden neurons, i.e.~of their logarithmic Euclidean norms and their unit normalisations.

\begin{corollary}
\label{cor:prelim}
For all $j \in [m]$ and almost all $t \in [0, \infty)$ we have
\begin{align*}
\mathrm{d} \ln \|\vec{w}_j^t\| / \mathrm{d} t
& = s_j {\overline{\vec{w}}_j^t}^\top \vec{g}_j^t
&
\mathrm{d} \overline{\vec{w}}_j^t / \mathrm{d} t
& = s_j (\vec{g}_j^t - \overline{\vec{w}}_j^t \,
                       {\overline{\vec{w}}_j^t}^\top \vec{g}_j^t) \;.
\end{align*}
\end{corollary}

\section{Proofs for the assumptions}

In this section we prove \autoref{main:pr:omega} from \autoref{s:ass}, which is subsumed by \autoref{pr:omega} below, and then show as \autoref{cor:measure} that the cases excluded by \autoref{ass:enum}~\ref{ass:enum.omega} have Lebesgue measure zero, as claimed in \autoref{s:ass}.

\newcommand{\der}{\mathrm d}
\newcommand{\tra}{{}^\top}
\newcommand{\myphi}{\varphi_j}
\newcommand{\mygamma}{\vec{\gamma}_{I_j^\ell}}
\newcommand{\actgamma}{\vec{\gamma}_{I_\ppp(\curomega)}}
\newcommand{\actgammatra}{(\actgamma)\!\tra}
\newcommand{\lastgamma}{\vec{\gamma}_{I_\ppp(\lastomega)}}
\newcommand{\newgamma}{\vec{\gamma}}
\newcommand{\norgamma}{\overline{\vec{\gamma}}_{I_j^\ell}}
\newcommand{\normygamma}{\overline{\vec{\gamma}}}
\newcommand{\normygammafix}{\overline{\vec{\gamma}}_{I_j^\ell}}
\newcommand{\nexttime}{{\tau_j^\ell}}
\newcommand{\lasttime}{{\tau_j^{\ell-1}}}
\newcommand{\limlasttime}{\varphi_j^{(\ell - 1)^+}}
\newcommand{\limnexttime}{\varphi_j^{\ell^-}}
\newcommand{\nextphi}{\myphi^{\nexttime}}
\newcommand{\lastphi}{\myphi^{\lasttime}}
\newcommand{\curphi}{\myphi^t}
\newcommand{\phileft}{\myphi^{t_1^+}}
\newcommand{\curomega}{\vec{\alpha}_j^t}
\newcommand{\curomeganought}{\vec{\alpha}_j^{t_0}}
\newcommand{\curomegatra}{(\vec{\alpha}_j^t)\!\tra}
\newcommand{\curomegaprimetra}{(\vec{\alpha}_j^{t'})\!\tra}
\newcommand{\curomeganoughttra}{(\vec{\alpha}_j^{t_0})\!\tra}
\newcommand{\lastomega}{\vec{\alpha}_j^{\lasttime}}
\newcommand{\nextomega}{\vec{\alpha}_j^{\nexttime}}
\newcommand{\omegaleft}{\vec{\alpha}_j^{t_1}}
\newcommand{\leftomega}{\omegaleft}
\newcommand{\noromegaleft}{\overline{\vec{\alpha}}_j^{t_1}}
\newcommand{\enorm}[1]{\|#1\|}
\newcommand{\mys}{s_j}
\newcommand{\noromega}{\overline{\vec{\alpha}}_j^t}
\newcommand{\noromegatra}{(\overline{\vec{\alpha}}_j^t)\!\tra}
\newcommand{\nordpi}{\overline{\vec{x}}_i}
\newcommand{\nordplasti}{\overline{\vec{x}}_{i_j^{\ell - 1}}}
\newcommand{\nordpnexti}{\overline{\vec{x}}_{i_j^\ell}}
\newcommand{\dpi}{\vec{x}_i}
\newcommand{\dpitra}{\vec{x}_i^\top}
\newcommand{\dplasti}{\vec{x}_{i_j^{\ell-1}}}
\newcommand{\dpnexti}{\vec{x}_{i_j^{\ell}}}
\newcommand{\dpnextifix}{\overline{\vec{x}}_{i_j^{\ell}}}
\newcommand{\veczero}{\vec 0}
\newcommand{\zerovec}{\veczero}

We first establish several elementary properties of the yardstick trajectories $\curomega$.

\begin{proposition}
\label{pr:omega-elementary}
For all $j \in J_\ppp \cup J_\mmm$, the following statements hold.
\begin{enumerate}[(i),itemsep=0ex,leftmargin=3em]
\item
\label{pr:omega-elementary.gamma=zero}
$\actgamma = \zerovec$ if and only if $I_\ppp(\curomega) = \emptyset$.
\item
\label{pr:omega-elementary.omega.gamma}
$\curomegatra\actgamma \ge 0$. Moreover,
$\curomegatra\actgamma = 0$ if and only if $\curomegatra\dpi \le 0$ for all $i \in [n]$.
\item
\label{pr:omega-elementary.omega=gamma}
If $\curomegatra\actgamma = \enorm{\curomega}\enorm{\actgamma}$ for some $t$,
then $I_\ppp(\curomega) = \emptyset$ or $I_\ppp(\curomega) = [n]$.
\item
\label{pr:omega-elementary.nonzero}
$\curomega \ne \veczero$ for all $t \in [0, \infty)$.
\end{enumerate}
\end{proposition}

\begin{proof}
For part~\ref{pr:omega-elementary.gamma=zero},
only the left-to-right implication requires a proof.
Let $\actgamma = \frac{1}{n} \sum_{i} y_i \dpi$, where the summation is over
$i \in I_\ppp(\curomega)$; then
$\enorm{\actgamma}^2 = \frac{1}{n^2} \sum_{i,k} y_i y_k \dpitra\vec{x}_k$.
Since all training points are positively correlated, and
all coefficients $y_i$ that are present in the sum are positive,
$\enorm{\actgamma}^2 > 0$ unless the sum contains no terms;
that is, unless $I_\ppp(\curomega) = \emptyset$.

For part~\ref{pr:omega-elementary.omega.gamma},
observe that $\actgamma$ is a positive linear combination of $\vec x_k$ for
$k \in I_\ppp(\curomega)$, and every such~$\vec x_k$ forms an acute angle with $\curomega$;
thus, $\curomegatra\actgamma > 0$ unless $\actgamma = \zerovec$.

For part~\ref{pr:omega-elementary.omega=gamma},
we again use the definition of $\actgamma$ and the fact that
$\dpitra\vec{x}_k > 0$ for all $i, k \in [n]$.
If $I_\ppp(\curomega) \ne \emptyset$, then $\curomega \neq \zerovec$ and
$\dpitra\curomega = \dpitra\actgamma \cdot \enorm{\curomega} / \enorm{\actgamma} > 0$ for all $i$, and thus $I_\ppp(\curomega) = [n]$.

For part~\ref{pr:omega-elementary.nonzero}, note that
\begin{equation*}
\frac{\der \enorm{\curomega}^2}{\der t} =
2 \curomegatra \cdot \mys \enorm{\curomega} \actgamma =
\enorm{\curomega}^2 \enorm{\actgamma} \cdot 2 \mys \cos\angle(\curomega, \actgamma) \;,
\end{equation*}
where the final product is $0$ if $\curomega = \veczero$.
It follows that
$\der \enorm{\curomega}^2 / \der t \ge -2 \enorm\curomega^2 \enorm{\vec{\gamma}_{[n]}}$ for all $t \in [0, \infty)$.
By Gr\"onwall's inequality,
$\enorm{\curomega}^2 \ge \enorm{\vec\alpha_j^0}^2 \, \e^{- 2 \enorm{\vec{\gamma}_{[n]}} t} > 0$ for all $t$,
and it remains to recall that $\vec\alpha_j^0 = \vec z_j$ was initially chosen to be non-zero.
\end{proof}

The next proposition establishes continuity and monotonicity properties,
assuming that the trajectory~$\curomega$ crosses between regions of continuity
of the right-hand side of the ODE finitely many times.

\begin{proposition}
\label{pr:omega-cont}
Let $j \in J_\ppp \cup J_\mmm$.
Let $T > 0$ be such that $I_0(\curomega)$ is only non-empty for finitely many $t \in [0, T]$.
Then the following statements hold.
\begin{enumerate}[(i),itemsep=0ex,leftmargin=3em]
\item
\label{pr:omega-cont.c-omega.gamma}
The map $t \mapsto \curomegatra\actgamma$ from $[0, T]$ to $\mathbb R$ is continuous.
\item
\label{pr:omega-cont.zero-at-end}
If $\actgamma = \veczero$ for some $t \in [0, T]$, then $t = T$.
\item
\label{pr:omega-cont.c-gamma}
$
\actgamma =
\begin{cases}
\lim_{\xi \to t ^-} \vec{\gamma}_{I_\ppp(\vec{\alpha}_j^\xi)} & \text{if $\mys = +1$ and $t \in (0, T]$,} \\
\lim_{\xi \to t ^+} \vec{\gamma}_{I_\ppp(\vec{\alpha}_j^\xi)} & \text{if $\mys = -1$ and $t \in [0, T)$.}
\end{cases}
$
\item
\label{pr:omega-cont.mon}
For each $i \in [n]$,
$\mys \curomegatra\dpi$ is a strictly increasing function of $t$ on $[0, T]$.
Furthermore, for all $i \in [n]$, whenever $0 \le t_1 < t_2 \le T$:
\begin{itemize}
\item if $\mys = +1$ and $i \in I_0(\vec{\alpha}_j^{t_1}) \cup I_\ppp(\vec{\alpha}_j^{t_1})$, then
                         $i \in I_\ppp(\vec{\alpha}_j^{t_2})$;
\item if $\mys = -1$ and $i \not\in I_\ppp(\vec{\alpha}_j^{t_1})$, then
                         $i \not\in I_0(\vec{\alpha}_j^{t_2}) \cup I_\ppp(\vec{\alpha}_j^{t_2})$.
\end{itemize}
\item
\label{pr:omega-cont.no-jump-to-1}
Let $t_0 \in [0, T)$ be either $0$ or a point of discontinuity of $\actgamma$.
If $\mys = -1$, then
$\lim_{t \to t_0^+} \cos\curphi \ne 1$.
\end{enumerate}
\end{proposition}

We remark that the assumption of \autoref{pr:omega-cont}
that the set $\{ t \ge 0 \mid I_0(\curomega) \ne \emptyset \}$ has a finite
intersection with $[0, T]$
will be justified in \autoref{pr:omega}, in the following sense:
we will show that the assumption
holds for every segment $[0, T]$ such that $\actgamma \ne \veczero$
for all $t \in [0, T)$.

\begin{proof}
For part~\ref{pr:omega-cont.c-omega.gamma}, since
the map $t \mapsto \curomega$ from $[0, \infty)$ to $\mathbb R^d$ is continuous by definition, it suffices to consider
points of discontinuity of $\actgamma$.
The set $I_0(\curomeganought)$ is necessarily non-empty at every such point $t_0 \in [0, T]$.
Let us fix such a $t_0$; then one-sided limits of $\actgamma$ as $t \to t_0^+$ and $t \to t_0^-$ exist
(excepting $t \to 0^-$ and $t \to T^+$, which we do not consider).
Here we used the assumption from the statement of the proposition: in a small enough neighbourhood of $t_0$
there are no other points $t$ for which $I_0(\curomega)$ is non-empty; this assumption
could have been avoided if necessary.
The value of
$\curomeganoughttra\newgamma_{I_\ppp(\vec{\alpha}_j^{t_0})}$
may only differ from (either of) the one-sided limits of
$\curomegatra\actgamma$ as $t \to t_0^\pm$
by the summands
$\curomeganoughttra \cdot \frac{1}{n} y_i \dpi$ with $i \in I_0(\vec{\alpha}_j^{t_0})$.
But every such summand is equal to $0$ anyway by the definition of $I_0$.

Before establishing part~\ref{pr:omega-cont.zero-at-end},
we first prove a weaker version of part~\ref{pr:omega-cont.mon}, namely non-strict monotonicity:
for each $i \in [n]$,
$\mys \curomegatra\dpi$ is a non-decreasing function of $t$ on $[0, T]$.
To this end, we consider the derivatives
\begin{equation*}
\frac{\der \curomegatra \dpi}{\der t} = \mys \enorm{\curomega} \cdot \actgammatra\dpi
\end{equation*}
inside each interval $(t_1, t_2)$ on which $\actgamma$ is continuous.
Observe that $\actgammatra\dpi \ge 0$ for each $i \in [n]$,
since training points are pairwise positively correlated.
Therefore, $\mys \curomegatra\dpi$ is non-decreasing on $(t_1, t_2)$.
Since $\curomega$ is continuous, and there are only finitely many points at which
$\actgamma$ is discontinuous, $\mys \curomegatra\dpi$ is also non-decreasing
on $[0, T]$.

We now establish part~\ref{pr:omega-cont.zero-at-end}.
If $\actgamma = \veczero$ for some $t \in [0, T]$, then
\[t_0 \coloneqq \inf \{ t \in [0, T] \mid I_\ppp(\curomega) = \emptyset \}\]
is well-defined by \autoref{pr:omega-elementary}, part~\ref{pr:omega-elementary.gamma=zero}.
Observe that if $I_\ppp(\curomeganought)$ is non-empty, then
$I_\ppp(\curomega)$ is non-empty for all $t$ in a small neighbourhood of $t_0$
by the continuity of $\curomega$.
Therefore, $I_\ppp(\curomeganought) = \emptyset$ by our choice of $t_0$.
At the same time, notice that $0 < t_0$ because we assume $j \in J_\ppp \cup J_\mmm$.
Furthermore, for all $t' < t_0$ there is some $i \in [n]$ with $\curomegaprimetra \dpi > 0$.
By compactness and by the (non-strict) monotonicity property proved above,
there exists a single $i \in [n]$ such that $\curomegaprimetra\dpi > 0$
for all $t' < t_0$.
Once again by continuity, we have $(\curomeganought)\tra\dpi \ge 0$.
Since $I_\ppp(\curomeganought) = \emptyset$, we conclude that
$i \in I_0(\curomeganought)$.

We have shown that, assuming
$\actgamma = \veczero$ for some $t \in [0, T]$,
the existence of $t_0 \in (0, t]$ such that
$I_\ppp(\curomeganought) = \emptyset$ and $I_0(\curomeganought) \ne \emptyset$.
It follows that $\curomega = \curomeganought$ and $I_0(\curomega) = I_0(\curomeganought)$
for all $t \ge t_0$. Under the assumptions of the proposition,
this means $t_0 = t = T$. This completes the proof of part~\ref{pr:omega-cont.zero-at-end}.

We now proceed to part~\ref{pr:omega-cont.mon},
proving \emph{strict} monotonicity of each $\mys \curomegatra \dpi$.
By part~\ref{pr:omega-cont.zero-at-end},
for every interval $(t_1, t_2)$ on which $\actgamma$ is continuous,
we have in fact $\actgamma \ne \veczero$.
Therefore, $\actgammatra\dpi > 0$, because training points are pairwise positively correlated,
and each $\mys \curomegatra\dpi$ strictly increases on $(t_1, t_2)$.
Since $\curomega$ is continuous, and there are only finitely many points at which
$\actgamma$ is discontinuous, $\mys \curomegatra\dpi$ is also strictly increasing
on $[0, T]$.
The two remaining implications in the statement of part~\ref{pr:omega-cont.mon} follow.

Part~\ref{pr:omega-cont.c-gamma} is a consequence of part~\ref{pr:omega-cont.mon}.

To establish part~\ref{pr:omega-cont.no-jump-to-1},
firstly observe that,
by part~\ref{pr:omega-cont.c-gamma}
and by the continuity of $\curomega$,
the function $\cos\curphi$ is in fact right-continuous at $t_0$:
the one-sided limit in question exists and is equal to $\cos\myphi^{t_0}$.
By \autoref{pr:omega-elementary}, part~\ref{pr:omega-elementary.omega=gamma},
if $\cos\myphi^{t_0} = 1$, then
$I_\ppp(\curomeganought) = [n]$.
We now consider two cases.
If $t_0 = 0$, then $\curomeganought = \vec z_j$,
but this is ruled out by \autoref{ass:enum}, part~\ref{ass:enum.init}.
Otherwise $t_0$ is a point of discontinuity of~$\actgamma$.
Since $\curomega$ is a continuous function of $t$,
the set $I_0(\curomeganought)$ is necessarily non-empty,
but this is also a contradiction because this set must be disjoint from
$I_\ppp(\curomeganought)$.
\end{proof}

The following proposition strengthens the previously proved statement
(\autoref{pr:omega-elementary}, part~\ref{pr:omega-elementary.nonzero})
that $\curomega \ne \veczero$, bounding $\enorm{\curomega}$ from below.
In the sequel, we will
require analytic expressions for two related quantities:
\begin{align*}
\frac{\der{\enorm{\curomega}}}{\der t} &=
\frac{\der \sqrt{\enorm{\curomega}^2}}{\der t} =
\frac{1}{2 \enorm{\curomega}} \cdot \frac{\der \enorm{\curomega}^2}{\der t} =
\frac{2 \curomegatra \cdot \mys \enorm{\curomega} \actgamma}{2 \enorm{\curomega}} =
\mys \curomegatra \actgamma
\;, \\
\frac{\der}{\der t} \!\left(\frac{\curomega}{\enorm\curomega}\right)\! &=
\frac{\frac{\der \curomega}{\der t} \cdot \enorm\curomega -
      \frac{\der\enorm{\curomega}}{\der t} \cdot \curomega}%
     {\enorm\curomega^2} =
\frac{\mys \enorm\curomega^2 \, \actgamma - \mys \cdot \curomegatra\actgamma \cdot \curomega}%
     {\enorm\curomega^2}
\;.
\end{align*}

\begin{proposition}
\label{pr:omega-norm}
Let $0 \le t_1 < t_2$ be such that
$I_0(\curomega) = \emptyset$
for all
$t \in (t_1, t_2)$
and $I_0(\curomega) \neq \emptyset$ for at most finitely many $t \in [0, t_1]$.
\begin{enumerate}[(i),itemsep=0ex,leftmargin=3em]
\item
\label{pr:omega-norm.solve-bound}
If $t_1$ is either $0$ or a point of discontinuity of $\actgamma$,
then there is a constant $\mu > 0$, only dependent on $t_1$ but not on
$t$ or $t_2$, such that $ \enorm{\curomega} \ge \mu $ for all $t \in (t_1, t_2)$.
\item
\label{pr:omega-norm.solve-analytic}
Suppose $\actgamma = \newgamma$ for all
$t \in (t_1, t_2)$, and denote
$\phileft \coloneqq \angle(\omegaleft, \newgamma) = \arccos\bigl((\noromegaleft)\!\tra\normygamma\bigr)$.
Then for all $t \in (t_1, t_2)$ we have
\begin{align*}
\enorm{\curomega} =
{}&\tfrac{1}{2} \cdot (1 + \mys \cos\phileft) \cdot \enorm\omegaleft \cdot \e^{\,\enorm\newgamma (t - t_1)} + {}\\
{}&\tfrac{1}{2} \cdot (1 - \mys \cos\phileft) \cdot \enorm\omegaleft \cdot \e^{- \enorm\newgamma (t - t_1)} \;.
\end{align*}
\end{enumerate}
\end{proposition}

\begin{proof}
We establish part~\ref{pr:omega-norm.solve-analytic} first.
The functions $\enorm\curomega$ and $\curomegatra\newgamma$ satisfy,
for all $t \in (t_1, t_2)$,
the following system of ordinary differential equations:
\begin{align*}
\frac{\der \enorm\curomega}{\der t} &= \mys \cdot \curomegatra\newgamma \;,
&
\frac{\der \!\left(\curomegatra\newgamma\right)\!}{\der t} &= \mys \enorm\newgamma^2 \enorm\curomega \;.
\end{align*}
By the standard theory of linear ODEs, the solution can be sought in the form
\begin{align*}
\enorm\curomega        &=     c_1 \, \e^{\enorm\newgamma t} + c_2 \, \e^{- \enorm\newgamma t} \;, \\
\curomegatra\newgamma  &=   ( c_1 \, \e^{\enorm\newgamma t} - c_2 \, \e^{- \enorm\newgamma t} ) \cdot \mys \enorm\newgamma \;.
\end{align*}
The constants $c_1$ and $c_2$ are chosen based on the initial conditions
as $t \to t_1 ^+$, i.e., they should satisfy the following system of linear equations:
\begin{equation*}
\begin{bmatrix}
\e^{\enorm\newgamma t_1} &\phantom{-}\e^{-\enorm\newgamma t_1} \\
\e^{\enorm\newgamma t_1} &         - \e^{-\enorm\newgamma t_1}
\end{bmatrix}
\begin{bmatrix} c_1 \vphantom{\e^{t_1}} \\ c_2 \vphantom{\e^{t_1}} \end{bmatrix}
=
\begin{bmatrix} \enorm\omegaleft \\[1ex] \mys (\leftomega)\!\tra\normygamma \end{bmatrix}
\:.
\end{equation*}
Here we rely on the continuity of $t \mapsto \curomega$.
We obtain
\begin{align*}
\enorm{\curomega} =
{}&\tfrac{1}{2} \cdot \Bigl(\enorm\leftomega + \mys (\leftomega)\!\tra\normygamma\Bigr) \cdot \e^{  \enorm\newgamma (t - t_1)} + {}\\
{}&\tfrac{1}{2} \cdot \Bigl(\enorm\leftomega - \mys (\leftomega)\!\tra\normygamma\Bigr) \cdot \e^{- \enorm\newgamma (t - t_1)}\;,
\end{align*}
which can then be rewritten in the required form.

We now establish part~\ref{pr:omega-norm.solve-bound}.
By the continuity of dot products $\curomegatra \dpi$,
there exists a vector $\newgamma \in \mathbb R^d$ such that
$\actgamma = \newgamma$ for all $t \in (t_1, t_2)$.
In the degenerate case, $\newgamma = \zerovec$,
we have $\curomega = \omegaleft$ for all $t \in (t_1, t_2)$.
Hence, we can choose $\mu \coloneqq \enorm{\omegaleft}$, which is positive
by \autoref{pr:omega-elementary}, part~\ref{pr:omega-elementary.nonzero}.
We will therefore assume $\newgamma \ne \zerovec$.
The idea is to rely on \autoref{pr:omega-norm}, part~\ref{pr:omega-norm.solve-analytic},
noting that $\cos\phileft = \lim_{t \to t_1 ^+} \noromegatra\normygamma =
(\noromegaleft)\!\tra\normygamma \ge 0$,
by \autoref{pr:omega-elementary}, part~\ref{pr:omega-elementary.omega.gamma},
and by continuity of $\curomega$. 
So, if $\mys = +1$, then clearly $\enorm\curomega \ge \frac{1}{2} \enorm\omegaleft \eqqcolon \mu$.
If $\mys = -1$, then, again dropping the second term in
the closed-form expression for $\enorm\curomega$, we obtain
$\enorm\curomega \ge \frac{1}{2} \enorm\omegaleft \cdot (1 - \lim_{t \to t_1 ^+} \cos\curphi) \cdot 1$.
By \autoref{pr:omega-cont}, part~\ref{pr:omega-cont.no-jump-to-1},
$\lim_{t \to t_1 ^+} \cos\curphi < 1$, which completes the proof.
\end{proof}

\begin{proposition}
\label{pr:omega}
For all $j \in J_\ppp \cup J_\mmm$ there exist a unique enumeration $i_j^1, \ldots, i_j^{n_j}$ of $I_{-s_j}(\vec{z}_j)$ and unique $\tau_j^1, \ldots, \tau_j^{n_j} \in [0, \infty)$ such that for all $\ell \in [n_j]$ the following hold, where
$\tau_j^0 \coloneqq 0$,
$\varphi_j^{(\ell - 1)^+} \coloneqq
 \lim_{t \to (\tau_j^{\ell - 1})^+}
 \varphi_j^t$,
$\varphi_j^{\ell^-} \coloneqq
 \lim_{t \to (\tau_j^\ell)^-}
 \varphi_j^t$, and
\[I_j^\ell \coloneqq
  \begin{cases}
  I_\ppp(\vec{z}_j) \cup \{i_j^1, \ldots, i_j^{\ell - 1}\}
  & \text{if $s_j = 1$,} \\[1.33ex]
  I_\ppp(\vec{z}_j) \setminus \{i_j^1, \ldots, i_j^{\ell - 1}\}
  & \text{if $s_j = -1$:}
  \end{cases}\]
\begin{enumerate}[(i),itemsep=0ex,leftmargin=3em]
\item
\label{pr:omega.argmin}
$i_j^\ell =
 \argmin\!
 \left\{\!
-s_j {\left(\!\overline{\vec{\alpha}}_j^{\tau_j^{\ell - 1}}\right)\!\!}^\top \!
 \overline{\vec{x}}_i
 \Big/ \,
 {\overline{\vec{\gamma}}_{I_j^\ell}\!}^\top \,
 \overline{\vec{x}}_i
 \,\,\,\middle\vert\,\,\,
 i \,\in\, I_{-s_j}(\vec{z}_j) \setminus \{i_j^1, \ldots, i_j^{\ell - 1}\}
 \!\right\}$;
\item
\label{pr:omega.sin}
$\sin \!\left(\varphi_j^{(\ell - 1)^+} \! - \varphi_j^{\ell^-}\right)\!
 \Big/ \!
 \sin \varphi_j^{\ell^-} \! =
-{\!\left(\!\overline{\vec{\alpha}}_j^{\tau_j^{\ell - 1}}\right)\!\!}^\top \!
 \overline{\vec{x}}_{i_j^\ell}
 \Big/ \,
 {\overline{\vec{\gamma}}_{I_j^\ell}\!}^\top \,
 \overline{\vec{x}}_{i_j^\ell}$;
\item
\label{pr:omega.tau}
$\tau_j^{\ell - 1} < \tau_j^\ell$;
\item
\label{pr:omega.Ip}
$I_\ppp(\vec{\alpha}_j^t) = I_j^\ell$
for all $t \in (\tau_j^{\ell - 1}, \tau_j^\ell)$;
\item
\label{pr:omega.I0}
$I_0(\vec{\alpha}_j^t) = \emptyset$
for all $t \in (\tau_j^{\ell - 1}, \tau_j^\ell)$, and
$I_0\!\left(\!\vec{\alpha}_j^{\tau_j^\ell}\right)\! = \{i_j^\ell\}$;
\item
\label{pr:omega.cos}
$\cos \varphi_j^t =
 \tanh\!\left(\!
 \artanh \cos \varphi_j^{(\ell - 1)^+} \! +
 s_j \left\|\vec{\gamma}_{I_j^\ell}\right\| (t - \tau_j^{\ell - 1})
 \!\right)$
for all $t \in (\tau_j^{\ell - 1}, \tau_j^\ell)$;
\item
\label{pr:omega.inter}
if $\ell < n_j$ then
$\left\|\vec{\gamma}_{I_j^\ell}\right\|
 \cos \varphi_j^{\ell^-} =
 \left\|\vec{\gamma}_{I_j^{\ell + 1}}\right\|
 \cos \varphi_j^{\ell^+}$;
\item
\label{pr:omega.finalp}
if $\ell = n_j$ and $s_j = 1$ then
$\left\|\vec{\gamma}_{I_j^\ell}\right\|
 \cos \varphi_j^{\ell^-} =
 \|\vec{\gamma}_{[n]}\|
 \cos \lim_{t \to (\tau_j^\ell)^+} \varphi_j^t$;
\item
\label{pr:omega.finalm}
if $\ell = n_j$ and $s_j = -1$ then
$\varphi_j^{\ell^-} = \pi / 2$;
\item
\label{pr:omega.omega}
$\overline{\vec{\alpha}}_j^t =
 \!\left(\!
 \sin (\varphi_j^t) \,
 \overline{\vec{\alpha}}_j^{\tau_j^{\ell - 1}} \! +
 \sin \!\left(\varphi_j^{(\ell - 1)^+} \! - \varphi_j^t\right)\! \,
 \overline{\vec{\gamma}}_{I_j^\ell}
 \!\right)\!
 \Big/ \! \sin \varphi_j^{(\ell - 1)^+}$
for all $t \in (\tau_j^{\ell - 1}, \tau_j^\ell)$;
\item
\label{pr:omega.dox.compl}
$s_j \, \mathrm{d} \,
 {\overline{\vec{\alpha}}_j^t}^\top \, \overline{\vec{x}}_i
 / \mathrm{d} t
 \geq {\vec{\gamma}_{I_j^\ell}\!}^\top \, \overline{\vec{x}}_i$
for all $i \notin I_j^\ell$
and all $t \in (\tau_j^{\ell - 1}, \tau_j^\ell)$;
\item
\label{pr:omega.dox.ddox}
if $s_j = -1$, then
$\mathrm{d} \,
 {\overline{\vec{\alpha}}_j^t}^\top \, \overline{\vec{x}}_{i_j^\ell}
 / \mathrm{d} t
 < 0$
and
$\mathrm{d}^2 \,
 {\overline{\vec{\alpha}}_j^t}^\top \, \overline{\vec{x}}_{i_j^\ell}
 / \mathrm{d} t^2
 < 0$
for all $t \in (\tau_j^{\ell - 1}, \tau_j^\ell)$.
\end{enumerate}
\end{proposition}

\begin{proof}
Throughout, we let $j \in J_\ppp \cup J_\mmm$ stay fixed but arbitrary.

\subparagraph{Parts~\ref{pr:omega.argmin},
\ref{pr:omega.tau},
\ref{pr:omega.Ip}, and
\ref{pr:omega.I0}.}
We establish these parts by
a common inductive argument.
The induction is on~$\ell$, ranging from $1$ to $n_j$.
We do not separate the base case.
We first notice that $\der(\curomegatra\nordpi) / \der t = \mys \enorm\curomega \cdot \actgammatra\nordpi$.
In particular, for any fixed~$i \in [n]$ we can write
\begin{equation}
\mys (\vec\alpha_j^\xi)\!\tra\nordpi =
\mys {\left(\!\lastomega\right)\!\!}^\top\!\nordpi +
\int_{\lasttime}^{\xi} \enorm\curomega \cdot \actgammatra\nordpi \,\der t \;.
\label{eq:sox}
\end{equation}
This equality holds for any $\xi > \lasttime$ as long as the integral on the right-hand side
is well-defined;
we first need to justify the existence of an appropriate $\xi > \lasttime$.
This is not automatic because the function~$\actgamma$ is not assumed continuous.
We consider two cases.

If $\ell = 1$, then $I_0\!\left(\!\lastomega\right)\! = \emptyset$ by \autoref{ass:enum}, part~\ref{ass:enum.init},
and thus ${\!\left(\!\lastomega\right)\!\!}^\top\!\nordpi$ are all non-zero;
by continuity of $\curomega$, this holds in a sufficiently small right-neighbourhood
of $\lasttime$. Thus, if the set $\{ t > \lasttime \mid I_0(\curomega) \ne \emptyset\}$
is non-empty, its infimum is strictly greater than $\lasttime$, and we can pick
this infimum as $\xi$.
In fact, we will show below that the set cannot be empty, but for now
let us say that it is safe to pick any $\xi > \lasttime$ in this hypothetical situation.

Now suppose $\ell > 1$, then
by the inductive hypothesis
$I_0\!\left(\!\vec\alpha_j^{\lasttime}\right)\! = \{i_j^{\ell - 1}\}$.
For all $i \ne i_j^{\ell - 1}$, we have
${\!\left(\!\lastomega\right)\!\!}^\top\!\nordpi \ne 0$ and thus
each $\curomegatra\nordpi$ maintains the sign in some right-neighbourhood of~$\lasttime$.
For $i = i_j^{\ell - 1}$, rewrite \autoref{eq:sox} as
\begin{equation}
\mys (\vec\alpha_j^\xi)\!\tra\nordplasti =
\int_{\lasttime}^{\xi} \enorm\curomega \cdot (\actgamma)\!\tra\nordplasti \,\der t \;,
\label{eq:sox.lm1}
\end{equation}
where the integrand is non-negative.
Therefore, the function $s_j (\vec\alpha_j^\xi)\!\tra\nordplasti$ is non-negative for all $\xi > \lasttime$
and moreover is non-decreasing.
Consider the set $I_j^\ell$ defined in the statement of the proposition;
we have $\emptyset \ne I_j^\ell = I_\ppp(\curomega)$ and
$(\actgamma)\!\tra\nordplasti > 0$ for all $t$ greater than $\lasttime$ in
a small neighbourhood of $\lasttime$.
(Note that $\actgamma \ne \veczero$
 by \autoref{pr:omega-elementary}, part~\ref{pr:omega-elementary.gamma=zero}.)
So we can replace $\actgamma$ with $\mygamma$
in \autoref{eq:sox.lm1} if $\xi$ is close enough to $\lasttime$; we have now shown the existence of a $\xi > \lasttime$ such that
$I_\ppp(\curomega) = \mygamma$ for all
$t \in (\lasttime, \xi)$. (In fact, 
we can again choose $\xi \coloneqq \inf \{ t > \lasttime \mid I_0(\curomega) \ne \emptyset \}$.)

Having found an appropriate $\xi$,
let us observe that, by the inductive hypothesis (part~\ref{pr:omega.I0})
and by the choice of $\xi$, the set $I_0(\curomega)$ is only non-empty
for finitely many time points $t \in [0, \xi]$.
Let us consider \[I' \coloneqq \left\{ i \in [n] \;\middle\vert\; \mys {\left(\!\lastomega\right)\!\!}^\top\!\nordpi < 0 \right\} = I_{-s_j}(\vec{z}_j) \setminus \{i_j^1, \ldots, i_j^{\ell - 1}\} \;.\]
Recalling that all training points are positively correlated, we see that
$(\mygamma)\!\tra\nordpi > 0$.
By \autoref{pr:omega-norm}, part~\ref{pr:omega-norm.solve-bound}, the integrand in \autoref{eq:sox} is
lower-bounded by $\mu \cdot (\mygamma)\!\tra\nordpi$ for all $t$.
Therefore, for each~$i \in I'$ the expression on the right-hand side of \autoref{eq:sox}
tends to $+\infty$ if we let, formally, $\xi \to +\infty$.
Since for $\xi = \lasttime$ each of the right-hand sides is negative if $i \in I'$,
there exists some $\xi > \lasttime$ and an $i \in I'$ for which
the left-hand side, $(\vec\alpha_j^\xi)\!\tra\nordpi$, becomes $0$.
Rewriting \autoref{eq:sox} as
$\int_{\lasttime}^{\xi} \enorm\curomega \,\der t = r_i$,
where $r_i \coloneqq - \mys {\left(\!\lastomega\right)\!\!}^\top\!\nordpi \Big/ (\mygamma)\!\tra\nordpi > 0$,
we observe that the integral on the left-hand side does not depend on~$i$.
Thus, the smallest $\xi > \lasttime$ for which the integral is equal
to $r_i$ for some $i \in I'$ is the earliest time point after $\lasttime$
at which the set $I_0(\vec{\alpha}_j^\xi)$ becomes non-empty.
This value of $\xi$ is then, by definition, $\nexttime$.
Since for $\xi = \lasttime$ the integral is zero and since $r_i > 0$ for all $i \in I'$,
we also have $I_0(\vec{\alpha}_j^\xi) = \{ i_j^\ell \}$ where
$i = i_j^\ell$ is the index of the smallest $r_i$ among $i \in I'$;
this $i$ is unique by \autoref{ass:enum}, part~\ref{ass:enum.omega}.

To complete the proof of the inductive step for part~\ref{pr:omega.argmin},
it remains to note
that rescaling each $r_i$ by a factor of $\left\|\mygamma\right\| \Big/ \left\|\lastomega\right\|$ does not change
the $\argmin$.

Notice that, for the current value of $\ell$
we have also justified the inequality $\lasttime < \nexttime$ of part~\ref{pr:omega.tau},
as well as equalities 
$I_\ppp(\vec{\alpha}_j^t) = I_j^\ell$ and
$I_0(\vec{\alpha}_j^t) = \emptyset$
for all $t \in (\tau_j^{\ell - 1}, \tau_j^\ell)$,
and $I_0\!\left(\!\vec{\alpha}_j^{\tau_j^\ell}\right)\! = \{i_j^\ell\}$,
which together comprise parts~\ref{pr:omega.Ip} and~\ref{pr:omega.I0}.
This completes the inductive argument, proving
parts~\ref{pr:omega.argmin},
\ref{pr:omega.tau},
\ref{pr:omega.Ip}, and
\ref{pr:omega.I0}.

\subparagraph{Intermediate summary.}
\label{pr:omega.int.summ}
We have already established uniqueness of the enumeration
$i_j^1, \ldots, i_j^{n_j}$ and time points $\tau_j^1, \ldots, \tau_j
^{n_j}$: part~\ref{pr:omega.tau} requires that the latter be sorted
in the ascending order, and our argument for the choice of $\tau_j^\ell$
makes it clear that there is always only one possibility, if we
want to require (as parts~\ref{pr:omega.Ip} and~\ref{pr:omega.I0} do)
that $I_\ppp(\curomega)$ remain constant in between $\lasttime$ and $\nexttime$,
and $I_0\!\left(\!\nextomega\right)\!$ be non-empty.
In addition, we now know that the assumption of \autoref{pr:omega-cont}
holds for every time segment $[0, T]$ such that
$\actgamma \ne \veczero$ for all $t \in [0, T)$;
and in particular up to $T = \tau_j^{n_j}$.

\subparagraph{Parts
\ref{pr:omega.inter},
\ref{pr:omega.finalp}, and
\ref{pr:omega.finalm}.}
For part~\ref{pr:omega.inter}, suppose $\ell < n_j$.
By \autoref{pr:omega-cont}, part~\ref{pr:omega-cont.c-gamma},
both one-sided limits of $\cos\curphi$ as $t \to (\nexttime)^\pm$ exist.
By part~\ref{pr:omega.Ip} of the current proposition,
\begin{equation*}
\lim_{t \to (\nexttime)^-} \cos\curphi =
\frac{{\!\left(\!\nextomega\right)\!\!}^\top \! \newgamma_{I_j^\ell}}{\left\|\nextomega\right\| \left\|\newgamma_{I_j^{\ell}}\right\|}
\qquad\text{and}\qquad
\lim_{t \to (\nexttime)^+} \cos\curphi =
\frac{{\!\left(\!\nextomega\right)\!\!}^\top \! \newgamma_{I_j^{\ell+1}}}{\left\|\nextomega\right\| \left\|\newgamma_{I_j^{\ell+1}}\right\|}
\;,
\end{equation*}
and we notice that the numerators are equal
by \autoref{pr:omega-cont}, part~\ref{pr:omega-cont.c-omega.gamma}.
Multiplying each limit by the norm of the corresponding $\newgamma$,
we obtain the desired equation.

Part~\ref{pr:omega.finalp} follows from the same calculations
in the case $\ell = n_j$, where instead of $I_j^{\ell+1}$ we use
$I_j^{n_j} \cup \{i_j^{n_j}\} = [n]$.

For part~\ref{pr:omega.finalm} we observe that
$I_j^{n_j} = \{i_j^{n_j}\}$, so
we have $I_\ppp\!\left(\!\vec\alpha_j^{\tau_j^{n_j}}\right)\! = \emptyset$ and
           $I_0\!\left(\!\vec\alpha_j^{\tau_j^{n_j}}\right)\! = \{i_j^{n_j}\}$,
so indeed $\cos\curphi \to 0$ as $t \to (\tau_j^{n_j})^-$.

\subparagraph{Part~\ref{pr:omega.cos}.}
We rely on the facts that
$I_0(\curomega) = \emptyset$
and that $\actgamma \equiv \mygamma$
for all $t \in (\lasttime, \nexttime)$,
proved in parts~\ref{pr:omega.Ip} and~\ref{pr:omega.I0}.
Notice that
\begin{align*}
\frac{\der \bigl(\curomegatra \mygamma\bigr)}{\der t} &= 
\mys \enorm{\curomega} \bigl\|\mygamma\bigr\|^2
\end{align*}
and, using a previously obtained formula for
$\der{\enorm{\curomega}}/ \der t$ (just before \autoref{pr:omega-norm}),
\begin{align*}
\frac{\der \cos\curphi}{\der t} &=
\frac{\der}{\der t} \!\left(\frac{\curomegatra \mygamma}{\enorm{\curomega}\cdot\bigl\|\mygamma\bigr\|}\right)\! \\
&=
\frac{\frac{\der}{\der t}\bigl(\curomegatra \mygamma\bigr) \cdot \enorm\curomega \cdot \bigl\|\mygamma\bigr\| -
      \bigl\|\mygamma\bigr\| \cdot \frac{\der \enorm\curomega}{\der t} \cdot \curomegatra \mygamma}%
     {\enorm\curomega^2 \bigl\|\mygamma\bigr\|^2} \\
&=
\frac{\mys \enorm\curomega \bigl\|\mygamma\bigr\|^2 \enorm\curomega \bigl\|\mygamma\bigr\| -
      \bigl\|\mygamma\bigr\| \curomegatra\mygamma \cdot \mys \curomegatra\mygamma}%
     {\enorm\curomega^2 \bigl\|\mygamma\bigr\|^2} \\
&=
\mys \bigl\|\mygamma\bigr\| (1 - \cos^2 \curphi) \;.
\end{align*}
Separating variables, we obtain
\begin{equation*}
\frac{\der \cos\curphi}{1 - \cos^2 \curphi} = \mys\bigl\|\mygamma\bigr\| \der t \;,
\end{equation*}
and so $\artanh\cos\curphi = \mys\bigl\|\mygamma\bigr\| \,t + C$
for $t \in (\lasttime, \nexttime)$, where the constant $C$ is determined
from the initial condition
$\lim_{t \to (\lasttime)^+} \artanh \cos\curphi = \mys\bigl\|\mygamma\bigr\|\,\lasttime + C$.
The left-hand side is well-defined, since $\cos \limlasttime \!\not\in \{-1, 1\}$.
Indeed, $\cos\curphi \ge 0$ for all $t$
by \autoref{pr:omega-elementary}, part~\ref{pr:omega-elementary.omega.gamma},
so the limit cannot be negative; it thus suffices to rule out the value~$1$.
The case $\mys = -1$ is already handled
in \autoref{pr:omega-cont}, part~\ref{pr:omega-cont.no-jump-to-1}.
For $\mys = +1$, we observe that, if $\ell > 1$, then
$\cos \varphi_j^{{(\ell - 1)}^-} \!>
 \cos \varphi_j^{{(\ell - 1)}^+}$ by part~\ref{pr:omega.inter} of the current proposition,
since
$\bigl\|\vec{\gamma}_{I_j^{\ell - 1}}\bigr\|^2 <
 \bigl\|\vec{\gamma}_{I_j^{\ell - 1} \cup \{i_j^{\ell-1}\}}\bigr\|^2 =
 \bigl\|\vec{\gamma}_{I_j^{\ell    }}\bigr\|^2$ thanks to the positive correlation
between training points; hence $\cos \varphi_j^{{(\ell - 1)}^+} \!< 1$.
Finally, $\ell = 1$ implies
$\cos \varphi_j^{{(\ell - 1)}^+} \!= \lim_{t \to 0^+} \cos\curphi = \cos \angle(\vec z_j, \gamma_{I_\ppp(\vec\alpha_j^0)}) = 1$,
and in this case $I_\ppp(\vec\alpha_j^0) = [n]$
by \autoref{pr:omega-elementary}, part~\ref{pr:omega-elementary.omega=gamma}.
Hence, $I_{-\mys}(\vec z_j) = \emptyset$ and $n_j = 0$, a contradiction.
In conclusion, we have thus argued that
$\cos \limlasttime \!\not\in \{-1, 1\}$ in all cases, so
$\artanh\cos\curphi = \mys\bigl\|\mygamma\bigr\| (t - \lasttime) + \artanh\cos\limlasttime\!$
and it remains to take the hyperbolic tangent on both sides of this equation
to prove part~\ref{pr:omega.cos} for $t \in (\lasttime, \nexttime)$.

\subparagraph{Part~\ref{pr:omega.omega}.}
We rely on the result of part~\ref{pr:omega.cos}.
Recall the analytic expression for 
the derivative $\der\noromega / \der t$, obtained just before
\autoref{pr:omega-norm}.
Notice that,
for all $t \in (\lasttime, \nexttime)$,
the derivative $\der\noromega / \der t$ belongs to the linear subspace
spanned by vectors $\noromega$ and
$\overline{\vec{\gamma}}_{I_j^\ell}$.
It follows that $\noromega$ can be expressed as a linear
combination of
two fixed vectors,
$\overline{\vec{\alpha}}_j^{\tau_j^{\ell - 1}}$
and
$\overline{\vec{\gamma}}_{I_j^\ell}$.
It is thus sufficient to check that the vector
\begin{equation*}
 \vec f \coloneqq
 \!\left(\!
 \sin (\varphi_j^t) \,
 \overline{\vec{\alpha}}_j^{\tau_j^{\ell - 1}} \! +
 \sin \!\left(\varphi_j^{(\ell - 1)^+} \! - \varphi_j^t\right)\! \,
 \overline{\vec{\gamma}}_{I_j^\ell}
 \!\right)\!
 \Big/ \! \sin \varphi_j^{(\ell - 1)^+}
\end{equation*}
has norm~$1$ and forms an angle of $\curphi$ with $\mygamma$.
We have
\begin{align*}
\enorm{\vec f}^2 &=
\left(\frac{\sin\curphi}{\sin\limlasttime}\right)^{\negthickspace 2} \!+
\left(\frac{\sin\!\left(\limlasttime-\curphi\right)}{\sin\limlasttime}\right)^{\negthickspace\! 2} \\ & \quad\; +
2 \cdot \frac{\sin\curphi \, \sin\!\left(\limlasttime-\curphi\right)}{\sin^2 \limlasttime} \cdot \cos\limlasttime \:.
\end{align*}
Denote $a = \curphi$ and $b = \limlasttime - \curphi$, then
\begin{align*}
\enorm{\vec f}^2
&= \frac{\sin^2 a + \sin^2 b + 2 \sin a \sin b \cos(a + b)}{\sin^2 (a + b)} \\[0.5ex]
&= \frac{\sin^2 a + \sin^2 b + 2 \sin a \sin b \,(\cos a \cos b - \sin a \sin b)}
        {(\sin a \cos b + \cos a \sin b)^2} \\[0.5ex]
&= \frac{\sin^2 a + \sin^2 b + 2 \sin a \sin b \,(\cos a \cos b - \sin a \sin b)}
        {\sin^2 a \cos^2 b + 2 \sin a \cos b \cos a \sin b + \cos^2 a \sin^2 b} \\[0.5ex]
&= \frac{\sin^2 a + \sin^2 b + 2 \sin a \sin b \cos a \cos b - 2 \sin^2 a \sin^2 b}
        {\sin^2 a \,(1 - \sin^2 b) + 2 \sin a \cos b \cos a \sin b + (1 - \sin^2 a) \sin^2 b} \\[0.5ex]
&= 1 \;.
\end{align*}
To verify the second claim, observe that
\begin{align*}
\vec f \tra \normygammafix &=
\frac{\sin \curphi}{\sin \varphi_j^{(\ell - 1)^+}} \cdot
{\!\left(\!\overline{\vec{\alpha}}_j^{\lasttime}\right)\!\!}^\top \! \normygammafix +
\frac{\sin \!\left(\varphi_j^{(\ell - 1)^+} \! - \varphi_j^t\right)\!}{\sin \varphi_j^{(\ell - 1)^+}} \cdot
\bigl\|\normygammafix\bigr\|^2 \\
&=
\frac{\sin \curphi \, \cos \varphi_j^{(\ell - 1)^+} \!+
      \sin \varphi_j^{(\ell - 1)^+} \cos \varphi_j^t -
      \cos \varphi_j^{(\ell - 1)^+} \sin \varphi_j^t}
{\sin \varphi_j^{(\ell - 1)^+}} \\
&=
\cos \curphi \;.
\end{align*}
We must still check still that the vector $\vec f$ is on the correct side of $\normygammafix$:
indeed, there are two arcs on the unit
circle that connect the endpoint of vector $\normygammafix$
with a point at arc length $\curphi$ away from it.
However, this check is easy: for $t = \lasttime$, only one of these arcs connects $\normygammafix$
to
$\overline{\vec{\alpha}}_j^{\lasttime}$, and we can see that
$\vec f \to
 \overline{\vec{\alpha}}_j^{\lasttime}$ as $\curphi \to \limlasttime$.

\subparagraph{Part~\ref{pr:omega.sin}.}
Let $t \to (\nexttime)^-$
in the equation of part~\ref{pr:omega.omega}, and take the dot
product of each side with $\dpnextifix$. Observe that
${\!\left(\!\overline{\vec{\alpha}}_j^{\tau_j^{\ell}}\right)\!\!}^\top \! \dpnextifix = 0$,
because $I_0\!\left(\!\vec{\alpha}_j^{\tau_j^{\ell}}\right)\! = \{i_j^\ell\}$ by
part~\ref{pr:omega.I0}. We obtain
\begin{equation*}
0 =
\frac{\sin \limnexttime \!\cdot
{\left(\!\overline{\vec{\alpha}}_j^{\tau_j^{\ell - 1}}\right)\!\!}^\top \! \dpnextifix
+
\sin\!\left(\limlasttime \!- \limnexttime\right) \cdot
\bigl(\overline{\vec{\gamma}}_{I_j^\ell}\bigr)\!\tra \dpnextifix
}{
\sin \limlasttime
}
\end{equation*}
and the required equation follows.
It remains to note that $\sin\limnexttime \!\ne 0$ because
otherwise either $\limlasttime \!- \limnexttime \in \{-\pi, 0, \pi\}$
or 
$\bigl(\overline{\vec{\gamma}}_{I_j^\ell}\bigr)\!\tra \dpnextifix = 0$.
The former is impossible because
we know already from part~\ref{pr:omega.cos} that
$\cos\curphi \in (0, 1)$ when $t \in (\lasttime, \nexttime)$,
and $\curphi \ge 0$ by definition, so
$\limlasttime \!= \limnexttime$ but this would still contradict part~\ref{pr:omega.cos}.
The latter is impossible because $I_j^\ell \ne \emptyset$ and,
by \autoref{pr:omega-elementary}, part~\ref{pr:omega-elementary.gamma=zero},
the dot product must be positive due to positive correlation between
training points.

\subparagraph{Part~\ref{pr:omega.dox.compl}.}
Consider any interval $(t_1, t_2)$ such that $I_0(\curomega) = \emptyset$
for all $t \in (t_1, t_2)$, and let $\newgamma \coloneqq \actgamma$;
the choice of $t$ in the interval is immaterial
by the continuity of the map $t \mapsto \curomega$ and
of the dot product function with a fixed vector $\dpi$.
We have $\newgamma = \mygamma$ when $t \in (\lasttime, \nexttime)$ by part~\ref{pr:omega.Ip}.
Recall our calculations for the derivative of $\enorm\curomega$
and of $\noromega$
(before \autoref{pr:omega-norm}).
We have $\der \noromega / \der t = \mys \, \vec p$, where
$\vec p \coloneqq \newgamma - \noromega \cdot \noromegatra \newgamma$
is the vector obtained by subtracting from $\newgamma$ its orthogonal projection
onto the line with direction $\curomega$.
Then
\begin{equation*}
\mys \frac{\der \noromegatra\nordpi}{\der t}
= \newgamma\tra \nordpi - \noromegatra\nordpi \cdot \noromegatra\newgamma \;.
\end{equation*}
For $t \in (t_1, t_2)$, we have $\noromegatra\nordpi \le 0$
because $i \notin I_\ppp(\curomega)$.
Recall that $\noromegatra\newgamma \ge 0$ by \autoref{pr:omega-elementary}, part~\ref{pr:omega-elementary.omega.gamma}.
We have shown that $- \noromegatra\nordpi \cdot \noromegatra\newgamma \ge 0$,
completing the proof of part~\ref{pr:omega.dox.compl}.

\subparagraph{Part~\ref{pr:omega.dox.ddox}.}
We continue the calculation from
part~\ref{pr:omega.dox.compl} assuming that $\mys = -1$ and $i = i_j^\ell$.
For the function $g(t) \coloneqq \noromegatra\nordpnexti$, we have
\begin{align*}
\frac{\der g}{\der t}
&= - \newgamma\tra \nordpnexti + \noromegatra\nordpnexti \cdot \noromegatra\newgamma \;, \\
\frac{1}{\enorm{\newgamma}} \frac{\der^2 g}{\der t^2} &=
\frac{\der}{\der t} \left(\noromegatra\normygamma\right) \cdot \noromegatra\nordpnexti +
\noromegatra\normygamma \cdot \frac{\der}{\der t} \left(\noromegatra\nordpnexti\right) \\
&= \frac{\der \cos\curphi}{\der t} \cdot g + \cos\curphi \cdot \frac{\der g}{\der t} \;.
\end{align*}
We will show the following two properties:
\begin{itemize}
\item $\der g / \der t < 0$ as $t \to (\lasttime)^+$;
\item if $\der g / \der t = 0$ for some $t$, then $\der^2 g / \der t^2 < 0$ for the same $t$.
\end{itemize}
Together, these properties imply that $\der g / \der t < 0$ throughout
the interval $(\lasttime, \nexttime)$. Indeed, assume otherwise for the sake of
contradiction, then
$\der g / \der t = 0$ at some point $t_0 \in (\lasttime, \nexttime)$.
By the second property, $g$ must have a local maximum at $t_0$,
and in particular $\der g / \der t > 0$ for all $t < t_0$ close enough to~$t_0$.
By the first property, the minimum of $\der g / \der t$ on $(\lasttime, t_0)$
exists and is attained at an interior point of the interval.
But this contradicts the second property.

Let us now justify the properties.
For the first property, notice that
\begin{align*}
\left.\frac{\der g}{\der t}\right|_{t \to (\lasttime)^+}\! &=
- \enorm{\newgamma} \cdot
\!\left(\!
\normygamma\tra\nordpnexti - 
{\!\left(\!\overline{\vec{\alpha}}_j^{\lasttime}\right)\!\!}^\top \!
\nordpnexti
\cdot
{\!\left(\!\overline{\vec{\alpha}}_j^{\lasttime}\right)\!\!}^\top \!
\normygamma
\!\right)\!
\\
&=
- \enorm{\newgamma} \cdot
(\normygamma\tra\nordpnexti) \cdot
\!\left(\!
1 -
\frac{
{\!\left(\!\overline{\vec{\alpha}}_j^{\lasttime}\right)\!\!}^\top \!
\nordpnexti}%
{
\normygamma\tra\nordpnexti
}
\cdot
{\!\left(\!\overline{\vec{\alpha}}_j^{\lasttime}\right)\!\!}^\top \!
\normygamma
\!\right)\!
\;.
\end{align*}
Here, 
$\normygamma\tra\nordpnexti > 0$ since $i_j^{\ell} \in I_j^\ell$.
The value of the ratio
$
{\!\left(\!\overline{\vec{\alpha}}_j^{\lasttime}\right)\!\!}^\top \!
\nordpnexti
\Big/
\normygamma\tra\nordpnexti
$
appears in the statement of part~\ref{pr:omega.argmin}, and in particular
replacing the index $i_j^\ell$ with any other $i \in I_j^\ell$ would result in
a higher (positive) value.
Therefore, if we assume for the sake of contradiction that
the right-hand side in the last equation is non-negative, then
it will remain non-negative if
$i_j^\ell$ is replaced with every other $i \in I_j^\ell$.
In other words, if $g(t) = \noromegatra\nordpnexti$ is non-decreasing
in a right-neighbourhood of~$\lasttime$,
so is every dot product $\noromegatra\nordpi$ with $i \in I_j^\ell$.
But then their linear combination with positive coefficients $y_i \, \enorm{\dpi} / n$
is also non-decreasing. This, however, is not possible because this
linear combination is $\noromegatra\actgamma$, and we already saw
in the proof of part~\ref{pr:omega.dox.compl} that $\der \noromega / \der t = \mys \, \vec p$, where $\vec p$ is an orthogonal
projection of~$\actgamma$
onto a proper subspace. By standard properties
of projections we must have $\vec p\tra \actgamma > 0$ and,
since $\mys = -1$, $\der \bigl(\noromegatra\actgamma\bigr) / \der t < 0$,
which is a contradiction.
(The case $\vec{p}^\top \actgamma = 0$ is impossible by \autoref{pr:omega-cont}, part~\ref{pr:omega-cont.no-jump-to-1}, as we would then have $\cos \limlasttime \!= 1$.)
This concludes the proof of the first property.

The second property follows directly from the equation for $\der^2 g / \der t^2$,
because $\cos\curphi$ decreases by part~\ref{pr:omega.cos} and because $g > 0$.

For the sign of second derivative in general, it remains to consider the second term. The first factor is positive by
\autoref{pr:omega-elementary}, part~\ref{pr:omega-elementary.omega.gamma};
and we just proved above that $\der g / \der t < 0$.
(Note that $\newgamma \ne \zerovec$ by
\autoref{pr:omega-elementary}, part~\ref{pr:omega-elementary.gamma=zero},
because $i_j^\ell \in I_\ppp(\curomega)$.)
This completes the proof of part~\ref{pr:omega.dox.ddox}.
\end{proof}

\begin{corollary}
\label{cor:measure}
For all $j \in J_\ppp \cup J_\mmm$, the set of all $\vec{z}_j \in \mathbb{R}^d$ such that $|I_0(\vec{\alpha}_j^t)| > 1$ for some $t \in [0, \infty)$ has Lebesgue measure zero.
\end{corollary}

\begin{proof}
A single yardstick trajectory at any time~$t$ follows a direction~$\vec\gamma_S$
for some $S \subseteq [n]$. The set~$S$ changes at most $n$~times,
namely at the crossing of $\bigcup_{i \in [n]} H_i$,
where $H_i$~is the set of vectors orthogonal to the training point~$\vec x_i$.
(The proof of this fact does not rely on \autoref{ass:enum}, part~\ref{ass:enum.omega}.
 It is a consequence of \autoref{pr:omega-cont}, part~\ref{pr:omega-cont.mon}.
 We note that the assumption of \autoref{pr:omega-cont} is shown to be valid
 in the proof of \autoref{pr:omega}, under ``Intermediate summary'' on \autopageref{pr:omega.int.summ}.)

The union~$U$ of all $H_i \cap H_k$, $i < k$,
is a union of finitely many subspaces of dimension $d-2$
(because no two training points are collinear by \autoref{ass:enum}, part~\ref{ass:enum.eigen}).
Consider all the vectors~$\vec u$ such that the yardstick trajectory
starting at~$\vec u$ passes through~$U$.
We claim that this is a set of zero measure.
Indeed:
\begin{itemize}
\item
Every convex polyhedron~$P$ of dimension $d-2$, for example $H_i \cap H_k$,
can be reached by a straight-line trajectory (without change of direction)
from a convex polyhedron~$P'$ of dimension at most $d-1$, i.e., of co-dimension at least~$1$.
\item
The previous change of direction occurs at the intersection of the polyhedron~$P'$
and the union of all~$H_i$.
This intersection is a finite union of convex polyhedra of co-dimension at least~$2$,
because, for all nonempty subsets $S \subseteq [n]$,
the vector~$\vec\gamma_S$ cannot
belong to any~$H_i$,
thanks to the $45$-degree condition.
To each of these polyhedra, the previous bullet point applies.
\item
No more than~$n$ changes of direction may take place along a single trajectory.
\end{itemize}
Thus, all vectors from which a point in~$U$ can be reached along a yardstick trajectory
belong to a finite union of affine subspaces of co-dimension~$1$.
Thus, they form a measure zero set.
\end{proof}

\section{Proofs for the first phase}

Here we prove \autoref{main:l:0}, \autoref{main:l:len.w}, and \autoref{main:l:1}, as well as a number of related results.  The former are subsumed by \autoref{l:len.w}, \autoref{l:0}, and \autoref{l:1} below.

Recall the definitions of~$\delta$ and~$\Delta$ in \autoref{s:ass}:
\begin{align*}
\delta & \coloneqq
\min\!\left\{\!\!
\begin{array}{c}
\min_{i \in [n]}
\|\vec{x}_i\|, \;
\min_{i, i' \in [n]}
{\overline{\vec{x}}_i\!}^\top \, \overline{\vec{x}}_{i'}, \;
\min_{k \in [d - 1]}
(\!\sqrt{\eta_k} - \sqrt{\eta_{k + 1}}) (d - 1), \;
\sqrt{\eta_d},
\\[1.33ex]
\min_{k \in [d]}
\nu^*_k \sqrt{d}, \;
\min_{j \in [m]}
\|\vec{z}_j\|, \;
\min_{j \in J_\ppp}
\cos \varphi_j^0, \;
\min_{j \in J_\mmm}
\sin \varphi_j^0,
\\[1.33ex]
\min\!\left\{
|{\overline{\vec{\alpha}}_j^t}^\top \, \overline{\vec{x}}_i|
\,\,\,\middle\vert
\begin{array}{c}
j \in J_\ppp \cup J_\mmm \,\wedge\,
\ell \in [n_j]
\\[.67ex] \wedge\;
t \in [\tau_j^{\ell - 1}, \tau_j^\ell] \,\wedge\,
i \in [n]
\\[.67ex] \wedge\;
i \neq i_j^\ell \,\wedge\,
(\ell \neq 1 \Rightarrow i \neq i_j^{\ell - 1})
\end{array}
\!\!\right\}\!, \;
\min_{j \in J_\mmm}
{\overline{\vec{\alpha}}_j^0}^\top \, \overline{\vec{x}}_{i_j^1},
\\[4ex]
\min
\{\tau_j^\ell - \tau_j^{\ell - 1}
  \,\,\,\vert\,\,\,
  j \in J_\ppp \cup J_\mmm \,\wedge\,
  \ell \in [n_j]\}
\end{array}
\!\!\right\}\!
\\[.33ex]
\Delta & \coloneqq
\max
\{{\textstyle \max_{i \in [n]}}
  \|\vec{x}_i\|, \;
  {\textstyle \max_{j \in [m]}}
  \|\vec{z}_j\|, \;
  1\} \;.
\end{align*}
Thus $\delta$~is the minimum of: the length of any training point, the cosine of the angle between any two training points, the difference between the square roots of any consecutive eigenvalues adjusted by the dimension, the square root of the smallest eigenvalue, the smallest eigenvector coordinate of the teacher neuron adjusted by the square root of the dimension, the length of any unscaled hidden-neuron initialisation, the cosine or sine of the angle between it (if active) and the corresponding vector~$\vec{\gamma}_I$ depending on whether the last-layer sign is positive or negative (respectively), the absolute cosine of any angle between a trajectory point~$\vec{\alpha}_j^t$ and a training point which is neither the previous nor the next to cross the half-space boundary, the cosine of the angle between any initial negative-sign active hidden neuron and the first data point to cross the boundary, and the time between any two consecutive crossings; and $\Delta \geq 1$ is the maximum length of any traning point or unscaled hidden-neuron initialisation.

First we observe that, immediately from the definitions in \autoref{s:ass} of the vectors~$\vec{\gamma}_I$, the matrix~$\vec{X}$, the eigenvalues~$\eta_k$, the eigenvectors~$\vec{u}_k$, and the coordinates~$\nu^*_k$ of the teacher neuron with respect to the basis consisting of the eigenvectors, we have the following two alternative expressions for the vector~$\vec{\gamma}_{[n]}$.

\begin{proposition}
\label{pr:gamman}
$\vec{\gamma}_{[n]} =
 \frac{1}{n} \vec{X} \vec{X}^\top \vec{v}^* =
 \sum_{k = 1}^d \eta_k \nu^*_k \vec{u}_k$.
\end{proposition}

Then we establish upper bounds on: the largest eigenvalue of the matrix $\frac{1}{n} \vec{X} \vec{X}^\top$, the ratio of any two consecutive eigenvalues in their decreasing ordering, the Euclidean lengths of the vectors~$\vec{\gamma}_I$, the cosines of the angles~$\varphi_j^t$ that measure alignment of the yardstick trajectories~$\vec{\alpha}_j^t$ (both defined in \autoref{s:ass}) mapped backwards through the hyperbolic tangent sigmoid, and the finish time of the last intermediate alignment stage of a yardstick trajectory; and lower bounds on: the Euclidean lengths of the vectors~$\vec{\gamma}_I$, and the cosines of the angles between a yardstick trajectory and the training point that is the next to enter or exit its active half-space.

\begin{proposition}
\label{pr:d.D}
\begin{enumerate}[(i),itemsep=0ex,leftmargin=3em]
\item
\label{pr:d.D.alpha1}
$\eta_1 \leq {\Delta\!}^2$.
\item
\label{pr:d.D.alphakp1k}
$\frac{\eta_{k + 1}}{\eta_k} \leq
 {\left(1 - \frac{\delta}{(d - 1) \Delta}\right)\!}^2$
for all $k \in [d - 1]$.
\item
\label{pr:d.D.gamma}
$\frac{\delta^{5 / 2} |I|}{\sqrt{2} n} \leq
 \|\vec{\gamma}_I\| \leq
 \frac{{\Delta\!}^2 |I|}{n}$
for all $I \subseteq [n]$, and
$\delta^2 \leq
 \|\vec{\gamma}_{[n]}\|$.
\item
\label{pr:d.D.varphip}
$\max_{j \in J_\ppp}^{\ell \in [n_j]} \artanh \cos \varphi_j^{\ell^-} \!
 < \ln\!\left(\frac{2}{\delta}\right)$.
\item
\label{pr:d.D.varphim}
$\max_{j \in J_\mmm}^{\ell \in [n_j]} \artanh \cos \varphi_j^{(\ell - 1)^+} \!
 < \ln\!\left(\frac{2}{\delta}\right)$.
\item
\label{pr:d.D.tau}
$\max_{j \in J_\ppp \cup J_\mmm} \tau_j^{n_j}
 < \frac{4 n \ln n}{\delta^3}$.
\item
\label{pr:d.D.ox}
$\left|{\overline{\vec{\alpha}}_j^t}^\top \, \overline{\vec{x}}_{i_j^\ell}\right| \geq
 \frac{2 \delta^4}{3 n} (\tau_j^\ell - t)$
for all $j \in J_\ppp \cup J_\mmm$, $\ell \in [n_j]$, and $t \in [\tau_j^{\ell - 1}, \tau_j^\ell]$.
\end{enumerate}
\end{proposition}

\begin{proof}
For part~\ref{pr:d.D.alpha1}, we have
\[\eta_1
=    \left\|\frac{1}{n} \vec{X} \vec{X}^\top \vec{u}_1\right\|
=    \frac{1}{n} \left\|\sum_{i \in [n]} \vec{x}_i \, \vec{x}_i^\top \vec{u}_1\right\|
\leq \frac{1}{n} \sum_{i \in [n]} \|\vec{x}_i \, \vec{x}_i^\top \vec{u}_1\|
\leq \max_{i \in [n]} \|\vec{x}_i\|^2
=    {\Delta\!}^2 \;.\]

For part~\ref{pr:d.D.alphakp1k}, supposing $k \in [d - 1]$, by part~\ref{pr:d.D.alpha1} we have
\[\frac{\sqrt{\eta_{k + 1}}}{\sqrt{\eta_k}}
=    1 - \frac{\sqrt{\eta_k} - \sqrt{\eta_{k + 1}}}{\sqrt{\eta_k}}
\leq 1 - \frac{\sqrt{\eta_k} - \sqrt{\eta_{k + 1}}}{\sqrt{\eta_1}}
\leq 1 - \frac{\delta}{(d - 1) \Delta} \;.\]

For part~\ref{pr:d.D.gamma}, supposing $I \subseteq [n]$, recalling that $\angle(\vec{v}^*, \vec{x}_i) < \pi / 4$ for all $i \in [n]$ we have
\begin{multline*}
\|\vec{\gamma}_I\|
=    \frac{1}{n} \left\|\sum_{i \in I} y_i \vec{x}_i\right\|
=    \frac{1}{n} \sqrt{\sum_{i, i' \in I} y_i y_{i'} \vec{x}_i^\top \vec{x}_{i'}}
\geq \frac{|I|}{n} \min_{i, i' \in I} \sqrt{y_i y_{i'} \vec{x}_i^\top \vec{x}_{i'}} \\
=    \frac{|I|}{n} \min_{i, i' \in I} \sqrt{{\vec{v}^*}^\top \vec{x}_i \cdot
                                            {\vec{v}^*}^\top \vec{x}_{i'} \cdot
                                            \vec{x}_i^\top \vec{x}_{i'}}
\geq \frac{|I|}{n} \sqrt{{\!\left(\frac{\delta}{\sqrt{2}}\right)\!}^2 \delta^3}
=    \frac{\delta^{5 / 2} |I|}{\sqrt{2} n} \;,
\end{multline*}
and we have
\[|\vec{\gamma}_I\|
\leq \frac{1}{n} \sum_{i \in I} y_i \|\vec{x}_i\|
\leq \frac{|I|}{n} \max_{i \in I} {\vec{v}^*}^\top \vec{x}_i \cdot \|\vec{x}_i\|
\leq \frac{|I|}{n} \max_{i \in I} \|\vec{x}_i\|^2
\leq \frac{{\Delta\!}^2 |I|}{n} \;.\]
Also, recalling \autoref{pr:gamman} we have
$\|\vec{\gamma}_{[n]}\|
=    \|\frac{1}{n} \vec{X} \vec{X}^\top \vec{v}^*\|
\geq \eta_d
\geq \delta^2$.

For part~\ref{pr:d.D.varphip}, supposing $j \in J_\ppp$ and $\ell \in [n_j]$, and observing that $\artanh q = \frac{1}{2} \ln\!\left(\frac{1 + q}{1 - q}\right)\! < \frac{1}{2} \ln\!\left(\frac{2}{1 - q}\right)$ for all $|q| < 1$, by \autoref{pr:omega}~\ref{pr:omega.Ip} and~\ref{pr:omega.I0} we have
\begin{multline*}
\artanh \cos \varphi_j^{\ell^-}\!
=    \artanh \lim_{t \to (\tau_j^\ell)^-}
             \cos \angle\!\left(\vec{\alpha}_j^t, \vec{\gamma}_{I_j^\ell}\right)\! \\
<    \artanh \sin \angle\!\left(\vec{x}_{i_j^\ell}, \vec{\gamma}_{I_j^\ell}\right)\!
=    \artanh \sqrt{1 - \cos^2 \angle\!\left(\vec{x}_{i_j^\ell},
                                            \vec{\gamma}_{I_j^\ell}\right)\!}
\leq \artanh \sqrt{1 - \delta^2} \\
<    \artanh\!\left(\!1 - \frac{\delta^2}{2}\right)\!
<    \frac{1}{2} \ln\!\left(\frac{4}{\delta^2}\right)\!
=    \ln\!\left(\frac{2}{\delta}\right)\! \;.
\end{multline*}

Part~\ref{pr:d.D.varphim} follows analogously, once we recall that, for all $j \in J_\mmm$, by \autoref{ass:enum}~\ref{ass:enum.init} we have $\cos \varphi_j^{0^+}\! = \cos \varphi_j^0 = \sqrt{1 - \sin^2 \varphi_j^0} \leq \sqrt{1 - \delta^2}$.

For part~\ref{pr:d.D.tau}, supposing $j \in J_\ppp \cup J_\mmm$ we have
\begin{align*}
\tau_j^{n_j}
& =    \sum_{\ell \in [n_j]} \tau_j^\ell - \tau_j^{\ell - 1}
& & \text{since $\tau_j^0 \coloneqq 0$ in \autoref{pr:omega}} \\
& \leq \sum_{\ell \in [n_j]} \frac{\ln\!\left(\frac{2}{\delta}\right)\!}
                                  {\left\|\vec{\gamma}_{I_j^\ell}\right\|}
& & \text{by parts~\ref{pr:d.D.varphip} and~\ref{pr:d.D.varphim},
          and \autoref{pr:omega}~\ref{pr:omega.cos}} \\
& \leq \frac{\sqrt{2} n}{\delta^{5 / 2}}
       \ln\!\left(\frac{2}{\delta}\right)\!
       \sum_{\ell \in [n_j]} \frac{1}{\left|I_j^\ell\right|}
& & \text{by part~\ref{pr:d.D.gamma}} \\
& \leq \frac{\sqrt{2} n}{\delta^{5 / 2}}
       \ln\!\left(\frac{2}{\delta}\right)\!
       \sum_{i \in [n]} \frac{1}{i}
& & \text{by the definition of~$I_j^\ell$ in \autoref{pr:omega}} \\
& <    \frac{\sqrt{2} n (1 + \ln n)}{\delta^{5 / 2}}
       \ln\!\left(\frac{2}{\delta}\right)\!
& & \text{by properties of the harmonic series} \\
& <    \frac{7 n \ln n}{2 \delta^{5 / 2}}
       \ln\!\left(\frac{2}{\delta}\right)\!
& & \text{since $n \geq 2$ by \autoref{ass:enum}~\ref{ass:enum.span}} \\
& <    \frac{4 n \ln n}{\delta^3}
& & \text{since $\ln\!\left(\tfrac{2}{\delta}\right)\! < \tfrac{8}{7 \sqrt{\delta}}$.}
\end{align*}

For part~\ref{pr:d.D.ox}, if $s_j = 1$ then by \autoref{pr:omega}~\ref{pr:omega.dox.compl} and by part~\ref{pr:d.D.gamma} we have
\[\inf_{t \in (\tau_j^{\ell - 1}, \tau_j^\ell)} \!\!
  \frac{\mathrm{d} \, {\overline{\vec{\alpha}}_j^t}^\top \,
                      \overline{\vec{x}}_{i_j^\ell}}
       {\mathrm{d} t}
\geq {\vec{\gamma}_{I_j^\ell}\!}^\top \, \overline{\vec{x}}_{i_j^\ell}
\geq \frac{\delta^{7 / 2}}{\sqrt{2} n} \;.\]

If $s_j = -1$ then by \autoref{pr:omega}~\ref{pr:omega.dox.ddox} we have that ${\overline{\vec{\alpha}}_j^t}^\top \, \overline{\vec{x}}_{i_j^\ell}$ is concave on $[\tau_j^{\ell - 1}, \tau_j^\ell]$, so by part~\ref{pr:d.D.varphim}, by \autoref{pr:omega}~\ref{pr:omega.cos}, by part~\ref{pr:d.D.gamma}, and since $\ln\!\left(\tfrac{2}{\delta}\right)\! < \tfrac{3}{2 \sqrt{2 \delta}}$, for all $t \in [\tau_j^{\ell - 1}, \tau_j^\ell)$ we have
\[\frac{{\overline{\vec{\alpha}}_j^t}^\top \, \overline{\vec{x}}_{i_j^\ell}}
       {\tau_j^\ell - t}
\geq \frac{{\overline{\vec{\alpha}}_j^{\tau_j^{\ell - 1}}}^\top
           \overline{\vec{x}}_{i_j^\ell}}
          {\tau_j^\ell - \tau_j^{\ell - 1}}
\geq \frac{\delta}{\frac{\sqrt{2} n}{\delta^{5 / 2}} \ln\!\left(\frac{2}{\delta}\right)\!}
=    \frac{\delta^{7 / 2}}{\sqrt{2} n} \frac{1}{\ln\!\left(\frac{2}{\delta}\right)\!}
>    \frac{2 \delta^4}{3 n} \;.
\qedhere\]
\end{proof}

Recall from \autoref{s:first} that
$T_0 = \max_{j \in J_\ppp \cup J_\mmm} \tau_j^{n_j} + 1$ and
$T_1 = \varepsilon \ln(1 / \lambda) / \|\vec{\gamma}_{[n]}\|$.

\begin{proposition}
\label{pr:T0.T1}
$T_0 < T_1 / 3$.
\end{proposition}

\begin{proof}
We have
\begin{align*}
T_1 / 3
& \geq 9 n \ln n \, \frac{{\Delta\!}^2}{\delta^3} \Big/ \|\vec{\gamma}_{[n]}\|
& & \text{by \autoref{ass:lambda}} \\
& \geq \frac{9 n \ln n}{\delta^3}
& & \text{by \autoref{pr:d.D}~\ref{pr:d.D.gamma}} \\
& >    \frac{4 n \ln n}{\delta^3} + 1
& & \text{since $n \geq 2$ by \autoref{ass:enum}~\ref{ass:enum.span}} \\
& >    T_0
& & \text{by \autoref{pr:d.D}~\ref{pr:d.D.tau}.}
\qedhere
\end{align*}
\end{proof}

The following lemma states that, throughout the first phase of the training, the lengths of the positive-sign initially active hidden neurons are non-decreasing, and the lengths of the negative-sign initially active hidden neurons are non-increasing; and it provides a time-sensitive upper bound for the former.  As for the remaining hidden neurons, i.e.~those whose indices are not in the sets~$J_\ppp$ and~$J_\mmm$, they are by definition inactive at initialisation, and by \autoref{ass:enum}~\ref{ass:enum.init} have no training points in their activation boundaries, so they do not change throughout the training.

\begin{lemma}
\label{l:len.w}
\begin{enumerate}[(i),itemsep=0ex,leftmargin=3em]
\item
\label{l:len.w.upp}
$\|\vec{w}_j^t\| < 2 \|\vec{z}_j\| \lambda^{1 - \varepsilon \, t / T_1}$
for all $j \in J_\ppp$ and all $t \in [0, T_1]$.
\item
\label{l:len.w.der}
$s_j \, \mathrm{d} \|\vec{w}_j^t\| / \mathrm{d} t \geq 0$
for all $j \in J_\ppp \cup J_\mmm$ and almost all $t \in [0, T_1]$.
\end{enumerate}
\end{lemma}

\begin{proof}
First we establish the following, which implies~\ref{l:len.w.upp}.

\begin{claim}
\label{cl:len.w}
$\|\vec{w}_j^t\| < 2 \|\vec{z}_j\| \lambda^{1 - \varepsilon \, t / T_1}$
for all $j \in [m]$ and all $t \in [0, T_1]$.
\end{claim}

\begin{proof}[Proof of claim]
Assume for a contradiction that this fails, and let $t \in [0, T_1]$ be the smallest such that $\|\vec{w}_j^t\| \geq 2 \|\vec{z}_j\| \lambda^{1 - \varepsilon \, t / T_1}$ for some $j \in [m]$.  Then we have
\begin{align*}
\|\vec{w}_j^t\|
& \leq \lambda \|\vec{z}_j\|
       \e^{t \max_{t' \in [0, t]} \left\|\vec{g}_j^{t'}\right\|}
& & \text{by \autoref{pr:prelim} and Gr\"onwall's inequality} \\
& \leq \lambda \|\vec{z}_j\|
       \e^{t \left(\;\!\!\|\vec{\gamma}_{[n]}\|
                       + \max_{t' \in [0, t]}^{i \in [n]}
                         \left|h_{\vec{\theta}^{t'}}(\vec{x}_i)\right| \|\vec{x}_i\|\;\!\!\right)}
& & \text{by the definition of~$\vec{g}_j^{t'}$
          in \autoref{pr:prelim}~\ref{pr:prelim.gj}} \\
& \leq \lambda \|\vec{z}_j\|
       \e^{t \left(\!\|\vec{\gamma}_{[n]}\|
                   + m \max_{t' \in [0, t]}^{j' \in [m], i \in [n]}
                       \left\|\vec{w}_{j'}^{t'}\right\|^2 \|\vec{x}_i\|^2\!\right)}
& & \text{by the definition of~$h_{\vec{\theta}^{t'}}$ in \autoref{s:prelim}} \\
& \leq \lambda \|\vec{z}_j\|
       \e^{t (\|\vec{\gamma}_{[n]}\|
            + \lambda^{2 - 2 \varepsilon} 4 m {\Delta\!}^4)}
& & \text{since $\left\|\vec{w}_{j'}^{t'}\right\|
            \leq 2 \left\|\vec{z}_{j'}\right\| \lambda^{1 - \varepsilon \, t' / T_1}$} \\
& \leq \lambda \|\vec{z}_j\|
       \e^{t \|\vec{\gamma}_{[n]}\|}
       \e^{T_1 \lambda^{2 - 2 \varepsilon} 4 m {\Delta\!}^4}
& & \text{since $t \leq T_1$} \\
& =    \|\vec{z}_j\|
       \lambda^{1 - \varepsilon \, t / T_1}
       \lambda^{-\varepsilon \lambda^{2 - 2 \varepsilon}
                 4 m {\Delta\!}^4 / \|\vec{\gamma}_{[n]}\|}
& & \text{since $\e^{T_1 \|\vec{\gamma}_{[n]}\|} = \lambda^{-\varepsilon}$} \\
& \leq \|\vec{z}_j\|
       \lambda^{1 - \varepsilon \, t / T_1}
       \lambda^{-\varepsilon \lambda^{2 - 2 \varepsilon}
                 4 m {\Delta\!}^4 / \delta^2}
& & \text{by \autoref{pr:d.D}~\ref{pr:d.D.gamma}} \\
& =    \|\vec{z}_j\|
       \lambda^{1 - \varepsilon \, t / T_1}
       \e^{\ln(\lambda^{-\varepsilon}) \lambda^{2 - 2 \varepsilon}
           4 m {\Delta\!}^4 / \delta^2}
& & \text{since $\exp$ and $\ln$ are inverses} \\
& <    \|\vec{z}_j\|
       \lambda^{1 - \varepsilon \, t / T_1}
       \e^{\lambda^{2 - 3 \varepsilon}
           4 m {\Delta\!}^4 / \delta^2}
& & \text{since $\lambda^{-\varepsilon} > \ln(\lambda^{-\varepsilon})$} \\
& <    \|\vec{z}_j\|
       \lambda^{1 - \varepsilon \, t / T_1}
       \e^{\lambda^{2 - 4 \varepsilon}}
& & \text{since $\lambda^{-\varepsilon}
            \geq m^3 \, n^{9 \cdot 3 n {\Delta\!}^2 / \delta^3}
               > 4 m {\Delta\!}^4 / \delta^2$} \\
& <  2 \|\vec{z}_j\|
       \lambda^{1 - \varepsilon \, t / T_1}
& & \text{since $\lambda^{2 - 4 \varepsilon}
            \leq \lambda
            \leq 2^{-9 \cdot 2 \cdot 3 \cdot 4}
               < \ln 2$.}
\qedhere
\end{align*}
\end{proof}

To prove~\ref{l:len.w.der}, observing that by \autoref{pr:prelim} for all $j \in [m]$ and almost all $t \in [0, \infty)$ we have
\[s_j \, \mathrm{d} \|\vec{w}_j^t\| / \mathrm{d} t
=         {\vec{w}_j^t}^\top \vec{g}_j^t
\in       \frac{1}{n}
          \sum_{i = 1}^n
          (y_i - h_{\vec{\theta}^t}(\vec{x}_i)) \,
          \partial \sigma({\vec{w}_j^t}^\top \vec{x}_i) \,
          {\vec{w}_j^t}^\top \vec{x}_i
\subseteq \sum_{i = 1}^n
          (y_i - h_{\vec{\theta}^t}(\vec{x}_i)) \,
          [0, \infty) \;,\]
it suffices to show that for all $t \in [0, T_1]$ and all $i \in [n]$ we have $|h_{\vec{\theta}^t}(\vec{x}_i)| \leq y_i$.
Indeed
\begin{align*}
|h_{\vec{\theta}^t}(\vec{x}_i)|
& \leq m \max_{j = 1}^m \|\vec{w}_j^t\|^2 \|\vec{x}_i\|
& & \text{by the definition of~$h_{\vec{\theta}^{t'}}$ in \autoref{s:prelim}} \\
& <    4 m {\Delta\!}^3 \lambda^{2 - 2 \varepsilon}
& & \text{by \autoref{cl:len.w}} \\
& <    \frac{\delta}{\sqrt{2}}
& & \text{since $\lambda^{2 \varepsilon - 2}
            \geq m^{3 \cdot 6} \, n^{9 \cdot 3 \cdot 6 n {\Delta\!}^2 / \delta^3}
               > 4 \sqrt{2} m {\Delta\!}^3 / \delta$} \\
& <    y_i
& & \text{since $\angle(\vec{v}^*, \vec{x}_i) < \pi / 4$.}
\qedhere
\end{align*}
\end{proof}

The next lemma provides a detailed description of the intermediate alignment stages for each hidden neuron.  It states that the training points enter or exit the active half-space of each positive-sign or negative-sign (respectively) hidden neuron in the same order as they do for the positive half-space of the corresponding yardstick trajectory, and it provides non-asymptotic bounds for each difference: between a dynamics-governing vector~$\vec{g}_j^t$ (defined in \autoref{pr:prelim}~\ref{pr:prelim.gj}) and the corresponding intermediate yardstick target~$\vec{\gamma}_{I_j^\ell}$ (the sets~$I_j^\ell$ of indices of training points that are in the active half-space of hidden neuron~$j$ during stage~$\ell$ are defined in \autoref{pr:omega}), between the unit-sphere normalisations of a hidden neuron and the corresponding yardstick vector, and between the corresponding boundary crossing times for a hidden neuron and its yardstick vector.  In particular, it shows that each negative-sign initially active hidden neuron~$j$ deactivates at time~$t_j^{n_j}$, and hence before time~$T_0$.  The proof of the lemma is inductive over the stage index~$\ell$, and involves carefully controlling the differences between the trajectories on the unit sphere of the hidden neurons and their yardstick vectors; this is non-trivial because, in contrast to the latter which have separate individual dynamics, the dynamics of the former are joint since each governing vector~$\vec{g}_j^t$ depends on the outputs of the whole network.

\begin{lemma}
\label{l:0}
For all $j \in J_\ppp \cup J_\mmm$ there exist unique $t_j^1, \ldots, t_j^{n_j} \in [0, \infty)$ such that for all $\ell \in [n_j]$ the following hold, where $t_j^0 \coloneqq 0$:
\begin{enumerate}[(i),itemsep=0ex,leftmargin=3em]
\item
\label{l:0.Ip}
$I_\ppp(\vec{w}_j^t) = I_j^\ell$ for all $t \in (t_j^{\ell - 1}, t_j^\ell)$;
\item
\label{l:0.I0}
$I_0(\vec{w}_j^t) = \emptyset$ for all $t \in (t_j^{\ell - 1}, t_j^\ell)$, and
$I_0\!\left(\!\vec{w}_j^{t_j^\ell}\right)\! = \{i_j^\ell\}$;
\item
\label{l:0.gamma-g}
$\|\vec{\gamma}_{I_j^\ell} - \vec{g}_j^t\|
 \leq \lambda^{2 - \varepsilon}$
for all $t \in (t_j^{\ell - 1}, t_j^\ell)$;
\item
\label{l:0.omega-w}
$\|\overline{\vec{\alpha}}_j^t - \overline{\vec{w}}_j^t\|
 \leq \lambda^{1 - \left(\!1 + \frac{3 \ell}{3 n_j}\!\right) \varepsilon}$
for all $t \in (t_j^{\ell - 1}, \max\{t_j^\ell, \tau_j^\ell\}]$;
\item
\label{l:0.tau-t}
$|\tau_j^\ell - t_j^\ell|
 \leq \lambda^{1 - \left(\!1 + \frac{3 \ell - 1}{3 n_j}\!\right) \varepsilon}$.
\end{enumerate}
\end{lemma}

\begin{proof}
For all $j \in J_\ppp \cup J_\mmm$ and all $\ell \in [n_j]$ let
\begin{align*}
\mathsf{t}_j^0 & \coloneqq
0
&
\mathsf{t}_j^{-\ell} & \coloneqq
\tau_j^\ell - \lambda^{1 - \left(\!1 + \frac{3 \ell - 1}{3 n_j}\!\right) \varepsilon}
&
\mathsf{t}_j^\ell & \coloneqq
\tau_j^\ell + \lambda^{1 - \left(\!1 + \frac{3 \ell - 1}{3 n_j}\!\right) \varepsilon} \;.
\end{align*}

For all $j \in J_\ppp \cup J_\mmm$ and all $\ell \in [n_j]$, define a continuous $\mathsf{w}_j^t$ by
\begin{align*}
\mathsf{w}_j^{\mathsf{t}_j^{\ell - 1}} \! & \coloneqq
\vec{\alpha}_j^{\mathsf{t}_j^{\ell - 1}}
&
\mathrm{d} \mathsf{w}_j^t / \mathrm{d} t & \coloneqq
s_j \|\mathsf{w}_j^t\| \, \vec{\gamma}_{I_j^\ell}
\quad\text{for all } t \in (\mathsf{t}_j^{\ell - 1}, \mathsf{t}_j^\ell] \;.
\end{align*}
Thus $\mathsf{w}_j^t$~equals~$\vec{\alpha}_j^t$ on $[\mathsf{t}_j^{\ell - 1}, \tau_j^\ell]$, and $\mathsf{w}_j^t$~continues with the same dynamics on $(\tau_j^\ell, \mathsf{t}_j^\ell]$.

We first observe that, until the largest of the times~$\mathsf{t}_j^\ell$, the network outputs at all the training points remain small.  Specifically, for all $t \in [0, \max_{j \in J_\ppp \cup J_\mmm} \mathsf{t}_j^{n_j}]$ we have
\begin{align*}
\frac{1}{n} \sum_{i = 1}^n |h_{\vec{\theta}^t}(\vec{x}_i)| \|\vec{x}_i\|
& \leq m \!\left(\!\max_{j = 1}^m \|\vec{w}_j^t\|^2\!\right)\!
       \!\left(\max_{i = 1}^n \|\vec{x}_i\|^2\right)\!
& & \text{by the definition of~$h_{\vec{\theta}^{t}}$ in \autoref{s:prelim}} \\
& <    4 m {\Delta\!}^4 \lambda^{2 - 2 \varepsilon / 3}
& & \text{by \autoref{pr:T0.T1} and \autoref{l:len.w}} \\
& <    \lambda^{2 - \varepsilon}
& & \text{since $\lambda^{-\varepsilon / 3}
            \geq m \, n^{9 n {\Delta\!}^2 / \delta^3}
               > 4 m {\Delta\!}^4$.}
\end{align*}
That shows that part~\ref{l:0.gamma-g} of the lemma is implied by parts~\ref{l:0.Ip}, \ref{l:0.I0}, and~\ref{l:0.tau-t}.

Let $j \in J_\ppp \cup J_\mmm$ be fixed for the remainder of the proof.

We proceed to show the lemma by induction on $\ell \in [n_j]$, where we make use of the following inductive hypothesis:
\begin{quote}
at $t = \mathsf{t}_j^{\ell - 1}$ we have
$I_\ppp(\vec{w}_j^t) = I_j^\ell$,
$I_0(\vec{w}_j^t) = \emptyset$, and
$\|\overline{\vec{\alpha}}_j^t - \overline{\vec{w}}_j^t\|
 \leq \lambda^{1 - \left(\!1 + \frac{3 \ell - 3}{3 n_j}\!\right) \varepsilon}$.
\end{quote}
That holds for $\ell = 1$ by \autoref{ass:enum}~\ref{ass:enum.init} and since $\overline{\vec{\alpha}}_j^0 = \overline{\vec{z}_j} = \overline{\lambda \vec{z}_j} = \overline{\vec{w}}_j^0$.

Consider $\ell \in [n_j]$.  By the inductive hypothesis, at $t = \mathsf{t}_j^{\ell - 1}$ we have $1 - {\overline{\mathsf{w}}_j^t}^\top \, \overline{\vec{w}}_j^t \leq \frac{1}{2} \lambda^{2 - 2 \left(\!1 + \frac{3 \ell - 3}{3 n_j}\!\right) \varepsilon}$.

Let $\mathsf{T} \coloneqq \min\{t > \mathsf{t}_j^{\ell - 1} \,\,\vert\,\, I_0(\vec{w}_j^t) \neq \emptyset\}$.

If $s_j = 1$, then for all $t \in (\mathsf{t}_j^{\ell - 1}, \min\{\mathsf{t}_j^\ell, \mathsf{T}\})$ we have
\begin{align*}
& \mathrm{d} (1 - {\overline{\mathsf{w}}_j^t}^\top \, \overline{\vec{w}}_j^t) /
  \mathrm{d} t
& & \\
& =  - {\overline{\mathsf{w}}_j^t}^\top \vec{g}_j^t
     - {\overline{\vec{w}}_j^t}^\top \vec{\gamma}_{I_j^\ell}
     + {\overline{\mathsf{w}}_j^t}^\top \, \overline{\vec{w}}_j^t \,
       ({\overline{\mathsf{w}}_j^t}^\top \vec{\gamma}_{I_j^\ell}
      + {\overline{\vec{w}}_j^t}^\top \vec{g}_j^t)
& & \text{by \autoref{cor:prelim} and since $I_0(\vec{w}_j^t) = \emptyset$} \\
& <    2 \lambda^{2 - \varepsilon}
     - (1 - {\overline{\mathsf{w}}_j^t}^\top \, \overline{\vec{w}}_j^t) \,
       {(\overline{\mathsf{w}}_j^t + \overline{\vec{w}}_j^t)\!}^\top
       \vec{\gamma}_{I_j^\ell}
& & \text{since $\|\vec{\gamma}_{I_j^\ell} - \vec{g}_j^t\|
               < \lambda^{2 - \varepsilon}$} \\
& \leq 2 \lambda^{2 - \varepsilon}
     - (1 - {\overline{\mathsf{w}}_j^t}^\top \, \overline{\vec{w}}_j^t) \;
       {\overline{\mathsf{w}}_j^t}^\top
       \vec{\gamma}_{I_j^\ell}
& & \text{since $I_\ppp(\vec{w}_j^t) = I_j^\ell$} \\
& \leq 2 \lambda^{2 - \varepsilon}
     - (1 - {\overline{\mathsf{w}}_j^t}^\top \, \overline{\vec{w}}_j^t)
       \left\|\vec{\gamma}_{I_j^1}\right\| \cos \varphi_j^0
& & \text{by \autoref{pr:omega}~\ref{pr:omega.cos} and~\ref{pr:omega.inter}} \\
& \leq 2 \lambda^{2 - \varepsilon}
     - (1 - {\overline{\mathsf{w}}_j^t}^\top \, \overline{\vec{w}}_j^t) \,
       \tfrac{\delta^{7 / 2}}{\sqrt{2} n}
& & \text{by \autoref{pr:d.D}~\ref{pr:d.D.gamma}} \\
& \leq 4 \lambda^{\left(\!1 + 2 \frac{3 \ell - 3}{3 n_j}\!\right) \varepsilon}
       \!\left(\!\tfrac{1}{2}
                 \lambda^{2 - 2 \left(\!1 + \frac{3 \ell - 3}{3 n_j}\!\right) \varepsilon}
               - (1 - {\overline{\mathsf{w}}_j^t}^\top \, \overline{\vec{w}}_j^t)
       \!\right)\!
& & \text{since $\lambda^\varepsilon
            \leq n^{-27 n / \delta^3}\!
               < {\left(\tfrac{\delta^3}{n}\right)\!}^{27}\!
               < \tfrac{\delta^{7 / 2}}{4 \sqrt{2} n}$,}
\end{align*}
so for all $t \in (\mathsf{t}_j^{\ell - 1}, \min\{\mathsf{t}_j^\ell, \mathsf{T}\}]$ we have $1 - {\overline{\mathsf{w}}_j^t}^\top \, \overline{\vec{w}}_j^t \leq \frac{1}{2} \lambda^{2 - 2 \left(\!1 + \frac{3 \ell - 3}{3 n_j}\!\right) \varepsilon}$ and thus $\|\overline{\mathsf{w}}_j^t - \overline{\vec{w}}_j^t\| \leq \lambda^{1 - \left(\!1 + \frac{3 \ell - 3}{3 n_j}\!\right) \varepsilon}$.

If $s_j = -1$ then for all $t \in (\mathsf{t}_j^{\ell - 1}, \min\{\mathsf{t}_j^\ell, \mathsf{T}\})$ we have
\begin{align*}
& \mathrm{d}
  \!\left(\!
  \tfrac{1}{2}
  \lambda^{2 - 2 \left(\!1 + \frac{3 \ell - 3}{3 n_j}\!\right) \varepsilon} +
  (1 - {\overline{\mathsf{w}}_j^t}^\top \, \overline{\vec{w}}_j^t)
  \!\right)\! /
  \mathrm{d} t
& & \\
& =    {\overline{\mathsf{w}}_j^t}^\top \vec{g}_j^t
     + {\overline{\vec{w}}_j^t}^\top \vec{\gamma}_{I_j^\ell}
     - {\overline{\mathsf{w}}_j^t}^\top \, \overline{\vec{w}}_j^t \,
       ({\overline{\mathsf{w}}_j^t}^\top \vec{\gamma}_{I_j^\ell}
      + {\overline{\vec{w}}_j^t}^\top \vec{g}_j^t)
& & \text{by \autoref{cor:prelim} and since $I_0(\vec{w}_j^t) = \emptyset$} \\
& <    2 \lambda^{2 - \varepsilon}
     + (1 - {\overline{\mathsf{w}}_j^t}^\top \, \overline{\vec{w}}_j^t) \,
       {(\overline{\mathsf{w}}_j^t + \overline{\vec{w}}_j^t)\!}^\top
       \vec{\gamma}_{I_j^\ell}
& & \text{since $\|\vec{\gamma}_{I_j^\ell} - \vec{g}_j^t\|
               < \lambda^{2 - \varepsilon}$} \\
& \leq 2 \left\|\vec{\gamma}_{I_j^\ell}\right\|
       \!\left(\!
       \tfrac{1}{2}
       \lambda^{2 - 2 \left(\!1 + \frac{3 \ell - 3}{3 n_j}\!\right) \varepsilon} +
       (1 - {\overline{\mathsf{w}}_j^t}^\top \, \overline{\vec{w}}_j^t)
       \!\right)\!
& & \text{since $\lambda^\varepsilon
            \leq n^{-27 n / \delta^3}\!
               < \tfrac{\delta^{5 / 2}}{2 \sqrt{2} n}$,}
\end{align*}
so for all $t \in (\mathsf{t}_j^{\ell - 1}, \min\{\mathsf{t}_j^\ell, \mathsf{T}\}]$ we have
\begin{align*}
& 1 - {\overline{\mathsf{w}}_j^t}^\top \, \overline{\vec{w}}_j^t
& & \\
& \leq \lambda^{2 - 2 \left(\!1 + \frac{3 \ell - 3}{3 n_j}\!\right) \varepsilon}
       \!\left(\exp\!\left(2 \left\|\vec{\gamma}_{I_j^\ell}\right\|
                             (t - \mathsf{t}_j^{\ell - 1})\!\right)\!
                 - \tfrac{1}{2}\right)\!
& & \text{by Gr\"onwall's inequality} \\
& <    \lambda^{2 - 2 \left(\!1 + \frac{3 \ell - 3}{3 n_j}\!\right) \varepsilon}
       \!\left(\exp\!\left(2 \left\|\vec{\gamma}_{I_j^\ell}\right\|
                             (\tau_j^\ell - \tau_j^{\ell - 1} + \lambda^{1 - 2 \varepsilon})\!\right)\!
                 - \tfrac{1}{2}\right)\!
& & \text{by the definitions of~$\mathsf{t}_j^0$ and~$\mathsf{t}_j^\ell$} \\
& <    \lambda^{2 - 2 \left(\!1 + \frac{3 \ell - 3}{3 n_j}\!\right) \varepsilon}
       \!\left(\exp\!\left(3 \left\|\vec{\gamma}_{I_j^\ell}\right\|
                             (\tau_j^\ell - \tau_j^{\ell - 1})\!\right)\!
                 - \tfrac{1}{2}\right)\!
& & \text{since $\lambda^{1 - 2 \varepsilon}
            \leq n^{-9 \cdot 6 n / \delta^3}\!
               < \tfrac{\delta}{2}$} \\
& <    \lambda^{2 - 2 \left(\!1 + \frac{3 \ell - 3}{3 n_j}\!\right) \varepsilon}
       \!\left(\tfrac{8}{\delta^3} - \tfrac{1}{2}\right)\!
& & \begin{aligned}[t]
    & \text{by \autoref{pr:omega}~\ref{pr:omega.cos}} \\
    & \text{and \autoref{pr:d.D}~\ref{pr:d.D.varphim}}
    \end{aligned} \\
& <    \tfrac{1}{2}
       \lambda^{2 - 2 \left(\!1 + \frac{3 \ell - 2}{3 n_j}\!\right) \varepsilon}
& & \text{since $\lambda^{-\frac{2 \varepsilon}{3 n}}
            \geq n^{9 \cdot 2 / \delta^3}
               > \tfrac{16}{\delta^3}$}
\end{align*}
and thus $\|\overline{\mathsf{w}}_j^t - \overline{\vec{w}}_j^t\| \leq \lambda^{1 - \left(\!1 + \frac{3 \ell - 2}{3 n_j}\!\right) \varepsilon}$.

For all $i \in [n]$ such that $i \neq i_j^\ell$ and if $\ell \neq 1$ then $i \neq i_j^{\ell - 1}$, for all $t \in [\tau_j^\ell, \mathsf{t}_j^\ell]$ we have
\begin{align*}
|{\overline{\mathsf{w}}_j^t}^\top \, \overline{\vec{x}}_i|
& \geq \delta - \!\left(\max_{t' \in [\tau_j^\ell, t]}
                        \left\|\mathrm{d} \overline{\mathsf{w}}_j^{t'} /
                               \mathrm{d} t'\right\|\right)\!
                (t - \tau_j^\ell)
& & \text{since $|{\overline{\mathsf{w}}_j^t}^\top \, \overline{\vec{x}}_i| \geq \delta$
          for $t = \tau_j^\ell$} \\
& \geq \delta - \left\|\vec{\gamma}_{I_j^\ell}\right\|
                (t - \tau_j^\ell)
& & \text{by properties of projection} \\
& \geq \delta - {\Delta\!}^2 (t - \tau_j^\ell)
& & \text{by \autoref{pr:d.D}~\ref{pr:d.D.gamma}} \\
& >    \delta - {\Delta\!}^2 \lambda^{1 - 2 \varepsilon}
& & \text{by the definition of $\mathsf{t}_j^\ell$} \\
& >    \delta / 2
& & \text{since $\lambda^{1 - 2 \varepsilon}
            \leq n^{-9 \cdot 6 n {\Delta\!}^2 / \delta^3}\!
               < \tfrac{\delta}{2 {\Delta\!}^2}$,}
\end{align*}
so for all $t \in (\mathsf{t}_j^{\ell - 1}, \min\{\mathsf{t}_j^\ell, \mathsf{T}\}]$ we have $|{\overline{\vec{w}}_j^t}^\top \, \overline{\vec{x}}_i| > \delta / 2 - \lambda^{1 - \left(\!1 + \frac{3 \ell - 2}{3 n_j}\!\right) \varepsilon} > \delta / 4$.

If $\ell \neq 1$ then for all $t \in (\mathsf{t}_j^{\ell - 1}, \min\{\mathsf{t}_j^\ell, \mathsf{T}\})$ we have
\begin{align*}
& \mathrm{d} \, s_j \, {\overline{\vec{w}}_j^t}^\top \,
                       \overline{\vec{x}}_{i_j^{\ell - 1}} \Big/
  \mathrm{d} t
& & \\
& =    {\vec{g}_j^t}^\top \, \overline{\vec{x}}_{i_j^{\ell - 1}}
     - {\overline{\vec{w}}_j^t}^\top \vec{g}_j^t \;
       {\overline{\vec{w}}_j^t}^\top \, \overline{\vec{x}}_{i_j^{\ell - 1}}
& & \text{by \autoref{cor:prelim} and since $I_0(\vec{w}_j^t) = \emptyset$} \\
& >    {\vec{\gamma}_{I_j^\ell}\!}^\top \, \overline{\vec{x}}_{i_j^{\ell - 1}}
     - \left\|\vec{\gamma}_{I_j^\ell}\right\|
       \left|{\overline{\vec{w}}_j^t}^\top \,
             \overline{\vec{x}}_{i_j^{\ell - 1}}\right|
     - 2 \lambda^{2 - \varepsilon}
& & \text{since $\|\vec{\gamma}_{I_j^\ell} - \vec{g}_j^t\|
               < \lambda^{2 - \varepsilon}$} \\
& >    \frac{\delta^3}{\sqrt{2} n}
     - \left\|\vec{\gamma}_{I_j^\ell}\right\|
       \left|{\overline{\vec{w}}_j^t}^\top \,
             \overline{\vec{x}}_{i_j^{\ell - 1}}\right|
     - 2 \lambda^{2 - \varepsilon}
& & \text{recalling $\forall i \in [n] \colon \angle(\vec{v}^*, \vec{x}_i) < \pi / 4$} \\
& \geq \frac{\delta^3}{\sqrt{2} n}
     - {\Delta\!}^2
       \left|{\overline{\vec{w}}_j^t}^\top \,
             \overline{\vec{x}}_{i_j^{\ell - 1}}\right|
     - 2 \lambda^{2 - \varepsilon}
& & \text{by \autoref{pr:d.D}~\ref{pr:d.D.gamma}} \\
& >    \frac{\delta^3}{2 \sqrt{2} n}
     - {\Delta\!}^2
       \left|{\overline{\vec{w}}_j^t}^\top \,
             \overline{\vec{x}}_{i_j^{\ell - 1}}\right|
& & \text{since $\lambda^{2 - \varepsilon}
            \leq n^{-9 \cdot 3 \cdot 7 n / \delta^3}\!
               < \tfrac{\delta^3}{4 \sqrt{2} n}$,}
\end{align*}
so $i_j^{\ell - 1} \notin I_0(\vec{w}_j^t)$ at $t = \min\{\mathsf{t}_j^\ell, \mathsf{T}\}$ since otherwise the continuous curve ${\overline{\vec{w}}_j^t}^\top \, \overline{\vec{x}}_{i_j^{\ell - 1}}$ would have different signs to the right of~$\mathsf{t}_j^{\ell - 1}$ and to the left of $\min\{\mathsf{t}_j^\ell, \mathsf{T}\}$ without crossing zero in between.

Assume for a contradiction that $\mathsf{T} < \mathsf{t}_j^{-\ell}$.  Then $I_0(\vec{w}_j^{\mathsf{T}}) = \{i_j^\ell\}$.  But also
\begin{align*}
\left|{\overline{\vec{w}}_j^{\mathsf{T}}}^\top \,
      \overline{\vec{x}}_{i_j^\ell}\right|
& \geq \left|{\overline{\mathsf{w}}_j^{\mathsf{T}}}^\top \,
             \overline{\vec{x}}_{i_j^\ell}\right|
     - \lambda^{1 - \left(\!1 + \frac{3 \ell - 2}{3 n_j}\!\right) \varepsilon}
& & \text{since $\left\|\overline{\mathsf{w}}_j^{\mathsf{T}}
                      - \overline{\vec{w}}_j^{\mathsf{T}}\right\|
            \leq \lambda^{1 - \left(\!1 + \frac{3 \ell - 2}{3 n_j}\!\right)
                              \varepsilon}$} \\
& =    \left|{\overline{\vec{\alpha}}_j^{\mathsf{T}}}^\top \,
             \overline{\vec{x}}_{i_j^\ell}\right|
     - \lambda^{1 - \left(\!1 + \frac{3 \ell - 2}{3 n_j}\!\right) \varepsilon}
& & \text{since $\mathsf{w}_j^{\mathsf{T}} = \vec{\alpha}_j^{\mathsf{T}}$} \\
& \geq \tfrac{2 \delta^4}{3 n}
       (\tau_j^\ell - {\mathsf{T}})
     - \lambda^{1 - \left(\!1 + \frac{3 \ell - 2}{3 n_j}\!\right) \varepsilon}
& & \text{by \autoref{pr:d.D}~\ref{pr:d.D.ox}} \\
& >    \tfrac{2 \delta^4}{3 n}
       \lambda^{1 - \left(\!1 + \frac{3 \ell - 1}{3 n_j}\!\right) \varepsilon}
     - \lambda^{1 - \left(\!1 + \frac{3 \ell - 2}{3 n_j}\!\right) \varepsilon}
& & \text{by the definition of~$\mathsf{t}_j^{-\ell}$} \\
& =    \lambda^{1 - \left(\!1 + \frac{3 \ell - 1}{3 n_j}\!\right) \varepsilon}
       \!\left(\!\tfrac{2 \delta^4}{3 n}
               - \lambda^{\frac{\varepsilon}{3 n_j}}\!\right)\!
& & \text{calculation} \\
& >    0
& & \text{since $\lambda^{\frac{\varepsilon}{3 n}}
            \leq n^{-9 / \delta^3}\!
               < \tfrac{\delta^{3 \cdot 7}}{n^2}$.}
\end{align*}

For all $t \in [\mathsf{t}_j^{-\ell}, \mathsf{t}_j^\ell]$ we have
\begin{align*}
& \left|\mathrm{d} \, {\overline{\mathsf{w}}_j^t}^\top \, \overline{\vec{x}}_{i_j^\ell}
  \Big/ \mathrm{d} t\right|
& & \\
& =    \left|{\vec{\gamma}_{I_j^\ell}\!}^\top \, \overline{\vec{x}}_{i_j^\ell}
           - {\overline{\mathsf{w}}_j^t}^\top \vec{\gamma}_{I_j^\ell} \;
             {\overline{\mathsf{w}}_j^t}^\top \, \overline{\vec{x}}_{i_j^\ell}\right|
& & \text{by the definition of~$\mathsf{w}_j^t$} \\
& \geq \left|{\vec{\gamma}_{I_j^\ell}\!}^\top \, \overline{\vec{x}}_{i_j^\ell}\right|
     - \left|{\overline{\mathsf{w}}_j^t}^\top \vec{\gamma}_{I_j^\ell} \;
             {\overline{\mathsf{w}}_j^t}^\top \, \overline{\vec{x}}_{i_j^\ell}\right|
& & \text{by properties of absolute value} \\
& >    \frac{\delta^3}{\sqrt{2} n}
     - \left|{\overline{\mathsf{w}}_j^t}^\top \vec{\gamma}_{I_j^\ell} \;
             {\overline{\mathsf{w}}_j^t}^\top \, \overline{\vec{x}}_{i_j^\ell}\right|
& & \text{recalling $\forall i \in [n] \colon \angle(\vec{v}^*, \vec{x}_i) < \pi / 4$} \\
& \geq \frac{\delta^3}{\sqrt{2} n}
     - {\Delta\!}^2
       \left|{\overline{\mathsf{w}}_j^t}^\top \, \overline{\vec{x}}_{i_j^\ell}\right|
& & \text{by \autoref{pr:d.D}~\ref{pr:d.D.gamma}} \\
& \geq \frac{\delta^3}{\sqrt{2} n}
     - {\Delta\!}^2
       \!\left(\max_{t' \in [\mathsf{t}_j^{-\ell}, \mathsf{t}_j^\ell]}
               \left\|\mathrm{d} \overline{\mathsf{w}}_j^{t'} /
                      \mathrm{d} t'\right\|\right)\!
       \left|t - \tau_j^\ell\right|
& & \text{since ${\overline{\mathsf{w}}_j^t}^\top \, \overline{\vec{x}}_{i_j^\ell} = 0$
          at $t = \tau_j^\ell$} \\
& \geq \frac{\delta^3}{\sqrt{2} n}
     - {\Delta\!}^4
       \left|t - \tau_j^\ell\right|
& & \text{by properties of projection} \\
& >    \frac{\delta^3}{\sqrt{2} n}
    - {\Delta\!}^4
      \lambda^{1 - 2 \varepsilon}
& & \text{by the definitions of~$\mathsf{t}_j^{-\ell}$ and~$\mathsf{t}_j^\ell$} \\
& >   \frac{\delta^3}{2 \sqrt{2} n}
& & \text{since $\lambda^{1 - 2 \varepsilon}
            \leq n^{-9 \cdot 6 n {\Delta\!}^2 / \delta^3}\!
               < \tfrac{\delta^3}{2 \sqrt{2} n {\Delta\!}^4}$} \\
& >   \lambda^{\frac{\varepsilon}{3 n}}
& & \text{since $\lambda^{\frac{\varepsilon}{3 n}}
            \leq n^{-9 / \delta^3}$,}
\end{align*}
so at $t = \mathsf{t}_j^{-\ell}$ and at $t = \mathsf{t}_j^\ell$ it holds that ${\overline{\mathsf{w}}_j^t}^\top \, \overline{\vec{x}}_{i_j^\ell}$ has different signs and absolute values greater than
$\lambda^{\frac{\varepsilon}{3 n}}
 \lambda^{1 - \left(\!1 + \frac{3 \ell - 1}{3 n_j}\!\right) \varepsilon}
 \geq \lambda^{1 - \left(\!1 + \frac{3 \ell - 2}{3 n_j}\!\right) \varepsilon}$.

Assume for a contradiction that $\mathsf{T} > \mathsf{t}_j^\ell$.  Then at $t = \mathsf{t}_j^{-\ell}$ and at $t = \mathsf{t}_j^\ell$ it holds that $\|\overline{\mathsf{w}}_j^t - \overline{\vec{w}}_j^t\| \leq \lambda^{1 - \left(\!1 + \frac{3 \ell - 2}{3 n_j}\!\right) \varepsilon}$, so ${\vec{w}_j^t}^\top \, \overline{\vec{x}}_{i_j^\ell}$ has different signs.

Therefore $\mathsf{T} \in [\mathsf{t}_j^{-\ell}, \mathsf{t}_j^\ell]$ and $I_0(\vec{w}_j^{\mathsf{T}}) = \{i_j^\ell\}$.  Let $t_j^\ell \coloneqq \mathsf{T}$.

To complete the proof, it suffices to show that, for all $t \in (t_j^\ell, \mathsf{t}_j^\ell]$, we have
$\|\overline{\vec{\alpha}}_j^t - \overline{\vec{w}}_j^t\|
 \leq \lambda^{1 - \left(\!1 + \frac{3 \ell}{3 n_j}\!\right) \varepsilon}$,
and if $\ell \neq n_j$ then
$I_\ppp(\vec{w}_j^t) = I_j^{\ell + 1}$ and
$I_0(\vec{w}_j^t) = \emptyset$.

Since at $t = \min\{\tau_j^\ell, t_j^\ell\}$ we have $\|\overline{\vec{\alpha}}_j^t - \overline{\vec{w}}_j^t\| = \|\overline{\mathsf{w}}_j^t - \overline{\vec{w}}_j^t\| \leq \lambda^{1 - \left(\!1 + \frac{3 \ell - 2}{3 n_j}\!\right) \varepsilon}$, and for almost all $t \in (\min\{\tau_j^\ell, t_j^\ell\}, \mathsf{t}_j^\ell)$ we have
\begin{align*}
|\mathrm{d} \|\overline{\vec{\alpha}}_j^t - \overline{\vec{w}}_j^t\| / \mathrm{d} t|
& \leq \|\mathrm{d} (\overline{\vec{\alpha}}_j^t - \overline{\vec{w}}_j^t) / \mathrm{d} t\|
& & \text{by properties of projection} \\
& \leq \|\mathrm{d} \overline{\vec{\alpha}}_j^t / \mathrm{d} t\|
     + \|\mathrm{d} \overline{\vec{w}}_j^t / \mathrm{d} t\|
& & \text{by the triangle inequality} \\
& \leq \left(\left\|\vec{\gamma}_{I_\ppp(\vec{\alpha}_j^t)}\right\|
           + \|\vec{g}_j^t\|\right)
& & \text{by properties of projection} \\
& \leq \left(\left\|\vec{\gamma}_{I_\ppp(\vec{\alpha}_j^t)}\right\|
         + 2 \left\|\vec{\gamma}_{I_\ppp(\vec{w}_j^t) \,\cup\,
                                     I_0(\vec{w}_j^t)}\right\|\right)
& & \begin{aligned}[t]
    & \text{since $\forall i \in [n] \colon |h_{\vec{\theta}^t}(\vec{x}_i)| \leq y_i$} \\
    & \text{by the proof of \autoref{l:len.w}~\ref{l:len.w.der}}
    \end{aligned} \\
& \leq 3 {\Delta\!}^2
& & \text{by \autoref{pr:d.D}~\ref{pr:d.D.gamma},}
\end{align*}
it follows that for all $t \in [\!\min\{\tau_j^\ell, t_j^\ell\}, \mathsf{t}_j^\ell]$ we have
\begin{align*}
\|\overline{\vec{\alpha}}_j^t - \overline{\vec{w}}_j^t\|
& \leq \lambda^{1 - \left(\!1 + \frac{3 \ell - 2}{3 n_j}\!\right) \varepsilon}\!
     + 6 {\Delta\!}^2
       \lambda^{1 - \left(\!1 + \frac{3 \ell - 1}{3 n_j}\!\right) \varepsilon}
& & \text{by the definitions of~$\mathsf{t}_j^{-\ell}$ and~$\mathsf{t}_j^\ell$} \\
& <    7 {\Delta\!}^2
       \lambda^{1 - \left(\!1 + \frac{3 \ell - 1}{3 n_j}\!\right) \varepsilon}
& & \text{since $\lambda^{1 - \left(\!1 + \frac{3 \ell - 2}{3 n_j}\!\right) \varepsilon}\!
               < \lambda^{1 - \left(\!1 + \frac{3 \ell - 1}{3 n_j}\!\right) \varepsilon}$}
    \\
& <    \lambda^{1 - \left(\!1 + \frac{3 \ell}{3 n_j}\!\right) \varepsilon}
& & \text{since $\lambda^{-\frac{\varepsilon}{3 n}}
            \geq n^{9 {\Delta\!}^2}
               > n^3 {\Delta\!}^{2 \cdot 6}$.}
\end{align*}

If $\ell \neq n_j$ then for all $i \neq i_j^\ell$ and all $t \in [\mathsf{t}_j^{-\ell}, \mathsf{t}_j^\ell]$ we have
\[|{\overline{\vec{\alpha}}_j^t}^\top \, \overline{\vec{x}}_i|
  \geq \delta - {\Delta\!}^2 |t - \tau_j^\ell|
  >    \delta - {\Delta\!}^2 \lambda^{1 - 2 \varepsilon}
  >    \delta / 2 \;,\]
so for all $t \in (t_j^\ell, \mathsf{t}_j^\ell]$ we have $|{\overline{\vec{w}}_j^t}^\top \, \overline{\vec{x}}_i| > \delta / 2 - \lambda^{1 - \left(\!1 + \frac{3 \ell}{3 n_j}\!\right) \varepsilon} > \delta / 4$.  Also, for almost all $t \in (t_j^\ell, \mathsf{t}_j^\ell)$ the following holds, where $I \coloneqq I_j^\ell \cap I_j^{\ell + 1}$:
\begin{align*}
& \mathrm{d} \, s_j \, {\overline{\vec{w}}_j^t}^\top \,
                       \overline{\vec{x}}_{i_j^\ell} \Big/
  \mathrm{d} t
& & \\
& =    {\vec{g}_j^t}^\top \, \overline{\vec{x}}_{i_j^\ell}
     - {\overline{\vec{w}}_j^t}^\top \vec{g}_j^t \;
       {\overline{\vec{w}}_j^t}^\top \, \overline{\vec{x}}_{i_j^\ell}
& & \text{by \autoref{cor:prelim}} \\
& >    {\mathsf{g}_j^t}^\top \, \overline{\vec{x}}_{i_j^\ell}
     - {\overline{\vec{w}}_j^t}^\top \mathsf{g}_j^t
       \left|{\overline{\vec{w}}_j^t}^\top \,
             \overline{\vec{x}}_{i_j^\ell}\right|
     - 2 \lambda^{2 - \varepsilon}
& & \begin{aligned}[t]
    & \text{since $\tfrac{1}{n} {\textstyle \sum_{i = 1}^n}
                                |h_{\vec{\theta}^t}(\vec{x}_i)| \|\vec{x}_i\|
                 < \lambda^{2 - \varepsilon}$,} \\
    & \text{where $\exists \, \varsigma_j^t \in [0, 1] \colon
                   \mathsf{g}_j^t
                 = \vec{\gamma}_I
                 + \tfrac{1}{n} y_{i_j^\ell} \, \varsigma_j^t \, \vec{x}_{i_j^\ell}$}
    \end{aligned} \\
& \geq {\vec{\gamma}_I\!}^\top \, \overline{\vec{x}}_{i_j^\ell}
     - {\overline{\vec{w}}_j^t}^\top \vec{\gamma}_I
       \left|{\overline{\vec{w}}_j^t}^\top \,
             \overline{\vec{x}}_{i_j^\ell}\right|
     - 2 \lambda^{2 - \varepsilon}
& & \text{since ${\vec{x}_{i_j^\ell}\!}^\top \, \overline{\vec{x}}_{i_j^\ell}
            \geq {\overline{\vec{w}}_j^t}^\top \vec{x}_{i_j^\ell}
                 \left|{\overline{\vec{w}}_j^t}^\top \,
                       \overline{\vec{x}}_{i_j^\ell}\right|$} \\
& \geq {\vec{\gamma}_I\!}^\top \, \overline{\vec{x}}_{i_j^\ell}
     - \|\vec{\gamma}_I\|
       \left|{\overline{\vec{w}}_j^t}^\top \,
             \overline{\vec{x}}_{i_j^\ell}\right|
     - 2 \lambda^{2 - \varepsilon}
& & \text{since ${\overline{\vec{w}}_j^t}^\top \vec{\gamma}_I
            \leq \|\vec{\gamma}_I\|$} \\
& >    \frac{\delta^3}{\sqrt{2} n}
     - \|\vec{\gamma}_I\|
       \left|{\overline{\vec{w}}_j^t}^\top \,
             \overline{\vec{x}}_{i_j^\ell}\right|
     - 2 \lambda^{2 - \varepsilon}
& & \text{recalling $\forall i \in [n] \colon \angle(\vec{v}^*, \vec{x}_i) < \pi / 4$} \\
& \geq \frac{\delta^3}{\sqrt{2} n}
     - {\Delta\!}^2
       \left|{\overline{\vec{w}}_j^t}^\top \,
             \overline{\vec{x}}_{i_j^\ell}\right|
     - 2 \lambda^{2 - \varepsilon}
& & \text{by \autoref{pr:d.D}~\ref{pr:d.D.gamma}} \\
& >    \frac{\delta^3}{2 \sqrt{2} n}
     - {\Delta\!}^2
       \left|{\overline{\vec{w}}_j^t}^\top \,
             \overline{\vec{x}}_{i_j^\ell}\right|
& & \text{since $\lambda^{2 - \varepsilon}
            \leq n^{-9 \cdot 3 \cdot 7 n / \delta^3}\!
               < \tfrac{\delta^3}{4 \sqrt{2} n}$.}
\end{align*}
Hence $i_j^\ell \notin I_0\bigl(\vec{w}_j^{t'}\bigr)$ for all $t' \in (t_j^\ell, \mathsf{t}_j^\ell]$ since otherwise the continuous curve ${\overline{\vec{w}}_j^t}^\top \, \overline{\vec{x}}_{i_j^\ell}$ would have different signs to the right of~$t_j^\ell$ and to the left of the smallest such~$t'$ without crossing zero in between.
\end{proof}

For each positive-sign hidden neuron, after the completion of all of the intermediate alignment stages (if any), its yardstick vector proceeds to align to the vector~$\vec{\gamma}_{[n]}$.  For the cosine of the angle between them, whose starting point is the angle~$\varphi_j^{n_j^+}$ defined below, from the proof of \autoref{pr:omega}~\ref{pr:omega.cos} we obtain the following expression, which will be useful in the proof of the next lemma.  We also remark that the angle~$\varphi_j^{n_j^+}$ already featured in \autoref{pr:omega}~\ref{pr:omega.finalp}.

\begin{proposition}
\label{pr:0+}
For all $j \in J_\ppp$ and all $t > \tau_j^{n_j}$ we have
$I_\ppp(\vec{\alpha}_j^t) = [n]$ and
\[\cos \varphi_j^t =
  \tanh\!\left(\!
  \artanh \cos \varphi_j^{n_j^+} \! +
  \|\vec{\gamma}_{[n]}\| (t - \tau_j^{n_j})
  \!\right)\! \;,\]
where
$\varphi_j^{n_j^+} \! \coloneqq
 \lim_{t \to \bigl(\tau_j^{n_j}\bigr)^+}
 \varphi_j^t$.
\end{proposition}

The final lemma in this section establishes non-asymptotic bounds for the final alignment stage in the first phase of the training, which we consider to end at time~$T_1$.  During it, each positive-sign hidden neuron: continues to grow but keeps its length below $2 \|\vec{z}_j\| \lambda^{1 - \varepsilon}$, aligns to the vector~$\vec{\gamma}_{[n]}$ up to a cosine of at least $1 - \lambda^\varepsilon$, and maintains bounded by~$\lambda^{1 - 3 \varepsilon}$ the difference between the logarithm of its length divided by the initialisation scale and the logarithm of the corresponding yardstick vector length.  Establishing the latter bound, which is stated in part~\ref{l:1.lengths} and will be instrumental in the proof of the implicit bias (cf.~\autoref{l:lim}), involves putting together the bounds in \autoref{l:0}~\ref{l:0.gamma-g}, \ref{l:0.omega-w}, and~\ref{l:0.tau-t}, and \autoref{l:1}~\ref{l:1.gamma.g} and~\ref{l:1.omega.w} on the dynamics-governing vectors, the unit-sphere normalisations, and the boundary crossing times, over the lengthy time period up to~$T_1$ which depends on the initialisation scale~$\lambda$.

\begin{lemma}
\label{l:1}
For all $j \in J_\ppp$ we have:
\begin{enumerate}[(i),itemsep=0ex,leftmargin=3em]
\item
\label{l:1.gamma.g}
$\|\vec{\gamma}_{[n]} - \vec{g}_j^t\|
 \leq \lambda^{2 - 3 \varepsilon}$
for all $t \in (t_j^{n_j}, T_1]$;
\item
\label{l:1.omega.w}
$\|\overline{\vec{\alpha}}_j^t - \overline{\vec{w}}_j^t\|
 \leq \lambda^{1 - 2 \varepsilon}$
for all $t \in (t_j^{n_j}, T_1]$;
\item
\label{l:1.w.gamma}
${\overline{\vec{w}}_j^{T_1}\!}^\top \, \overline{\vec{\gamma}}_{[n]}
 \geq 1 - \lambda^\varepsilon$;
\item
\label{l:1.lengths}
$|\ln \|\vec{\alpha}_j^{T_1}\| - \ln \|\vec{w}_j^{T_1} \! / \lambda\||
 \leq \lambda^{1 - 3 \varepsilon}$.
\end{enumerate}
\end{lemma}

\begin{proof}
For all $t \in [0, T_1]$ we have
\begin{align*}
\frac{1}{n} \sum_{i = 1}^n |h_{\vec{\theta}^t}(\vec{x}_i)| \|\vec{x}_i\|
& \leq m \!\left(\!\max_{j = 1}^m \|\vec{w}_j^t\|^2\!\right)\!
       \!\left(\max_{i = 1}^n \|\vec{x}_i\|^2\right)\!
& & \text{by the definition of~$h_{\vec{\theta}^{t}}$ in \autoref{s:prelim}} \\
& <    4 m {\Delta\!}^4 \lambda^{2 - 2 \varepsilon}
& & \text{by \autoref{l:len.w}} \\
& <    \lambda^{2 - 3 \varepsilon}
& & \text{since $\lambda^{-\varepsilon}
            \geq m^3 \, n^{9 \cdot 3 n {\Delta\!}^2 / \delta^3}
               > 4 m {\Delta\!}^4$.}
\end{align*}

Suppose $j \in J_\ppp$.

Assume for a contradiction that $I_0(\vec{w}_j^t) \neq \emptyset$ for some $t \in (t_j^{n_j}, T_1]$, and let $\mathsf{T} > t_j^{n_j}$ be the smallest such that ${\vec{w}_j^{\mathsf{T}}}^\top \vec{x}_i = 0$ for some $i \in [n]$.  Then for all $t \in (t_j^{n_j}, \mathsf{T})$ we have
\begin{align*}
\mathrm{d} \, {\overline{\vec{w}}_j^t}^\top \, \overline{\vec{x}}_i /
\mathrm{d} t
& =    {\vec{g}_j^t}^\top \, \overline{\vec{x}}_i
     - {\overline{\vec{w}}_j^t}^\top \vec{g}_j^t \;
       {\overline{\vec{w}}_j^t}^\top \, \overline{\vec{x}}_i
& & \text{by \autoref{cor:prelim} and since $I_0(\vec{w}_j^t) = \emptyset$} \\
& >    {\vec{\gamma}_{[n]}\!}^\top \, \overline{\vec{x}}_i
     - \|\vec{\gamma}_{[n]}\| \;
       {\overline{\vec{w}}_j^t}^\top \, \overline{\vec{x}}_i
     - 2 \lambda^{2 - 3 \varepsilon}
& & \text{since $\|\vec{\gamma}_{[n]} - \vec{g}_j^t\| < \lambda^{2 - 3 \varepsilon}$} \\
& \geq \delta^3
     - {\Delta\!}^2 \;
       {\overline{\vec{w}}_j^t}^\top \, \overline{\vec{x}}_i
     - 2 \lambda^{2 - 3 \varepsilon}
& & \text{by \autoref{pr:d.D}~\ref{pr:d.D.gamma}} \\
& >    \frac{\delta^3}{2}
     - {\Delta\!}^2 \;
       {\overline{\vec{w}}_j^t}^\top \, \overline{\vec{x}}_i
& & \text{since $\lambda^{2 - 3 \varepsilon}
            \leq n^{-9 \cdot 3 \cdot 5 n / \delta^3}
               < \tfrac{\delta^3}{4}$,}
\end{align*}
so the continuous curve ${\overline{\vec{w}}_j^t}^\top \, \overline{\vec{x}}_i$ is positive to the right of~$t_j^{n_j}$, negative to the left of~$\mathsf{T}$, and does not cross zero in between.

Hence $I_0(\vec{w}_j^t) = \emptyset$ for all $t \in (t_j^{n_j}, T_1]$, which together with the inequality $\frac{1}{n} \sum_{i = 1}^n |h_{\vec{\theta}^t}(\vec{x}_i)| \|\vec{x}_i\| < \lambda^{2 - 3 \varepsilon}$ establishes part~\ref{l:1.gamma.g}.

By \autoref{l:0}~\ref{l:0.omega-w}, for all $t \in [t_j^{n_j}, \max\{t_j^{n_j}, \tau_j^{n_j}\}]$ we have $1 - {\overline{\vec{\alpha}}_j^t}^\top \, \overline{\vec{w}}_j^t \leq \frac{1}{2} \lambda^{2 - 4 \varepsilon}$.

Moreover, for all $t \in (\max\{t_j^{n_j}, \tau_j^{n_j}\}, T_1)$ we have
\begin{align*}
& \mathrm{d} (1 - {\overline{\vec{\alpha}}_j^t}^\top \, \overline{\vec{w}}_j^t) /
  \mathrm{d} t
& & \\
& =  - {\overline{\vec{\alpha}}_j^t}^\top \vec{g}_j^t
     - {\overline{\vec{w}}_j^t}^\top \vec{\gamma}_{[n]}
     + {\overline{\vec{\alpha}}_j^t}^\top \, \overline{\vec{w}}_j^t \,
       ({\overline{\vec{\alpha}}_j^t}^\top \vec{\gamma}_{[n]}
      + {\overline{\vec{w}}_j^t}^\top \vec{g}_j^t)
& & \text{by \autoref{cor:prelim} and \autoref{pr:0+}} \\
& <    2 \lambda^{2 - 3 \varepsilon}
     - (1 - {\overline{\vec{\alpha}}_j^t}^\top \, \overline{\vec{w}}_j^t) \,
       {(\overline{\vec{\alpha}}_j^t + \overline{\vec{w}}_j^t)\!}^\top
       \vec{\gamma}_{[n]}
& & \text{since $\|\vec{\gamma}_{[n]} - \vec{g}_j^t\| < \lambda^{2 - 3 \varepsilon}$} \\
& \leq 2 \lambda^{2 - 3 \varepsilon}
     - (1 - {\overline{\vec{\alpha}}_j^t}^\top \, \overline{\vec{w}}_j^t) \,
       \tfrac{\delta^{7/2}}{\sqrt{2} n}
& & \begin{aligned}[t]
    & \text{by \autoref{pr:omega}~\ref{pr:omega.cos},
                                  \ref{pr:omega.inter}, and~\ref{pr:omega.finalp},} \\
    & \text{and \autoref{pr:d.D}~\ref{pr:d.D.gamma}}
    \end{aligned} \\
& \leq 4 \lambda^\varepsilon
       \!\left(\!\tfrac{1}{2} \lambda^{2 - 4 \varepsilon}
               - (1 - {\overline{\vec{\alpha}}_j^t}^\top \, \overline{\vec{w}}_j^t)
       \!\right)\!
& & \text{since $\lambda^{-\varepsilon}
            \geq n^{9 \cdot 3 n / \delta^3}
               > \tfrac{4 \sqrt{2} n}{\delta^{7 / 2}}$.}
\end{align*}

Hence for all $t \in (t_j^{n_j}, T_1]$ we have $1 - {\overline{\vec{\alpha}}_j^t}^\top \, \overline{\vec{w}}_j^t \leq \frac{1}{2} \lambda^{2 - 4 \varepsilon}$ and thus $\|\overline{\vec{\alpha}}_j^t - \overline{\vec{w}}_j^t\| \leq \lambda^{1 - 2 \varepsilon}$, establishing part~\ref{l:1.omega.w}.

By \autoref{pr:T0.T1} and \autoref{pr:0+} we have
\[{\overline{\vec{\alpha}}_j^{T_1}\!}^\top \, \overline{\vec{\gamma}}_{[n]}
> \tanh\!\left(\tfrac{2}{3} \, \|\vec{\gamma}_{[n]}\| \, T_1\right)\!
= \tanh\!\left(\tfrac{2 \varepsilon}{3} \ln\!\left(\tfrac{1}{\lambda}\right)\!\right)\!
> 1 - 2 \lambda^{\frac{4 \varepsilon}{3}} \;,\]
so
${\overline{\vec{w}}_j^{T_1}\!}^\top \, \overline{\vec{\gamma}}_{[n]}
 > 1 - 2 \lambda^{\frac{4 \varepsilon}{3}} - \lambda^{1 - 2 \varepsilon}
 = 1 - \lambda^\varepsilon \!\left(2 \lambda^{\frac{\varepsilon}{3}}
                                 + \lambda^{1 - 3 \varepsilon}\right)\!
 > 1 - \lambda^\varepsilon$,
establishing part~\ref{l:1.w.gamma}.

Recalling \autoref{l:0}, for all $t \in [0, T_1]$ we have
\begin{align*}
\left\|\vec{\gamma}_{I_\ppp(\vec{\alpha}_j^t)} - \vec{g}_j^t\right\|
& \leq
\begin{cases}
\tfrac{{\Delta\!}^2}{n} + \lambda^{2 - 3 \varepsilon} <
\tfrac{2 {\Delta\!}^2}{n}
& \text{if $\exists \ell \in [n_j] \colon
            |t - \tau_j^\ell| \leq \lambda^{1 - 2 \varepsilon}$,} \\[.33ex]
\lambda^{2 - 3 \varepsilon}
& \text{otherwise}
\end{cases}
\\
\|\overline{\vec{\alpha}}_j^t - \overline{\vec{w}}_j^t\|
& \leq \lambda^{1 - 2 \varepsilon} \;,
\end{align*}
so for almost all $t \in [0, T_1]$ we have
\begin{align*}
& \left|\frac{\mathrm{d} \ln \|\vec{\alpha}_j^t\|}{\mathrm{d} t}
      - \frac{\mathrm{d} \ln \|\vec{w}_j^t / \lambda\|}{\mathrm{d} t}\right|
& & \\
& =    \left|{\overline{\vec{\alpha}}_j^t}^\top \vec{\gamma}_{I_\ppp(\vec{\alpha}_j^t)}
           - {\overline{\vec{w}}_j^t}^\top \vec{g}_j^t\right|
& & \text{by \autoref{cor:prelim}} \\
& \leq \left|{(\overline{\vec{\alpha}}_j^t - \overline{\vec{w}}_j^t)\!}^\top
             \vec{\gamma}_{I_\ppp(\vec{\alpha}_j^t)}\right|
     + \left|{\overline{\vec{w}}_j^t}^\top \!
             \left(\vec{\gamma}_{I_\ppp(\vec{\alpha}_j^t)} - \vec{g}_j^t\right)\right|
& & \text{by properties of absolute value} \\
& \leq \lambda^{1 - 2 \varepsilon} {\Delta\!}^2 +
       \begin{cases}
       \tfrac{2 {\Delta\!}^2}{n}
       & \text{if $\exists \ell \in [n_j] \colon
                   |t - \tau_j^\ell| \leq \lambda^{1 - 2 \varepsilon}$,} \\[.33ex]
       \lambda^{2 - 3 \varepsilon}
       & \text{otherwise}
       \end{cases}
& & \text{by \autoref{pr:d.D}~\ref{pr:d.D.gamma}} \\
& <    \begin{cases}
       \tfrac{3 {\Delta\!}^2}{n}
       & \text{if $\exists \ell \in [n_j] \colon
                   |t - \tau_j^\ell| \leq \lambda^{1 - 2 \varepsilon}$,} \\[.33ex]
       \lambda^{1 - \frac{5 \varepsilon}{2}}
       & \text{otherwise}
       \end{cases}
& & \text{since $\lambda^{-\frac{\varepsilon}{2}}
            \geq n^{\frac{9 \cdot 3}{2} n {\Delta\!}^2}
               > 2 {\Delta\!}^2$,}
\end{align*}
and therefore
\begin{align*}
& |\ln \|\vec{\alpha}_j^{T_1}\| - \ln \|\vec{w}_j^{T_1} \! / \lambda\||
& & \\
& \leq 2 n \lambda^{1 - 2 \varepsilon} \tfrac{3 {\Delta\!}^2}{n}
     + (T_1 - 2 n \lambda^{1 - 2 \varepsilon}) \lambda^{1 - \frac{5 \varepsilon}{2}}
& & \text{by the previous inequality} \\
& <    6 {\Delta\!}^2 \lambda^{1 - 2 \varepsilon}
     + T_1 \lambda^{1 - \frac{5 \varepsilon}{2}}
& & \text{omitting the negative term} \\
& <    \tfrac{1}{2} \lambda^{1 - 3 \varepsilon}
     + \tfrac{1}{\delta^2}
       \lambda^{1 - \frac{5 \varepsilon}{2}}
       \ln(1 / \lambda^\varepsilon)
& & \text{since $\lambda^{-\varepsilon} > 12 {\Delta\!}^2$ and
          by \autoref{pr:d.D}~\ref{pr:d.D.gamma}} \\
& <    \tfrac{1}{2} \lambda^{1 - 3 \varepsilon}
     + \tfrac{3}{\delta^2}
       \lambda^{1 - \frac{8 \varepsilon}{3}}
& & \text{since $3 \lambda^{-\frac{\varepsilon}{6}}
               = 3 \sqrt[6]{1 / \lambda^\varepsilon}
               > \ln(1 / \lambda^\varepsilon)$} \\
& <    \lambda^{1 - 3 \varepsilon}
& & \text{since $\lambda^{-\frac{\varepsilon}{3}}
            \geq n^{9 n / \delta^3}
               > \tfrac{6}{\delta^2}$,}
\end{align*}
establishing part~\ref{l:1.lengths}.
\end{proof}

\section{Proofs for the second phase}

Here we prove a number of results which culminate in \autoref{l:2} below, whose parts~\ref{l:2.T2} and~\ref{l:2.L} establish \autoref{th:2}.

We begin by observing that the eigenvector of the largest eigenvalue of the matrix $\frac{1}{n} \vec{X} \vec{X}^\top$ is in the interior of the cone spanned by the training points.

\begin{proposition}
\label{pr:u1}
$\vec{u}_1 \in \intr(\cone\{\vec{x}_1, \ldots, \vec{x}_n\})$.
\end{proposition}

\begin{proof}
If $\vec{v} \in \cone\{\vec{x}_1, \ldots, \vec{x}_n\} \setminus \{\vec{0}\}$, i.e.\ $\vec{v} = \sum_{i = 1}^n \beta_i \vec{x}_i$ for some $\beta_1, \ldots, \beta_n \geq 0$ that are not all zero, then
\[\frac{1}{n} \vec{X} \vec{X}^\top \vec{v}
=   \frac{1}{n}
    \sum_{i' = 1}^n
    \!\left(\sum_{i = 1}^n \beta_i \vec{x}_{i'}^\top \vec{x}_i\!\right)\!
    \vec{x}_{i'}
\in \intr(\cone\{\vec{x}_1, \ldots, \vec{x}_n\}) \;.\]
Thus $\frac{1}{n} \vec{X} \vec{X}^\top$ maps $\cone\{\vec{x}_1, \ldots, \vec{x}_n\} \setminus \{\vec{0}\}$ into $\intr(\cone\{\vec{x}_1, \ldots, \vec{x}_n\})$.

Let $\vec{v}_0 \coloneqq \vec{v}^*$, and $\vec{v}_{\ell + 1} \coloneqq \frac{1}{n} \vec{X} \vec{X}^\top \vec{v}_\ell$ for all $\ell \in \mathbb{N}$.

Recalling \autoref{pr:gamman}, we have $\vec{v}_1 = \gamma_{[n]} \in \intr(\cone\{\vec{x}_1, \ldots, \vec{x}_n\}) \subseteq \cone\{\vec{x}_1, \ldots, \vec{x}_n\} \setminus \{\vec{0}\}$, and so $\vec{v}_\ell \in \intr(\cone\{\vec{x}_1, \ldots, \vec{x}_n\})$ for all $\ell \geq 1$.

Since $\vec{v}_\ell = \sum_{k = 1}^d \eta_k^\ell \nu_k \vec{u}_k$ and $\eta_1$~is strictly the largest eigenvalue, we have that $\angle(\vec{u}_1, \vec{v}_\ell) \to 0$ as $\ell \to \infty$.  Therefore $\vec{u}_1 \in \cone\{\vec{x}_1, \ldots, \vec{x}_n\} \setminus \{\vec{0}\}$, but since $\frac{1}{n} \vec{X} \vec{X}^\top \vec{u}_1 = \eta_1 \vec{u}_1$, in fact $\vec{u}_1 \in \intr(\cone\{\vec{x}_1, \ldots, \vec{x}_n\})$.
\end{proof}

Our next observation is that the key set~$\mathcal{S}$ defined in \autoref{s:second} is strictly contained in the ball with centre $\vec{v}^* \! / 2$ and radius $\|\vec{v}^*\| / 2 = 1 / 2$, i.e.~that passes through the origin and the teacher neuron, and is centred half-way between them.

\begin{proposition}
\label{pr:ball}
For all $\vec{v} \in \mathcal{S}$ we have $\vec{v}^\top (\vec{v}^* - \vec{v}) > 0$.
\end{proposition}

\begin{proof}
Suppose $\vec{v} = \sum_{k = 1}^d \nu_k \vec{u}_k \in \mathcal{S}_\ell$ for some $\ell \in [d]$.  Then
\[\vec{v}^\top (\vec{v}^* - \vec{v})
  = \sum_{k = 1}^d \nu_k (\nu^*_k - \nu_k)
  > \sum_{k = 1}^d \frac{\eta_k}{\eta_\ell} \nu_k (\nu^*_k - \nu_k)
  = \frac{1}{\eta_\ell}
    \vec{v}^\top \frac{1}{n} \vec{X} \vec{X}^\top (\vec{v}^* - \vec{v})
  > 0 \;.
  \qedhere\]
\end{proof}

The angles between a vector~$\vec{v}$ in~$\mathcal{S}$ and the vector obtained by applying the operator $\frac{1}{n} \vec{X} \vec{X}^\top$ to the vector $\vec{v}^* - \vec{v}$ will be important in what follows.  We now show that, if $\vec{v}$~is in the subset~$\mathcal{S}_1$ (also defined in \autoref{s:second}), then the cosine of that angle has a positive lower bound that does not depend on the initialisation scale~$\lambda$.

\begin{proposition}
\label{pr:S1}
For all $\vec{v} = \sum_{k = 1}^d \nu_k \vec{u}_k \in \mathcal{S}_1$ we have
\[\overline{\vec{v}}^\top \,
  \overline{\vec{X} \vec{X}^\top (\vec{v}^* - \vec{v})} >
  \frac{1}{2}
  {\left(\frac{\eta_d \nu^*_d}{\|\vec{\gamma}_{[n]}\|}\right)\!}^2 \;.\]
\end{proposition}

\begin{proof}
Observe that
\begin{align*}
\overline{\vec{v}}^\top \,
\overline{\vec{X} \vec{X}^\top (\vec{v}^* - \vec{v})}
& = {\overline{\tfrac{\nu^*_1}{\nu_1} \vec{v}}}^\top \,
    \overline{\tfrac{1}{n} \vec{X} \vec{X}^\top
              \tfrac{\nu^*_d}{\nu^*_d - \nu_d} (\vec{v}^* - \vec{v})} \\
& > \frac{\eta_1 \sum_{k = 1}^d \eta_k {\nu^*_k}^2
                                \frac{\nu_k}{\nu^*_k}
                                \!\left(\!1 - \frac{\nu_k}{\nu^*_k}\!\right)\!}
         {\frac{\nu_1}{\nu^*_1}
          \!\left(\!1 - \frac{\nu_d}{\nu^*_d}\!\right)\!
          \|\vec{\gamma}_{[n]}\|^2} \\
& > \eta_1 \eta_d
    \frac{\frac{\nu_d}{\nu^*_d}}{\frac{\nu_1}{\nu^*_1}}
    \frac{{\nu^*_d}^2}{\|\vec{\gamma}_{[n]}\|^2} \\
& > \frac{1}{2}
    {\left(\frac{\eta_d \nu^*_d}{\|\vec{\gamma}_{[n]}\|}\right)\!}^2 \;.
\qedhere
\end{align*}
\end{proof}

Our final preparatory result is a positive lower bound, however depending on~$\lambda$, on the cosine of every angle between a vector~$\vec{v}$ in~$\mathcal{S}$ and a training point.  The proof relies on the correlation property of our datasets, i.e.~that the angles between the training points and the teacher neuron are less than $\pi / 4$.

\begin{proposition}
\label{pr:v.xi}
$\overline{\vec{v}}^\top
 \overline{\vec{x}}_i >
 \sqrt{8} \lambda^{\varepsilon / 2}$
for all $\vec{v} \in \mathcal{S}$ and all $i \in [n]$.
\end{proposition}

\begin{proof}
Suppose $\vec{v} = \sum_{k = 1}^d \nu_k \vec{u}_k \in \mathcal{S}$.

It suffices to establish that $\cos \angle(\vec{v}^*, \vec{v}) \geq \frac{1}{\sqrt{2}} + 2 \lambda^{\varepsilon / 2}$, because it implies that for all $i \in [n]$ we have
\begin{align*}
\cos \angle(\vec{v}, \vec{x}_i)
& \geq \cos \angle(\vec{v}^*, \vec{v}) \cos \angle (\vec{v}^*, \vec{x}_i)
     - \sin \angle(\vec{v}^*, \vec{v}) \sin \angle (\vec{v}^*, \vec{x}_i) \\
& >    \frac{1}{\sqrt{2}}
       \left(\frac{1}{\sqrt{2}} + 2 \lambda^{\varepsilon / 2}\right)
     - \frac{1}{\sqrt{2}}
       \sqrt{1 - {\left(\frac{1}{\sqrt{2}} + 2 \lambda^{\varepsilon / 2}\right)\!}^2} \\
& =    \frac{1}{2} + \sqrt{2} \lambda^{\varepsilon / 2}
     - \sqrt{\frac{1}{2} - {\left(\frac{1}{2}
                                + \sqrt{2} \lambda^{\varepsilon / 2}\right)\!}^2} \\
& >    \frac{1}{2} + \sqrt{2} \lambda^{\varepsilon / 2}
     - \sqrt{\frac{1}{4} - \sqrt{2} \lambda^{\varepsilon / 2}} \\
& >    \frac{1}{2} + \sqrt{2} \lambda^{\varepsilon / 2}
     - \left(\frac{1}{2} - \sqrt{2} \lambda^{\varepsilon / 2}\right) \\
& =    \sqrt{8} \lambda^{\varepsilon / 2} \;.
\end{align*}

By \autoref{pr:u1}, we have $\vec{u}_1^\top \vec{v}^* > 1 / \sqrt{2}$.

If $\vec{v} \in \mathcal{S}_1$ then
\[\left\|\vec{v}^* - \frac{\nu^*_1}{\nu_1} \vec{v}\right\|^2 \!
  =        \sum_{k = 2}^d 
           {\!\left(\!1 - \frac{\nu^*_1}{\nu_1}
                          \frac{\nu_k}{\nu^*_k}\!\right)\!}^2 {\nu^*_k}^2
  \,<\,    \sum_{k = 2}^d
           {\!\left(\!1 - \frac{\eta_k}{2 \eta_1}\!\right)\!}^2 {\nu^*_k}^2
  \,\leq\, {\!\left(\!1 - \frac{\eta_d}{2 \eta_1}\!\right)\!}^2
           \|\vec{v}^* - \nu^*_1 \vec{u}_1\|^2 \;,\]
so we have
\begin{multline*}
\cos \angle(\vec{v}^*, \vec{v})
>    \sqrt{1 - \frac{1}{2} {\left(\!1 - \frac{\eta_d}{2 \eta_1}\!\right)\!}^2}
=    \sqrt{\frac{1}{2} + \frac{\eta_d}{2 \eta_1} - \frac{\eta_d^2}{8 \eta_1^2}}
>    \sqrt{\frac{1}{2} + \frac{3 \eta_d}{8 \eta_1}} \\
>    \frac{1}{\sqrt{2}} + (2 - \sqrt{2}) \frac{3 \eta_d}{8 \eta_1}
\geq \frac{1}{\sqrt{2}} + (2 - \sqrt{2}) \frac{3 \delta^2}{8 {\Delta\!}^2}
>    \frac{1}{\sqrt{2}} + 2 \lambda^{\varepsilon / 2} \;.
\end{multline*}

Otherwise $\vec{v} \in \mathcal{S}_\ell$ for some $\ell \neq 1$.  Then $(\vec{v} - \nu^*_1 \vec{u}_1)^\top (\vec{v}^* - \vec{v}) > \vec{v}^\top (\vec{v}^* - \vec{v}) > 0$ by \autoref{pr:ball}.  Also $(\vec{v} - \nu^*_1 \vec{u}_1)^\top (\vec{v}^* - \nu^*_1 \vec{u}_1) = \sum_{k = 2}^d \nu_k \nu^*_k > \sum_{k = 2}^d \frac{\eta_k}{2 \eta_1} {\nu^*_k}^2 \geq \frac{\eta_d}{2 \eta_1} \|\vec{v}^* - \nu^*_1 \vec{u}_1\|^2$.  Hence
\[\|\vec{v}^* - \vec{v}\|^2
  \leq \|\vec{v}^* - \nu^*_1 \vec{u}_1\|^2 -
       \|\vec{v} - \nu^*_1 \vec{u}_1\|^2
  <    \!\left(\!1 - {\!\left(\frac{\eta_d}{2 \eta_1}\right)\!}^2\right)\!
       \|\vec{v}^* - \nu^*_1 \vec{u}_1\|^2 \;,\]
so we have
\[\begin{multlined}[b][33em]
\cos \angle(\vec{v}^*, \vec{v})
>    \sqrt{\frac{1}{2} + \frac{1}{2} {\left(\frac{\eta_d}{2 \eta_1}\right)\!}^2}
>    \frac{1}{\sqrt{2}}
   + \frac{2 - \sqrt{2}}{2} {\left(\frac{\eta_d}{2 \eta_1}\right)\!}^2 \\
\geq \frac{1}{\sqrt{2}}
   + \frac{2 - \sqrt{2}}{2} {\left(\frac{\delta^2}{2 {\Delta\!}^2}\!\right)\!}^2
>    \frac{1}{\sqrt{2}}
   + 2 \lambda^{\varepsilon / 2} \;.
\end{multlined} \qedhere\]
\end{proof}

For all $t \geq T_1$, recall from \autoref{s:second} that
$\vec{v}^t =
 \sum_{j \in J_\ppp} a_j^t \vec{w}_j^t$,
and let
\begin{align*}
\vec{g}^t & \coloneqq
\frac{1}{n} \vec{X} \vec{X}^\top (\vec{v}^* - \vec{v}^t) &
\vec{f}^t & \coloneqq
\|\vec{v}^t\| (\vec{g}^t + \overline{\vec{v}}^t \,
                           {\overline{\vec{v}}^t}^\top \vec{g}^t) \;.
\end{align*}

The next lemma is at the heart of our analysis of the training dynamics.  It establishes several key facts that hold at all times~$t$ from the start~$T_1$ of the second phase, and which form the statement of the lemma as follows.

\subparagraph{Parts~\ref{l:S.w} and~\ref{l:S.S}.}

The cosines of all angles between hidden neurons that form the aligned bundle remain above $1 - 4 \lambda^\varepsilon$, and the bundle vector~$\vec{v}^t$ which was defined as the sum of the constituent hidden neurons multiplied by their last-layer weights stays in the set~$\mathcal{S}$.  These two properties support each other, e.g.~we show that the containment in~$\mathcal{S}$ implies that the gradients of the individual hidden neurons are such that the bundle keeps together rather than breaks apart.

\subparagraph{Parts~\ref{l:S.v} and~\ref{l:S.g}.}

The network acts linearly on the training points, namely its outputs for the training points equal their inner products with the bundle vector.  Moreover, the vectors~$\vec{g}_j^t$ (defined in \autoref{pr:prelim}~\ref{pr:prelim.gj}) that govern the dynamics are all equal to the vector~$\vec{g}^t$ which is obtained by applying the operator $\frac{1}{n} \vec{X} \vec{X}^\top$ to the vector $\vec{v}^* - \vec{v}^t$.

\subparagraph{Parts~\ref{l:S.diff} and~\ref{l:S.f}.}

The derivative of the bundle vector~$\vec{v}^t$ with respect to the time~$t$ exists, i.e.~the issue of the non-differentiability of the ReLU activation at~$0$ does not arise in this respect.  However, the two-layer dynamics is such that this derivative is in general only approximated by the vector~$\vec{f}^t$ defined above, and we bound that error by a ball centred at~$\vec{f}^t$ whose radius depends on the initialisation scale~$\lambda$.

\subparagraph{Parts~\ref{l:S.dvS1}, \ref{l:S.dv} and~\ref{l:S.dvs-v}.}

We show that the squared norm of~$\vec{v}^t$ grows exponentially fast as it moves away from the saddle at the origin, obtaining a lower bound on the speed of its increase that does not depend on~$\lambda$ as long as~$\vec{v}^t$ is in the subset~$\mathcal{S}_1$, and a lower bound that depends on~$\lambda$ subsequently.  Also we show an upper bound on the speed of decrease of the squared norm of $\vec{v}^* - \vec{v}$, i.e.~the square of the distance between the bundle vector and the teacher neuron.

Perhaps the most involved segment of the proof proceeds by showing that each face of the boundary of the set~$\mathcal{S}$ is repelling towards the interior of~$\mathcal{S}$ with respect to the dynamics of the bundle vector~$\vec{v}^t$, whose derivative is approximately~$\vec{f}^t$.  A major complication is that this is in general not true for the entire boundary of the ``padded ellipsoid'' constraint~$\Xi$, but holds for its remainder after the slicing off by the other constraints that define~$\mathcal{S}$.

\begin{lemma}
\label{l:S}
For all $t \geq T_1$ we have:
\begin{enumerate}[(i),itemsep=0ex,leftmargin=3em]
\item
\label{l:S.w}
$1 - {\overline{\vec{w}}_j^t}^\top \, \overline{\vec{w}}_{j'}^t < 4 \lambda^\varepsilon$ for all $j, j' \in J_\ppp$;
\item
\label{l:S.S}
$\vec{v}^t \in \mathcal{S}$;
\item
\label{l:S.v}
$h_{\vec{\theta}^t}(\vec{x}_i) = {\vec{v}^t}^\top \vec{x}_i$ for all $i \in [n]$;
\item
\label{l:S.g}
$\vec{g}_j^t = \vec{g}^t$ for all $j \in J_\ppp$;
\item
\label{l:S.diff}
$\vec{v}^t$ is differentiable at~$t$;
\item
\label{l:S.f}
$\|\mathrm{d} \vec{v}^t / \mathrm{d} t - \vec{f}^t\| \leq
 3 \lambda^{\varepsilon / 2} \|\vec{f}^t\|$;
\item
\label{l:S.dvS1}
$\mathrm{d} \|\vec{v}^t\|^2 / \mathrm{d} t \geq
 (\eta_d \nu^*_d / \|\vec{\gamma}_{[n]}\|)^2
 \|\vec{v}^t\|^2 \|\vec{g}^t\|$
if $\vec{v}^t \in \mathcal{S}_1$;
\item
\label{l:S.dv}
$\mathrm{d} \|\vec{v}^t\|^2 / \mathrm{d} t \geq
 3 \lambda^{\varepsilon / 3}
 \|\vec{v}^t\|^2 \|\vec{g}^t\|$;
\item
\label{l:S.dvs-v}
$\mathrm{d} \|\vec{v}^* - \vec{v}^t\|^2 / \mathrm{d} t \geq
 -5 \eta_1 \|\vec{v}^t\| \|\vec{v}^* - \vec{v}^t\|^2$.
\end{enumerate}
\end{lemma}

\begin{proof}
First we establish the following.

\begin{claim}
\label{cl:S.imply}
For all $t \geq T_1$, assertions~\ref{l:S.w}--\ref{l:S.S} imply assertions~\ref{l:S.v}--\ref{l:S.dvs-v}.
\end{claim}

\begin{proof}[Proof of claim]
Suppose $t \geq T_1$, and~\ref{l:S.w} and~\ref{l:S.S} are true.

By \autoref{l:0} and \autoref{pr:v.xi}, we have~\ref{l:S.v}, \ref{l:S.g}, and~\ref{l:S.diff}.

For~\ref{l:S.f}, we have
\begin{align*}
\left\|\frac{\mathrm{d} \vec{v}^t}{\mathrm{d} t} - \vec{f}^t\right\|
& = \left\|
    \sum_{j \in J_\ppp}\!
    \frac{\mathrm{d}}{\mathrm{d} t}
    (\|\vec{w}_j^t\|^2 \, \overline{\vec{w}}_j^t) -
    \|\vec{v}^t\|
    (\vec{g}^t + \overline{\vec{v}}^t \,
                 {\overline{\vec{v}}^t}^\top   \vec{g}^t)
    \right\| \\
& = \left\|
    \sum_{j \in J_\ppp}\!
    \|\vec{w}_j^t\|^2
    (\vec{g}^t + \overline{\vec{w}}_j^t \,
                 {\overline{\vec{w}}_j^t}^\top \vec{g}^t) -
    \|\vec{v}^t\|
    (\vec{g}^t + \overline{\vec{v}}^t \,
                 {\overline{\vec{v}}^t}^\top   \vec{g}^t)
    \right\| \\
& = \left\|
    \sum_{j \in J_\ppp}\!
    \Bigl(
    \|\vec{w}_j^t\|^2 \, \vec{g}^t +
    \overline{\vec{w}}_j^t
    \|\vec{w}_j^t\|^2 \, {\overline{\vec{w}}_j^t}^\top \vec{g}^t -
    {\overline{\vec{v}}^t}^\top \overline{\vec{w}}_j^t
    \|\vec{w}_j^t\|^2 \, \vec{g}^t -
    \overline{\vec{v}}^t
    \|\vec{w}_j^t\|^2 \, {\overline{\vec{w}}_j^t}^\top \vec{g}^t
    \Bigr)
    \right\| \\
& = \left\|
    \sum_{j \in J_\ppp}\!
    (1 - {\overline{\vec{v}}^t}^\top \overline{\vec{w}}_j^t)
    \|\vec{w}_j^t\|^2 \, \vec{g}^t +
    \!\sum_{j \in J_\ppp}\!
    (\overline{\vec{w}}_j^t - \overline{\vec{v}}^t)
    \|\vec{w}_j^t\|^2 \, {\overline{\vec{w}}_j^t}^\top \vec{g}^t
    \right\| \\
& < 4 \lambda^\varepsilon
    \!\sum_{j \in J_\ppp}\!
    \|\vec{w}_j^t\|^2 \|\vec{g}^t\| +
    \sqrt{8} \lambda^{\varepsilon / 2}
    \!\sum_{j \in J_\ppp}\!
    \|\vec{w}_j^t\|^2 \, {\overline{\vec{w}}_j^t}^\top \vec{g}^t \\
& < \frac{4 \lambda^\varepsilon}{1 - 4 \lambda^\varepsilon}
    \|\vec{v}^t\| \|\vec{g}^t\| +
    \sqrt{8} \lambda^{\varepsilon / 2} \,
    {\vec{v}^t}^\top \vec{g}^t \\
& < (5 \lambda^\varepsilon + \sqrt{8} \lambda^{\varepsilon / 2})
    \|\vec{f}^t\| \\
& < 3 \lambda^{\varepsilon / 2}
    \|\vec{f}^t\| \;.
\end{align*}

Now
\begin{align*}
\mathrm{d} \|\vec{v}^t\|^2 / \mathrm{d} t
& \geq 2 ({\vec{v}^t}^\top \vec{f}^t -
          3 \lambda^{\varepsilon / 2} \|\vec{v}^t\| \|\vec{f}^t\|) \\
& =    2 \|\vec{v}^t\|
       (2 {\vec{v}^t}^\top \vec{g}^t -
        3 \lambda^{\varepsilon / 2} \|\vec{f}^t\|) \;,
\end{align*}
so for~\ref{l:S.dvS1}, if $\vec{v}^t \in \mathcal{S}_1$ then by \autoref{pr:S1} we have
\begin{align*}
2 \|\vec{v}^t\|
(2 {\vec{v}^t}^\top \vec{g}^t -
 3 \lambda^{\varepsilon / 2} \|\vec{f}^t\|)
& > (2 (\eta_d \nu^*_d / \|\vec{\gamma}_{[n]}\|)^2 - 12 \lambda^{\varepsilon / 2})
    \|\vec{v}^t\|^2 \|\vec{g}^t\| \\
& > (\eta_d \nu^*_d / \|\vec{\gamma}_{[n]}\|)^2
    \|\vec{v}^t\|^2 \|\vec{g}^t\|
\end{align*}
since \[\lambda^{\varepsilon / 2}
   \leq n^{-\frac{9 \cdot 3 n {\Delta\!}^2}{2 \delta^3}}
      < {\left(\!\frac{4 \delta^3}{9 \cdot 3 n {\Delta\!}^2}\!\right)\!}^2\!
      < \frac{\delta^6}{12 d {\Delta\!}^4}
   \leq (\eta_d \nu^*_d / \|\vec{\gamma}_{[n]}\|)^2 / 12 \;,\]
and for~\ref{l:S.dv}, in general we have
\begin{align*}
2 \|\vec{v}^t\|
(2 {\vec{v}^t}^\top \vec{g}^t -
 3 \lambda^{\varepsilon / 2} \|\vec{f}^t\|)
& > (4 \lambda^{\varepsilon / 3} - 12 \lambda^{\varepsilon / 2})
    \|\vec{v}^t\|^2 \|\vec{g}^t\| \\
& > 3 \lambda^{\varepsilon / 3}
    \|\vec{v}^t\|^2 \|\vec{g}^t\| \;.
\end{align*}

For~\ref{l:S.dvs-v}, we have
\begin{align*}
\mathrm{d} \|\vec{v}^* - \vec{v}^t\|^2 / \mathrm{d} t
& \geq -2 ((\vec{v}^* - \vec{v}^t)^\top \vec{f}^t +
           3 \lambda^{\varepsilon / 2} \|\vec{v}^* - \vec{v}^t\| \|\vec{f}^t\|) \\
& \geq -4 (1 + 3 \lambda^{\varepsilon / 2})
          \|\vec{v}^t\| \|\vec{g}^t\| \|\vec{v}^* - \vec{v}^t\| \\
& >    -5 \|\vec{v}^t\| \|\vec{g}^t\| \|\vec{v}^* - \vec{v}^t\| \\
& >    -5 \eta_1 \|\vec{v}^t\| \|\vec{v}^* - \vec{v}^t\|
\end{align*}
since $\|\vec{g}^t\| = \|\frac{1}{n} \vec{X} \vec{X}^\top (\vec{v}^* - \vec{v}^t)\| \leq \eta_1 \|\vec{v}^* - \vec{v}^t\| < \eta_1 \|\vec{v}^*\| = \eta_1$ by \autoref{pr:ball}.
\end{proof}

Second we show the following.

\begin{claim}
Assertions~\ref{l:S.w} and~\ref{l:S.S} are true for $t = T_1$.
\end{claim}

\begin{proof}[Proof of claim]
By \autoref{l:1}~\ref{l:1.w.gamma}, for all $j, j' \in J_\ppp$ we have
\begin{align*}
1 - {\overline{\vec{w}}_j^{T_1}\!}^\top \, \overline{\vec{w}}_{j'}^{T_1}
& =    \|\overline{\vec{w}}_j^{T_1}    - \overline{\vec{w}}_{j'}^{T_1}\|^2 / 2 \\
& <    \|\overline{\vec{w}}_j^{T_1}    - \overline{\vec{\gamma}}_{[n]}\|^2 +
       \|\overline{\vec{w}}_{j'}^{T_1} - \overline{\vec{\gamma}}_{[n]}\|^2 \\
& \leq 4 \lambda^\varepsilon
\end{align*}
where the first inequality is strict unless $\overline{\vec{w}}_j^{T_1} = \overline{\vec{\gamma}}_{[n]} = \overline{\vec{w}}_{j'}^{T_1}$, but in that case $1 - {\overline{\vec{w}}_j^{T_1}\!}^\top \, \overline{\vec{w}}_{j'}^{T_1} = 0$.

Writing $\vec{v}^{T_1} = \sum_{k = 1}^d \nu_k \vec{u}_k$, and recalling \autoref{pr:gamman}, \autoref{l:len.w}~\ref{l:len.w.upp}, and \autoref{l:1}~\ref{l:1.gamma.g} and~\ref{l:1.w.gamma}, we obtain that $\vec{v}^{T_1} \in \mathcal{S}_1$ because:
\begin{description}[itemsep=.25ex,style=unboxed,wide]
\item[$\Phi_1$:]
we have
\[\frac{\nu_1}{\|\vec{v}^{T_1}\|}
\geq \frac{\eta_1 \nu^*_1}{\|\vec{\gamma}_{[n]}\|}
   - \sqrt{2} \lambda^{\varepsilon / 2}
\geq \frac{4 \delta^3}{\sqrt{d} {\Delta\!}^2}
   - \sqrt{2} \lambda^{\varepsilon / 2}
> 0\]
and
\[\frac{\nu_1}{\nu^*_1}
\leq \frac{4 m \sqrt{d} {\Delta\!}^2}{\delta} \lambda^{2 - 2 \varepsilon}
<    \frac{1}{2} \;;\]
\item[$\Psi_{k, k'}^\downarrow$ for all $1 \leq k < k' \leq d$:]
we have
\begin{align*}
\frac{\nu_k}{\|\vec{v}^{T_1}\| \eta_k \nu^*_k}
& \leq \frac{1}{\|\vec{\gamma}_{[n]}\|}
     + \frac{\sqrt{2} \lambda^{\varepsilon / 2}}{\eta_k \nu^*_k} \\
& \leq \frac{1}{\|\vec{\gamma}_{[n]}\|}
     + \frac{\sqrt{2 d} \lambda^{\varepsilon / 2}}{\delta^3} \\
& <    \frac{2}{\|\vec{\gamma}_{[n]}\|}
     - \frac{2 \sqrt{2 d} \lambda^{\varepsilon / 2}}{\delta^3} \\
& \leq \frac{2}{\|\vec{\gamma}_{[n]}\|}
     - \frac{2 \sqrt{2} \lambda^{\varepsilon / 2}}{\eta_{k'} \nu^*_{k'}} \\
& \leq \frac{2 \nu_{k'}}{\|\vec{v}^{T_1}\| \eta_{k'} \nu^*_{k'}} \;;
\end{align*}
\item[$\Psi_{k, k'}^\uparrow$ for all $1 \leq k < k' \leq d$:]
by Bernoulli's inequality we have
\begin{align*}
{\!\left(\!1 - \frac{\nu_k}{\nu^*_k}\!\right)\!}
^{\frac{\eta_k + \eta_{k'}}{2 \eta_k}} \!\!
& <    1 - \frac{\eta_k + \eta_{k'}}{2 \eta_k}
           \frac{\nu_k}{\nu^*_k} \\
& \leq 1 - \|\vec{v}^{T_1}\|
           \frac{\eta_k + \eta_{k'}}{2}
           \!\left(\!\frac{1}{\|\vec{\gamma}_{[n]}\|}
                   - \frac{\sqrt{2} \lambda^{\varepsilon / 2}}
                          {\eta_k \nu^*_k}\!\right)\! \\
& =    1 - \|\vec{v}^{T_1}\|
           \!\left(\!
           \frac{\eta_{k'}}{\|\vec{\gamma}_{[n]}\|}
         + \frac{\eta_k - \eta_{k'}}{2 \|\vec{\gamma}_{[n]}\|}
         - \frac{\eta_k + \eta_{k'}}{2}
           \frac{\sqrt{2} \lambda^{\varepsilon / 2}}
                {\eta_k \nu^*_k}
           \!\right)\! \\
& \leq 1 - \|\vec{v}^{T_1}\|
           \!\left(\!
           \frac{\eta_{k'}}{\|\vec{\gamma}_{[n]}\|}
         + \frac{\delta^2}{d {\Delta\!}^2}
         - {\Delta\!}^2
           \frac{\sqrt{2 d} \lambda^{\varepsilon / 2}}
                {\delta^3}
           \!\right)\! \\
& <    1 - \|\vec{v}^{T_1}\|
           \!\left(\!
           \frac{\eta_{k'}}{\|\vec{\gamma}_{[n]}\|}
         + {\Delta\!}^2
           \frac{\sqrt{2 d} \lambda^{\varepsilon / 2}}
                {\delta^3}
           \!\right)\! \\
& \leq 1 - \|\vec{v}^{T_1}\| \eta_{k'}
           \!\left(\!\frac{1}{\|\vec{\gamma}_{[n]}\|}
                   + \frac{\sqrt{2} \lambda^{\varepsilon / 2}}
                          {\eta_{k'} \nu^*_{k'}}\!\right)\! \\
& \leq 1 - \frac{\nu_{k'}}{\nu^*_{k'}} \;;
\end{align*}
\item[$\Xi$:]
we have
\begin{align*}
{\overline{\vec{v}}^{T_1}\!}^\top \vec{g}^{T_1}
& \geq {\overline{\vec{v}}^{T_1}\!}^\top \vec{\gamma}_{[n]}
     - \lambda^{2 - 3 \varepsilon} \\
& \geq (1 - \lambda^\varepsilon)
       \|\vec{\gamma}_{[n]}\|
     - \lambda^{2 - 3 \varepsilon} \\
& >    3 \lambda^{\varepsilon / 3} \|\vec{\gamma}_{[n]}\|
     - \lambda^{2 - 3 \varepsilon} \\
& >    2 \lambda^{\varepsilon / 3} \|\vec{\gamma}_{[n]}\| \\
& >    \lambda^{\varepsilon / 3} (\|\vec{\gamma}_{[n]}\| + \lambda^{2 - 3 \varepsilon}) \\
& \geq \lambda^{\varepsilon / 3} \|\vec{g}^{T_1}\| \;.
\qedhere
\end{align*}
\end{description}
\end{proof}

Assume for a contradiction that there exists $t > T_1$ such that either~\ref{l:S.w} or~\ref{l:S.S} is false, and let $t$~be the smallest such.

For all $j, j' \in J_\ppp$ we have
\begin{align*}
\mathrm{d} (1 - {\overline{\vec{w}}_j^t}^\top \,
                \overline{\vec{w}}_{j'}^t) / \mathrm{d} t
& =    -{\overline{\vec{w}}_{j'}^t\!}^\top
        (\vec{g}^t - \overline{\vec{w}}_j^t \,
                     {\overline{\vec{w}}_j^t}^\top \vec{g}^t)
       -{\overline{\vec{w}}_j^t}^\top
        (\vec{g}^t - \overline{\vec{w}}_{j'}^t \,
                     {\overline{\vec{w}}_{j'}^t\!}^\top \vec{g}^t) \\
& =    -(1 - {\overline{\vec{w}}_j^t}^\top \, \overline{\vec{w}}_{j'}^t)
        {(\overline{\vec{w}}_j^t + \overline{\vec{w}}_{j'}^t)\!}^\top \vec{g}^t \\
& \leq -2 (\lambda^{\varepsilon / 3} - \sqrt{8} \lambda^{\varepsilon / 2}) \|\vec{g}^t\|
        (1 - {\overline{\vec{w}}_j^t}^\top \, \overline{\vec{w}}_{j'}^t) \\
& \leq -\lambda^{\varepsilon / 3} \|\vec{g}^t\|
        (1 - {\overline{\vec{w}}_j^t}^\top \, \overline{\vec{w}}_{j'}^t) \;.
\end{align*}

Therefore \ref{l:S.S}~is false.  Hence $\vec{v}^t$~is in at least one possibly curved face of the boundary of~$\mathcal{S}$, and is distinct from~$\vec{0}$ and~$\vec{v}^*$.  We consider those faces in the following cases, where we omit the superscripts~$t$, assume $\vec{v} = \sum_{k = 1}^d \nu_k \vec{u}_k \in \clos(\mathcal{S}_\ell)$ for some $\ell \in [d]$, write e.g.~$\widehat{\Omega}_k$ for the constraint obtained by replacing the unique strict inequality in~$\Omega_k$ by equality, and denote by~$\vec{p}$ a normal vector to the respective face that is at~$\vec{v}$ and on the side of the interior of~$\mathcal{S}$.  To get a contradiction, it suffices to show in each of the cases that $\overline{\vec{p}}^\top \overline{\vec{f}} > 3 \lambda^{\varepsilon / 2}$, because by \autoref{cl:S.imply} and continuity we have that \ref{l:S.f} is true at~$t$, and so $\overline{\vec{p}}^\top (\mathrm{d} \vec{v} / \mathrm{d} t) \geq \overline{\vec{p}}^\top \vec{f} - 3 \lambda^{\varepsilon / 2} \|\vec{f}\| > 0$.

\subparagraph{Case~$\widehat{\Omega}_k$ for some $1 \leq k < \ell$.}

Picking $\vec{p} \coloneqq \vec{u}_k$, we have
\[\overline{\vec{p}}^\top \overline{\vec{f}}
=    \nu^*_k \, \overline{\vec{v}}^\top \vec{g} / \|\vec{f}\|
>    \frac{\nu^*_k}{2} \,
     \overline{\vec{v}}^\top \overline{\vec{g}}
\geq \frac{\delta}{2 \sqrt{d}} \lambda^{\varepsilon / 3}
>    3 \lambda^{\varepsilon / 2}\]
since $\lambda^{-\varepsilon / 6}
  \geq n^{\frac{9}{2} n / \delta^3}
  \geq \frac{9}{\sqrt{2}} \e (\ln 2) \sqrt{d} / \delta^3
     > \frac{6 \sqrt{d}}{\delta}$.

\subparagraph{Case~$\widehat{\Phi}_\ell$.}

Necessarily $\ell \neq 1$.  Picking $\vec{p} \coloneqq \vec{u}_\ell$, we have
\begin{multline*}
\overline{\vec{p}}^\top \overline{\vec{f}}
\geq \!\left(\!
     \eta_\ell
     \!\left(\!1 - \frac{\eta_\ell}{2 \eta_{\ell - 1}}\!\right)\! \nu^*_\ell +
     \frac{1}{\|\vec{v}\|}
     \frac{\eta_\ell}{2 \eta_{\ell - 1}} \nu^*_\ell \,
     \overline{\vec{v}}^\top \vec{g}
     \!\right)\!
     \Big/ (2 \|\vec{g}\|) \\
>    \frac{\eta_\ell}{2 \eta_1}
     \!\left(\!1 - \frac{\eta_\ell}{2 \eta_{\ell - 1}}\!\right)\!
     \nu^*_\ell
>    \frac{\delta^3}{4 \sqrt{d} {\Delta\!}^2}
>    3 \lambda^{\varepsilon / 2} \;.
\end{multline*}

\subparagraph{Case~$\widehat{\Psi}_{k, k'}^\downarrow$ for some $\ell \leq k < k' \leq d$.}

Picking $\vec{p} \coloneqq - \eta_{k'} \nu^*_{k'} \vec{u}_k + 2 \eta_k \nu^*_k \vec{u}_{k'}$, we have
\begin{align*}
\overline{\vec{p}}^\top \overline{\vec{f}}
& =    \eta_k \eta_{k'} \nu^*_k \nu^*_{k'}
       \!\left(\!
       - \!\left(\!1 - \frac{\nu_k}{\nu^*_k}\!\right)\!
       + 2 \!\left(\!1 - \frac{\nu_{k'}}{\nu^*_{k'}}\!\right)\!
       \!\right)\!
       \|\vec{v}\| / (\|\vec{p}\| \|\vec{f}\|) \\
& =    \eta_k \eta_{k'} \nu^*_k \nu^*_{k'}
       \!\left(\!
       1 + \!\left(\!1 - \frac{\eta_{k'}}{\eta_k}\!\right)\!
           \frac{\nu_k}{\nu^*_k}
       \!\right)\!
       \|\vec{v}\| / (\|\vec{p}\| \|\vec{f}\|) \\
& >    \eta_k \eta_{k'} \nu^*_k \nu^*_{k'}
     / (2 \|\vec{p}\| \|\vec{g}\|) \\
& >    \frac{\eta_k \eta_{k'} \nu^*_k \nu^*_{k'}}
            {2 \eta_1 \sqrt{(2 \eta_k \nu^*_k)^2 +
                            (\eta_{k'} \nu^*_{k'})^2}} \\
& =    {\!\left(\!
       {\!\left(\frac{2 \eta_1}{\eta_k}
                \frac{1}{\nu^*_k}\right)\!}^2 \!\! +
       {\!\left(\frac{2 \eta_1}{\eta_{k'}}
                \frac{2}{\nu^*_{k'}}\right)\!}^2
       \right)\!\!}^{-1 / 2} \\
& \geq \frac{\delta^3}{5 \sqrt{2 d} {\Delta\!}^2} \\
& >    3 \lambda^{\varepsilon / 2} \;.
\end{align*}

\subparagraph{Case~$\widehat{\Psi}_{k, k'}^\uparrow$ for some $\ell \leq k < k' \leq d$.}

Picking
\[\vec{p} \coloneqq
  \frac{\eta_k + \eta_{k'}}{2} \nu^*_{k'} \vec{u}_k -
  \eta_k {\left(\!1 - \frac{\nu_k}{\nu^*_k}\!\right)\!}
         ^{\frac{1}{2} - \frac{\eta_{k'}}{2 \eta_k}} \nu^*_k \vec{u}_{k'} \;,\]
we have
\begin{align*}
\vec{p}^\top \vec{g}
& = \eta_k \eta_{k'} \nu^*_k \nu^*_{k'}
    \!\left(\!\!\left(\frac{1}{2} + \frac{\eta_k}{2 \eta_{k'}}\!\right)\!
              \!\left(\!1 - \frac{\nu_k}{\nu^*_k}\!\right)\! -
              {\!\left(\!1 - \frac{\nu_k}{\nu^*_k}\!\right)\!}
              ^{\frac{1}{2} - \frac{\eta_{k'}}{2 \eta_k}}
              \!\!\left(\!1 - \frac{\nu_{k'}}
                                   {\nu^*_{k'}}\!\right)\!\!\right)\! \\
& = \eta_k \eta_{k'} \nu^*_k \nu^*_{k'}
    \!\left(\!\!\left(\frac{1}{2} + \frac{\eta_k}{2 \eta_{k'}}\!\right)\!
              \!\left(\!1 - \frac{\nu_k}{\nu^*_k}\!\right)\! -
              \!\left(\!1 - \frac{\nu_k}
                                 {\nu^*_k}\!\right)\!\!\right)\! \\
& = \frac{\eta_k^2 - \eta_k \eta_{k'}}{2}
    \!\left(\!1 - \frac{\nu_k}{\nu^*_k}\!\right)\!
    \nu^*_k \nu^*_{k'}
\end{align*}
and
\begin{align*}
\vec{p}^\top \vec{v}
& = \frac{\eta_k + \eta_{k'}}{2} \nu_k \nu^*_{k'} -
    \eta_k {\left(\!1 - \frac{\nu_k}{\nu^*_k}\!\right)\!}
           ^{\frac{1}{2} - \frac{\eta_{k'}}{2 \eta_k}} \nu_{k'} \nu^*_k \\
& = \eta_k \nu^*_k \nu^*_{k'}
    \!\left(\!\!\left(\frac{1}{2} + \frac{\eta_{k'}}{2 \eta_k}\!\right)\!
              \frac{\nu_k}{\nu^*_k} -
              {\!\left(\!1 - \frac{\nu_k}{\nu^*_k}\!\right)\!}
              ^{\frac{1}{2} - \frac{\eta_{k'}}{2 \eta_k}}
              \!\!\left(\!
              1 - {\!\left(\!1 - \frac{\nu_k}{\nu^*_k}\!\right)\!}
                  ^{\frac{1}{2} + \frac{\eta_{k'}}{2 \eta_k}}
              \right)\!\!\right)\! \\
& = \eta_k \nu^*_k \nu^*_{k'}
    \!\left(\!\!\left(\!1 - \!\left(\frac{1}{2} -
                                    \frac{\eta_{k'}}{2 \eta_k}\!\right)\!
                            \frac{\nu_k}{\nu^*_k}\!\right)\! -
              {\!\left(\!1 - \frac{\nu_k}{\nu^*_k}\!\right)\!}
              ^{\frac{1}{2} - \frac{\eta_{k'}}{2 \eta_k}}\right)\! \\
& > \frac{\eta_k^2 - \eta_{k'}^2}{8 \eta_k}
    {\left(\!\frac{\nu_k}{\nu^*_k}\!\right)\!}^2 \nu^*_k \nu^*_{k'} \;.
\end{align*}

Hence if $\nu_k / \nu^*_k \leq 1 / 2$ then
\begin{align*}
\overline{\vec{p}}^\top \overline{\vec{f}}
& >    \frac{\eta_k^2 - \eta_k \eta_{k'}}{4}
       \nu^*_k \nu^*_{k'}
       / (2 \|\vec{p}\| \|\vec{g}\|) \\
& >    \frac{(\eta_k^2 - \eta_k \eta_{k'}) \nu^*_k \nu^*_{k'}}
            {\eta_1 \sqrt{(\eta_k + \eta_{k'})^2 {\nu^*_{k'}}^2 +
                          4 \eta_k^2 {\nu^*_k}^2}} \\
& \geq \frac{\delta^6}{\sqrt{2} d^2 {\Delta\!}^4} \\
& >    3 \lambda^{\varepsilon / 2} \;,
\end{align*}
else
\begin{align*}
\overline{\vec{p}}^\top \overline{\vec{f}}
& >    \frac{\eta_k^2 - \eta_{k'}^2}{32 \eta_k} \nu^*_k \nu^*_{k'} \,
       \overline{\vec{v}}^\top \vec{g}
     / (\|\vec{p}\| \|\vec{f}\|) \\
& >    \frac{(\eta_k^2 - \eta_{k'}^2) \nu^*_k \nu^*_{k'}}
            {32 \eta_k \sqrt{(\eta_k + \eta_{k'})^2 {\nu^*_{k'}}^2 +
                             4 \eta_k^2 {\nu^*_k}^2}} \,
       \overline{\vec{v}}^\top \overline{\vec{g}} \\
& \geq \frac{\delta^6}{16 \sqrt{2} d^2 {\Delta\!}^4}
       \lambda^{\varepsilon / 3} \\
& >    3 \lambda^{\varepsilon / 2}
\end{align*}
since $\lambda^{-\varepsilon / 6}
  \geq n^{\frac{9}{2} n {\Delta\!}^2 / \delta^3}
  \geq {\left(2^{\frac{9}{4} n {\Delta\!}^2 / \delta^3}\right)\!}^2
     > {\left(\frac{11 n {\Delta\!}^2}{\delta^3}\right)\!}^2
     > \frac{48 \sqrt{2} d^2 {\Delta\!}^4}{\delta^6}$.

\subparagraph{Case~$\widehat{\Xi}$.}

Here $\overline{\vec{v}}^\top \overline{\vec{g}} = \lambda^{\varepsilon / 3}$.

Hence, by \autoref{pr:S1}, since
$\lambda^{\varepsilon / 3}
 <    \frac{1}{2}
      \frac{\delta^6}{d {\Delta\!}^4}
 \leq \frac{1}{2}
      {\left(\frac{\eta_d \nu^*_d}{\|\vec{\gamma}_{[n]}\|}\right)\!}^2$,
necessarily $\ell \neq 1$.

Picking
\[\vec{p} \coloneqq
  \sum_{k = 1}^d
  \eta_k \nu^*_k
  \!\left(\!1 - \frac{2 \nu_k}{\nu^*_k} -
            \lambda^{\varepsilon / 3}
            \!\left(\frac{\nu_k}{\nu^*_k}
                    \frac{\|\vec{g}\|}{\eta_k \|\vec{v}\|} -
                    \!\left(\!1 - \frac{\nu_k}{\nu^*_k}\!\right)\!
                    \frac{\eta_k \|\vec{v}\|}{\|\vec{g}\|}\right)\!\!\right)\!
  \vec{u}_k
\;,\]
we have
\begin{align*}
\vec{p}^\top \vec{g}
& =    \sum_{k = 1}^d
       \eta_k^2 {\nu^*_k}^2
       \!\left(\!1 - \frac{\nu_k}{\nu^*_k}\!\right)\!
       \!\left(\!1 - \frac{2 \nu_k}{\nu^*_k} -
                 \lambda^{\varepsilon / 3}
                 \!\left(\frac{\nu_k}{\nu^*_k}
                         \frac{\|\vec{g}\|}{\eta_k \|\vec{v}\|} -
                         \!\left(\!1 - \frac{\nu_k}{\nu^*_k}\!\right)\!
                         \frac{\eta_k \|\vec{v}\|}
                              {\|\vec{g}\|}\right)\!\!\right)\! \\
& =    \sum_{k = 1}^d
       \eta_k^2 {\nu^*_k}^2
       \!\left(\!1 - \frac{\nu_k}{\nu^*_k}\!\right)\!
       \!\left(\!1 - \frac{2 \nu_k}{\nu^*_k}\!\right)\! +
       \lambda^{\varepsilon / 3}
       \sum_{k = 1}^d
       \eta_k^3 {\nu^*_k}^2
       {\left(\!1 - \frac{\nu_k}{\nu^*_k}\!\right)\!}^2
       \frac{\|\vec{v}\|}{\|\vec{g}\|} -
       \lambda^{2 \varepsilon / 3} \|\vec{g}\|^2 \\
& =    \|\vec{g}\|^2 +
       \sum_{k = 1}^{\ell - 1}
       \eta_k^2 {\nu^*_k}^2
       \frac{\nu_k}{\nu^*_k}
       \!\left(\!\frac{\nu_k}{\nu^*_k} - 1\!\right)\! -
       \sum_{k = \ell}^d
       \eta_k^2 {\nu^*_k}^2
       \frac{\nu_k}{\nu^*_k}
       \!\left(\!1 - \frac{\nu_k}{\nu^*_k}\!\right)\! \\ & \quad +
       \lambda^{\varepsilon / 3}
       \sum_{k = 1}^d
       \eta_k^3 {\nu^*_k}^2
       {\left(\!1 - \frac{\nu_k}{\nu^*_k}\!\right)\!}^2
       \frac{\|\vec{v}\|}{\|\vec{g}\|} -
       \lambda^{2 \varepsilon / 3} \|\vec{g}\|^2 \\
& \geq \|\vec{g}\|^2 +
       \eta_{\ell - 1}
       \sum_{k = 1}^{\ell - 1}
       \eta_k {\nu^*_k}^2
       \frac{\nu_k}{\nu^*_k}
       \!\left(\!\frac{\nu_k}{\nu^*_k} - 1\!\right)\! -
       \eta_\ell
       \sum_{k = \ell}^d
       \eta_k {\nu^*_k}^2
       \frac{\nu_k}{\nu^*_k}
       \!\left(\!1 - \frac{\nu_k}{\nu^*_k}\!\right)\! \\ & \quad +
       \lambda^{\varepsilon / 3}
       \eta_d \|\vec{v}\| \|\vec{g}\| -
       \lambda^{2 \varepsilon / 3} \|\vec{g}\|^2 \\
& =    \|\vec{g}\|^2 +
       \frac{\eta_{\ell - 1} - \eta_\ell}{2}
       \!\left(
       \sum_{k = 1}^{\ell - 1}
       \eta_k {\nu^*_k}^2
       \frac{\nu_k}{\nu^*_k}
       \!\left(\!\frac{\nu_k}{\nu^*_k} - 1\!\right)\! + \!
       \sum_{k = \ell}^d
       \eta_k {\nu^*_k}^2
       \frac{\nu_k}{\nu^*_k}
       \!\left(\!1 - \frac{\nu_k}{\nu^*_k}\!\right)\!
       \!\right)\! \\ & \quad -
       \lambda^{\varepsilon / 3}
       \!\left(\!\frac{\eta_{\ell - 1} + \eta_\ell}{2} - \eta_d\!\right)\!
       \|\vec{v}\| \|\vec{g}\| -
       \lambda^{2 \varepsilon / 3} \|\vec{g}\|^2 \\
& \geq \!\left(
       \frac{1}{8}
       \!\left(\!1 - \frac{\eta_\ell}{\eta_{\ell - 1}}\!\right)\!
       \eta_d \!\left(\:\!\!\min_{k = 1}^{\ell - 1} \nu^*_k\!\right)\! +
       \|\vec{g}\| -
       \lambda^{\varepsilon / 3}
       \!\left(\!\frac{\eta_{\ell - 1} + \eta_\ell}{2} - \eta_d\!\right)\!
       \|\vec{v}\| -
       \lambda^{2 \varepsilon / 3} \|\vec{g}\|
       \!\right)\!
       \|\vec{g}\|
\end{align*}
and
\begin{align*}
\vec{p}^\top \vec{v}
& =    \sum_{k = 1}^d
       \eta_k {\nu^*_k}^2
       \frac{\nu_k}{\nu^*_k}
       \!\left(\!1 - \frac{2 \nu_k}{\nu^*_k} -
                 \lambda^{\varepsilon / 3}
                 \!\left(\frac{\nu_k}{\nu^*_k}
                         \frac{\|\vec{g}\|}{\eta_k \|\vec{v}\|} -
                         \!\left(\!1 - \frac{\nu_k}{\nu^*_k}\!\right)\!
                         \frac{\eta_k \|\vec{v}\|}{\|\vec{g}\|}\right)\!\!\right)\! \\
& >    - \eta_1 \|\vec{v}\|^2
       + \lambda^{\varepsilon / 3}
         \frac{\|\vec{v}\|}{\|\vec{g}\|}
         \sum_{k = 1}^d
         \eta_k^2 {\nu^*_k}^2
         \frac{\nu_k}{\nu^*_k}
         \!\left(\!1 - \frac{\nu_k}{\nu^*_k}\!\right)\! \\
& \geq - \eta_1 \|\vec{v}\|^2
       - \frac{\lambda^{\varepsilon / 3}}{n}
         \frac{\|\vec{v}\|}{\|\vec{g}\|}
         \!\left(\!
         \eta_1
         \sum_{k = 1}^{\ell - 1}
         \eta_k {\nu^*_k}^2
         \frac{\nu_k}{\nu^*_k}
         \!\left(\!\frac{\nu_k}{\nu^*_k} - 1\!\right)\! -
         \eta_d
         \sum_{k = \ell}^d
         \eta_k {\nu^*_k}^2
         \frac{\nu_k}{\nu^*_k}
         \!\left(\!1 - \frac{\nu_k}{\nu^*_k}\!\right)\!
         \!\right)\! \\
& =    - \eta_1 \|\vec{v}\|^2 (1 - \lambda^{2 \varepsilon / 3})
       - \lambda^{\varepsilon / 3}
         \frac{\|\vec{v}\|}{\|\vec{g}\|}
         (\eta_1 - \eta_d)
         \sum_{k = \ell}^d
         \eta_k {\nu^*_k}^2
         \frac{\nu_k}{\nu^*_k}
         \!\left(\!1 - \frac{\nu_k}{\nu^*_k}\!\right)\! \\
& \geq - \eta_1
         \!\left(\!
         1 + \lambda^{\varepsilon / 3}
             \!\left(\!1 - \frac{\eta_d}{\eta_1}\!\right)\!
           - \lambda^{2 \varepsilon / 3}
         \!\right)\!
         \|\vec{v}\|^2 \;,
\end{align*}
also
\begin{align*}
\|\vec{p}\|
& < 2 \sqrt{
      \sum_{k = 1}^d
      \!\left(\!
      (\eta_1 \nu_k)^2 +
      (\eta_k (\nu^*_k - \nu_k))^2 +
      {\left(\!\lambda^{\varepsilon / 3} \nu_k
               \frac{\|\vec{g}\|}{\|\vec{v}\|}\right)\!}^2 \! +
      {\left(\!\lambda^{\varepsilon / 3} \eta_1 \eta_k (\nu^*_k - \nu_k)
               \frac{\|\vec{v}\|}{\|\vec{g}\|}\right)\!}^2
      \right)\!} \\
& = 2 \sqrt{1 + \lambda^{2 \varepsilon / 3}}
      \sqrt{\eta_1^2 \|\vec{v}\|^2 + \|\vec{g}\|^2} \\
& < 2 (1 + \lambda^{\varepsilon / 3})
      (\eta_1 \|\vec{v}\| + \|\vec{g}\|) \\
& < 4 (1 + \lambda^{\varepsilon / 3}) \eta_1
\end{align*}
and
\[\|\vec{f}\|
= (1 + \lambda^{\varepsilon / 3}) \|\vec{v}\| \|\vec{g}\| \;.\]

Therefore
\begin{align*}
\overline{\vec{p}}^\top \overline{\vec{f}}
& >
\!\biggl[
\frac{1}{8}
\!\left(\!1 - \frac{\eta_\ell}{\eta_{\ell - 1}}\!\right)\!
\eta_d \!\left(\:\!\!\min_{k = 1}^{\ell - 1} \nu^*_k\!\right)\!
\\ & \quad\;\: -
\lambda^{\varepsilon / 3}
\!\left(\!\frac{\eta_{\ell - 1} + \eta_\ell}{2} + \eta_1 - \eta_d\!\right)\!
\|\vec{v}\|
\\ & \quad\;\: -
\lambda^{2 \varepsilon / 3}
((\eta_1 - \eta_d) \|\vec{v}\| + \|\vec{g}\|)
\biggr]
\\ & \quad\;\: \Big/
4 (1 + \lambda^{\varepsilon / 3})^2 \eta_1
& & \text{calculation} \\
& >
\frac{1}{40}
\!\left(\!1 - \frac{\eta_\ell}{\eta_{\ell - 1}}\!\right)\!
\frac{\eta_d}{\eta_1} \!\left(\:\!\!\min_{k = 1}^{\ell - 1} \nu^*_k\!\right)\!
\\ & \quad\: -
\frac{\lambda^{\varepsilon / 3}}{5}
\!\left(\!\frac{\eta_{\ell - 1} + \eta_\ell}{2 \eta_1}
          + 1 - \frac{\eta_d}{\eta_1}\!\right)\!
\|\vec{v}\|
\\ & \quad\: -
\frac{\lambda^{2 \varepsilon / 3}}{5}
\!\left(\!\!\left(\!1 - \frac{\eta_d}{\eta_1}\!\right)\! \|\vec{v}\|
          + \frac{\|\vec{g}\|}{\eta_1}\right)\!
& & \text{since $\lambda^{\varepsilon / 3}
            \leq n^{-9 n}
            \leq 2^{-9 \cdot 2}
               < \sqrt{5 / 4} - 1$} \\
& >
\frac{\delta^4}{40 d \sqrt{d} {\Delta\!}^3} -
\frac{2}{5} \lambda^{\varepsilon / 3} -
\frac{2}{5} \lambda^{2 \varepsilon / 3}
& & \begin{aligned}[t]
    & \text{by the definitions of~$\delta$ and~$\Delta$,} \\
    & \text{\autoref{pr:d.D}~\ref{pr:d.D.alphakp1k}, and \autoref{pr:ball}}
    \end{aligned} \\
& >
\frac{1}{40}
\!\left(
\frac{\delta^4}{d \sqrt{d} {\Delta\!}^3} -
17 \lambda^{\varepsilon / 3}
\!\right)\!
& & \text{since $\lambda^{\varepsilon / 3}
            \leq n^{-9 n}
            \leq 2^{-9 \cdot 2}
               < 1 / 16$} \\
& >
\frac{\delta^4}{80 d \sqrt{d} {\Delta\!}^3}
& & \text{since }
    \begin{multlined}[t]
         \lambda^{-\varepsilon / 3}
    \geq n^{9 n {\Delta\!}^2 / \delta^3}
    \geq {\left(2^{6 n {\Delta\!}^2 / \delta^3}\right)\!}^{3/2} \\
    \geq (2^{11} n {\Delta\!}^2 / \delta^3)^{3/2}
       > \tfrac{17 \cdot 80 d \sqrt{d} {\Delta\!}^3}{\delta^4}
    \end{multlined} \\
& >
3 \lambda^{\varepsilon / 2}
& & \text{since }
    \begin{multlined}[t]
         \lambda^{-\varepsilon / 2}
    \geq n^{\frac{9 \cdot 3}{2} n {\Delta\!}^2 / \delta^3}
    \geq {\left(2^{9 n {\Delta\!}^2 / \delta^3}\right)\!}^{3/2} \\
    \geq (2^{17} n {\Delta\!}^2 / \delta^3)^{3/2}
       > \tfrac{3 \cdot 80 d \sqrt{d} {\Delta\!}^3}{\delta^4} \;,
    \end{multlined}
\end{align*}
completing the proof.
\end{proof}

\begin{example}
\label{ex:running}
To illustrate some aspects of the training dynamics, let us consider a single run of gradient descent with learning rate~$0.01$, for a network of width $m = 25$ initialised using $\vec{z}_j \!\iid \mathcal{N}(\vec{0}, \frac{1}{d \, m} \vec{I}_d)$ and $s_j \!\iid \mathcal{U}\{\pm 1\}$ with scale $\lambda = 4^{-7}$, and on a synthetic uncentred training dataset in dimension $d = 16$ as described in \autoref{s:exp}.

\autoref{f:crossing} shows the coordinates of the vector $\vec{v}^* - \vec{v}^t$ in the eigenvectors basis crossing zero one by one exactly in the order of their indices, i.e.~in the decreasing order of the corresponding eigenvalues of the matrix $\frac{1}{n} \vec{X} \vec{X}^\top$.  Thus the bundle vector~$\vec{v}^t$ travels through the subsets $\mathcal{S}_1, \mathcal{S}_2, \ldots$ of the set~$\mathcal{S}$ exactly in their order, and in line with what we established in the proof of \autoref{l:S}, passing through each~$\mathcal{S}_k$ at most once.
\end{example}

\def\PgfmathparseFPU#1{\begingroup%
  \pgfkeys{/pgf/fpu,/pgf/fpu/output format=fixed}%
  \pgfmathparse{#1}%
  \pgfmathsmuggle\pgfmathresult\endgroup}%

\begin{figure}
\centering
\includegraphics{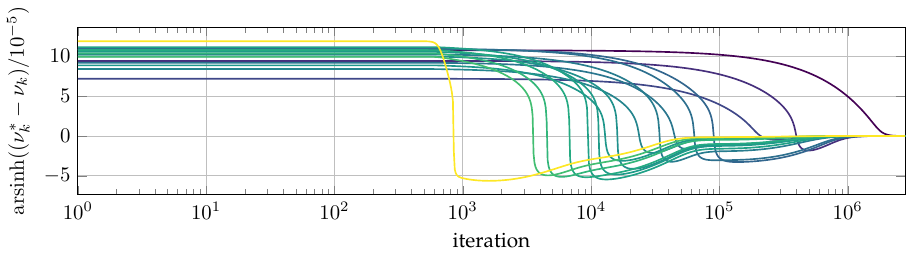}
\\[2ex]
\includegraphics{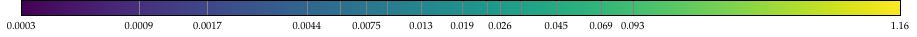}
\caption{The coordinates in the eigenvectors basis of the difference between the teacher neuron and the weighted sum of the hidden neurons crossing zero in the decreasing eigenvalue order.  The horizontal axis is logarithmic. The vertical axis shows the values mapped using the inverse of the hyperbolic sine in order to be able to visualise numbers at different scales on both sides around zero. The colours are picked from a colourmap based on the corresponding eigenvalue.}
\label{f:crossing}
\end{figure}

Let us write $\vec{v}^t = \sum_{k = 1}^d \nu_k^t \vec{u}_k$.

Recall from \autoref{s:second} that
$T_2 =
 \inf \{t \geq T_1 \,\,\vert\,\,
        \nu_1^t / \nu^*_1 \geq 1 / 2\}$.

Building on the preceding results, the final lemma in this section establishes that the loss converges to zero at an exponential rate.  To show it, we partition the second phase of the training into two, namely before and after the time~$T_2$ which is when the coordinate of the bundle vector~$\vec{v}^t$ with respect to the largest-eigenvalue eigenvector of the matrix $\frac{1}{n} \vec{X} \vec{X}^\top$ crosses the half-way threshold to the corresponding coordinate of the teacher neuron~$\vec{v}^*$.  The period before~$T_2$ (and after the start~$T_1$ of the second phase) consists of an exponentially fast departure of~$\vec{v}^t$ from near the saddle at the origin, whereas the period after~$T_2$ obeys a Polyak-{\L}ojasiewicz inequality which implies exponentially fast convergence.

\begin{lemma}
\label{l:2}
\begin{enumerate}[(i),itemsep=0ex,leftmargin=3em]
\item
\label{l:2.v.1}
$\|\vec{v}^t\| < \frac{1}{2}$
for all $t \in [T_1, T_2]$.
\item
\label{l:2.T2}
$T_2 - T_1 <
 \ln\bigl(\frac{1}{\lambda}\bigr)
 \frac{(4 + \varepsilon / 2) d {\Delta\!}^2}{\delta^6}$
and
$T_2 <
 \ln\bigl(\frac{1}{\lambda}\bigr)
 \frac{(4 + \varepsilon) d {\Delta\!}^2}{\delta^6}$.
\item
\label{l:2.v.2}
$\|\vec{v}^t\| > \frac{\|\vec{\gamma}_{[n]}\|}{4 \eta_1}$
for all $t \geq T_2$.
\item
\label{l:2.PL}
$\|\nabla L(\vec{\theta}^t)\|^2 >
 \frac{2 \eta_d \|\vec{\gamma}_{[n]}\|}{5 \eta_1}
 L(\vec{\theta}^t)$
for all $t \geq T_2$.
\item
\label{l:2.L}
$L(\vec{\theta}^t) <
 \frac{{\Delta\!}^2}{2}
 \exp\Bigl(-(t - T_2) \frac{2 \delta^4}{5 {\Delta\!}^2}\Bigr)$
for all $t \geq T_2$.
\end{enumerate}
\end{lemma}

\begin{proof}
By \autoref{l:len.w}~\ref{l:len.w.der}, we have $\|\vec{v}^{T_1}\| > |J_\ppp| (1 - 4 \lambda^\varepsilon) \lambda \min_{j \in J_\ppp} \|\vec{z}_j\|$.

For all $t \in [T_1, T_2]$,
by \autoref{l:S}~\ref{l:S.S} we have
$\|\vec{v}^t\| < \frac{\|\vec{v}^*\|}{2}$,
and by \autoref{l:S}~\ref{l:S.dvS1} we have
$\frac{\mathrm{d}}{\mathrm{d} t} \|\vec{v}^t\|^2 >
 \frac{\eta_d^2 {\nu^*_d}^2}{2 \|\vec{\gamma}_{[n]}\|}
 \|\vec{v}^t\|^2$.

Hence
\begin{align*}
T_2 - T_1
& <    \!\left(\!
       \ln\!\left(\frac{1}{\lambda}\right)\! +
       \ln\!\left(\frac{\|\vec{v}^*\|}{\min_{j \in J_\ppp} \|\vec{z}_j\|}\right)\!
       \!\right)\!
       \frac{4 \|\vec{\gamma}_{[n]}\|}{\eta_d^2 {\nu^*_d}^2} \\
& \leq \!\left(\!
       \ln\!\left(\frac{1}{\lambda}\right)\! +
       \ln\!\left(\frac{1}{\delta}\right)\!
       \!\right)\!
       \frac{4 d {\Delta\!}^2}{\delta^6} \\
& <   \ln\!\left(\frac{1}{\lambda}\right)\!
      \frac{(4 + \varepsilon / 2) d {\Delta\!}^2}{\delta^6}
\end{align*}
since $\ln(1 / \lambda) \varepsilon / 2
  \geq \frac{9 \cdot 3}{2} n (\ln n) / \delta^3
  \geq 9 \cdot 3 (\ln 2) / \delta^3
     > 4 \ln(1 / \delta)$, and so
\[T_2
<    \ln\!\left(\frac{1}{\lambda}\right)\!
     \frac{\varepsilon}{\|\vec{\gamma}_{[n]}\|}
   + \ln\!\left(\frac{1}{\lambda}\right)\!
     \frac{(4 + \varepsilon / 2) d {\Delta\!}^2}{\delta^6}
\leq \ln\!\left(\frac{1}{\lambda}\right)\!
     \frac{(4 + \varepsilon) d {\Delta\!}^2}{\delta^6} \;.\]

By \autoref{l:S}~\ref{l:S.S} we have
$L(\vec{\theta}^{T_2})
< \frac{1}{2} {\left(\!1 - \frac{\eta_d}{4 \eta_1}\!\right)\!}^2
  \|\vec{v}^*\| \|\vec{\gamma}_{[n]}\|
< \frac{\|\vec{\gamma}_{[n]}\|}{2}$.

For all $t \geq T_2$,
by \autoref{l:S}~\ref{l:S.S} we have
$\|\vec{v}^t\| > \frac{\|\vec{\gamma}_{[n]}\|}{4 \eta_1}$,
and so
\begin{align*}
\|\nabla L(\vec{\theta}^t)\|^2
& = -\mathrm{d} L(\vec{\theta}^t) / \mathrm{d} t \\
& = {\vec{g}^t}^\top \mathrm{d} \vec{v}^t / \mathrm{d} t \\
& \geq {\vec{g}^t}^\top \vec{f}^t -
       3 \lambda^{\varepsilon / 2} \|\vec{g}^t\| \|\vec{f}^t\|) \\
& >    \!\left(\!
       (1 + \lambda^{2 \varepsilon / 3})
       \frac{\|\vec{\gamma}_{[n]}\|}{4 \eta_1} -
       6 \lambda^{\varepsilon / 2} \|\vec{v}^*\|
       \!\right)\!
       \|\vec{g}^t\|^2 \\
& >    \frac{\|\vec{\gamma}_{[n]}\|}{5 \eta_1}
       \|\vec{g}^t\|^2 \\
& \geq \frac{\eta_d}{5 \eta_1}
       \|\vec{\gamma}_{[n]}\|
       \|\vec{v}^* - \vec{v}^t\| \|\vec{g}^t\| \\
& \geq \frac{2 \eta_d}{5 \eta_1}
       \|\vec{\gamma}_{[n]}\|
       L(\vec{\theta}^t)
\end{align*}
since $\lambda^{\varepsilon / 2}
  \leq n^{-\frac{9 \cdot 3}{2} n {\Delta\!}^2 / \delta^3}\!
     < \frac{\delta^2}{120 {\Delta\!}^2}
  \leq \frac{\|\vec{\gamma}_{[n]}\|}{120 \eta_1}$.

Hence for all $t \geq T_2$ we have
\[L(\vec{\theta}^t)
<    \frac{\|\vec{\gamma}_{[n]}\|}{2}
     \exp\!\left(\!-(t - T_2) \frac{2 \eta_d \|\vec{\gamma}_{[n]}\|}{5 \eta_1}\right)\!
\leq \frac{{\Delta\!}^2}{2}
     \exp\!\left(\!-(t - T_2) \frac{2 \delta^4}{5 {\Delta\!}^2}\right)\! \;.
\qedhere\]
\end{proof}

\begin{example}
Using the single run from \autoref{ex:running} again, we illustrate in \autoref{f:length} the progression of several significant measures of the sum~$\vec{v}^t$ of the hidden neurons multiplied by their last-layer weights during the training.  In particular, we can see that the alignment with the vector~$\vec{\gamma}_{[n]}$ reaches its maximum around iteration~$500$, after which the distance and the angle to the teacher neuron~$\vec{v}^*$ starts to decrease.
\end{example}

\begin{figure}
\centering
\includegraphics{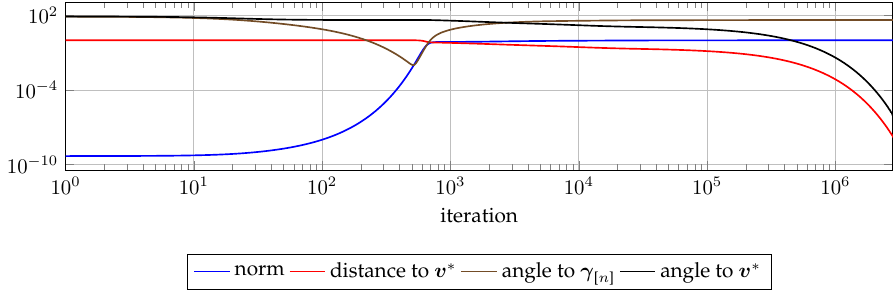}
\caption{The evolution of several measures of the weighted sum of the hidden neurons during the training.  Both axes are logarithmic, and the angles are in degrees.}
\label{f:length}
\end{figure}

\section{Proofs for the implicit bias}

In this section, we include an explicit subscript~$\lambda$ for quantities that depend on the initialisation scale.

A key part of showing that the networks to which the training converges as time tends to infinity themselves have a limit in parameter space as the initialisation scale tends to zero is to establish the existence of that double limit for every ratio between the Euclidean norms of two hidden neurons.  The following lemma does that, and provides two alternative expressions for each such double-limit ratio: the limit of the same ratio at time~$T_1$ as $\lambda$~tends to zero, and the corresponding limit for the yardstick trajectories as $t$~tends to infinity.

\begin{lemma}
\label{l:lim}
For all $j, j' \in J_\ppp$ we have
\[\lim_{\lambda \to 0^+}
  \lim_{t \to \infty\mathstrut}
  \frac{\|\vec{w}_{\lambda, j}^t\|}
       {\|\vec{w}_{\lambda, j'}^t\|}
= \lim_{\lambda \to 0^+}
  \frac{\|\vec{w}_{\lambda, j}^{T_{\lambda, 1}}\|}
       {\|\vec{w}_{\lambda, j'}^{T_{\lambda, 1}}\|}
= \lim_{t \to \infty\mathstrut}
  \frac{\|\vec{\alpha}_j^t\|}
       {\|\vec{\alpha}_{j'}^t\|} \;.\]
\end{lemma}

\begin{proof}
Suppose $j, j' \in J_\ppp$.

Recalling \autoref{pr:0+} and letting $u_j \coloneqq \artanh \cos \varphi_j^{T_0}$, for all $t \geq T_0$ we have
\begin{align*}
\left|
\mathrm{d} \ln \frac{\|\vec{\alpha}_j^t\|}
                    {\|\vec{\alpha}_{j'}^t\|} /
\mathrm{d} t
\right|
& = |{(\overline{\vec{\alpha}}_j^t
     - \overline{\vec{\alpha}}_{j'}^t)\!}^\top
     \vec{\gamma}_{[n]}| \\
& = \bigl|\tanh(u_j + \|\vec{\gamma}_{[n]}\| (t - T_0))
        - \tanh(u_{j'} + \|\vec{\gamma}_{[n]}\| (t - T_0))\bigr| \,
    \|\vec{\gamma}_{[n]}\| \\
& = |\!\tanh(u_j - u_{j'})| \, \|\vec{\gamma}_{[n]}\| \\ & \quad\;
    \bigl(1 - \tanh(u_j + \|\vec{\gamma}_{[n]}\| (t - T_0))
              \tanh(u_{j'} + \|\vec{\gamma}_{[n]}\| (t - T_0))\bigr) \\
& < |\!\tanh(u_j - u_{j'})| \, \|\vec{\gamma}_{[n]}\|
    \bigl(1 - \tanh^2(\|\vec{\gamma}_{[n]}\| (t - T_0))\bigr) \\
& < |\!\tanh(u_j - u_{j'})| \, \|\vec{\gamma}_{[n]}\|
    \!\left(\!
    1 - \bigl(1 - 2 /\! \exp(2 \|\vec{\gamma}_{[n]}\| (t - T_0))\bigr)^2
    \right)\! \\
& < 4 |\!\tanh(u_j - u_{j'})| \, \|\vec{\gamma}_{[n]}\|
    \exp(-2 \|\vec{\gamma}_{[n]}\| (t - T_0)) \;,
\end{align*}
so
${\displaystyle \lim_{t \to \infty\mathstrut}}
 \frac{\|\vec{\alpha}_j^t\|}
      {\|\vec{\alpha}_{j'}^t\|}$
exists.

By \autoref{l:1}~\ref{l:1.lengths}, we have
\begin{multline*}
\left|\ln \frac{\|\vec{w}_{\lambda, j}^{T_{\lambda, 1}}\|}
               {\|\vec{w}_{\lambda, j'}^{T_{\lambda, 1}}\|}
    - \ln \frac{\|\vec{\alpha}_j^{T_{\lambda, 1}}\|}
               {\|\vec{\alpha}_{j'}^{T_{\lambda, 1}}\|}\right| \\
=    \Bigl|\bigl(\ln \|\vec{w}_{\lambda, j}^{T_{\lambda, 1}} / \lambda\|
               - \ln \|\vec{\alpha}_j^{T_{\lambda, 1}}\|\bigr)
         - \bigl(\ln \|\vec{w}_{\lambda, j'}^{T_{\lambda, 1}} / \lambda\|
               - \ln \|\vec{\alpha}_{j'}^{T_{\lambda, 1}}\|\bigr)\Bigr|
\leq 2 \lambda^{1 - 3 \varepsilon} \;,
\end{multline*}
so
${\displaystyle \lim_{\lambda \to 0^+}}
 \frac{\bigl\|\vec{w}_{\lambda, j\mathstrut}^{T_{\lambda, 1}}\bigr\|}
      {\bigl\|\vec{w}_{\lambda, j'}^{T_{\lambda, 1}\mathstrut}\bigr\|}
= {\displaystyle \lim_{t \to \infty\mathstrut}}
  \frac{\|\vec{\alpha}_j^t\|}
       {\|\vec{\alpha}_{j'}^t\|}$.

By \autoref{l:S} \ref{l:S.w},~\ref{l:S.g}, and~\ref{l:S.diff}, for all $t \geq T_{\lambda, 1}$ we have
\[\left|
  \mathrm{d} \ln \frac{\|\vec{w}_{\lambda, j}^t\|}
                      {\|\vec{w}_{\lambda, j'}^t\|} /
  \mathrm{d} t
  \right|
= |{(\overline{\vec{w}}_{\lambda, j}^t
   - \overline{\vec{w}}_{\lambda, j'}^t)\!}^\top
     \vec{g}_\lambda^t|
< \sqrt{8} \lambda^{\varepsilon / 2} \|\vec{g}_\lambda^t\| \;.\]

By \autoref{l:S}~\ref{l:S.dvs-v} and \autoref{l:2}~\ref{l:2.v.2}, for all $t \geq T_{\lambda, 2}$ we have
\[\|\vec{g}_\lambda^t\|
\leq \eta_1 \|\vec{v}^* - \vec{v}_\lambda^t\|
<    \eta_1 \exp(-5 \|\vec{\gamma}_{[n]}\| (t - T_{\lambda, 2}) / 8) \;.\]

Hence
${\displaystyle \lim_{t \to \infty\mathstrut}}
 \frac{\|\vec{w}_{\lambda, j}^t\|}
      {\|\vec{w}_{\lambda, j'}^t\|}$
exists.
Moreover, by \autoref{l:2}~\ref{l:2.T2}, for all $t \geq T_{\lambda, 2}$ we have
\begin{align*}
\left|\ln \frac{\|\vec{w}_{\lambda, j}^t\|}
               {\|\vec{w}_{\lambda, j'}^t\|}
    - \ln \frac{\|\vec{w}_{\lambda, j}^{T_{\lambda, 1}}\|}
               {\|\vec{w}_{\lambda, j'}^{T_{\lambda, 1}}\|}\right|
& < \sqrt{8} \lambda^{\varepsilon / 2} {\Delta\!}^2
    \!\left(\!
    \ln\!\left(\frac{1}{\lambda}\right)\!
    \frac{(4 + \varepsilon / 2) d {\Delta\!}^2}{\delta^6}
  + \!\int_0^{t - T_{\lambda, 2}}\! \exp(-5 \delta^2 t' / 8) \, \mathrm{d} t'
    \!\right)\! \\
& < \sqrt{8} \lambda^{\varepsilon / 2} {\Delta\!}^2
    \!\left(\!
    \ln\!\left(\frac{1}{\lambda}\right)\!
    \frac{(4 + \varepsilon / 2) d {\Delta\!}^2}{\delta^6}
  + \frac{8}{5 \delta^2}
    \!\right)\! \\
& < \sqrt{8} \lambda^{\varepsilon / 2}
    \ln\!\left(\frac{1}{\lambda}\right)\!
    \frac{(4 + \varepsilon) d {\Delta\!}^4}{\delta^6} \\
& < \lambda^{\varepsilon / 3}
    \ln\!\left(\frac{1}{\lambda}\right)\!
\end{align*}
since $\lambda^{-\varepsilon / 6}
  \geq n^{\frac{9}{2} n {\Delta\!}^2 / \delta^3}
  \geq {\left(2^{\frac{9}{4} n {\Delta\!}^2 / \delta^3}\right)\!}^2
     > {\left(\frac{11 n {\Delta\!}^2}{\delta^3}\right)\!}^2
     > \frac{\sqrt{8} (4 + \varepsilon) d {\Delta\!}^4}{\delta^6}$.
Therefore
\[\left|\lim_{t \to \infty\mathstrut}
        \frac{\|\vec{w}_{\lambda, j}^t\|}
             {\|\vec{w}_{\lambda, j'}^t\|}
     - \frac{\|\vec{w}_{\lambda, j}^{T_{\lambda, 1}}\|}
            {\|\vec{w}_{\lambda, j'}^{T_{\lambda, 1}}\|}\right|
< \lambda^{\varepsilon / 3}
  \ln\!\left(\frac{1}{\lambda}\right)\! \;,\]
so
${\displaystyle \lim_{\lambda \to 0^+}}
 {\displaystyle \lim_{t \to \infty\mathstrut}}
 \frac{\|\vec{w}_{\lambda, j}^t\|}
      {\|\vec{w}_{\lambda, j'}^t\|}
= {\displaystyle \lim_{\lambda \to 0^+}}
  \frac{\bigl\|\vec{w}_{\lambda, j\mathstrut}^{T_{\lambda, 1}}\bigr\|}
       {\bigl\|\vec{w}_{\lambda, j'}^{T_{\lambda, 1}\mathstrut}\bigr\|}$.
\end{proof}

We are now in a position to prove the main theorem, restated from \autoref{s:bias}.  It establishes that, as the initialisation scale~$\lambda$ tends to zero, the networks with zero loss to which the gradient flow converges tend to a network in the set~$\Theta_{\vec{v}^*}$ (defined in \autoref{s:bias}) of balanced interpolators of rank~$1$.

\thbias*

\begin{proof}
By \autoref{l:2}~\ref{l:2.L}, we have
${\displaystyle \lim_{t \to \infty\mathstrut}} L(\vec{\theta}_\lambda^t) = 0$.

By \autoref{pr:loss}, for all $t \in [0, \infty)$ we have
\[\int_t^\infty
  \left\|\frac{\mathrm{d}}{\mathrm{d} t'} \vec{\theta}_\lambda^{t'}\right\|^2 \!
  \mathrm{d} t'
= -\!\int_t^\infty\!
   \frac{\mathrm{d}}{\mathrm{d} t'} L\!\left(\vec{\theta}_\lambda^{t'}\right)\!
   \mathrm{d} t'
= L(\vec{\theta}_\lambda^t) \;.\]

Hence
$\vec{\theta}_\lambda^\infty \coloneqq
 {\displaystyle \lim_{t \to \infty\mathstrut}} \vec{\theta}_\lambda^t$
exists, and since the loss function is continuous, we have
$L(\vec{\theta}_\lambda^\infty) = 0$.

Let us write
$\vec{\theta}_\lambda^\infty =
 \bigl([a_{\lambda, 1}^\infty, \ldots, a_{\lambda, m}^\infty],
       [\vec{w}_{\lambda, 1}^\infty, \ldots, \vec{w}_{\lambda, m}^\infty]^\top\bigr)$,
and let
$\vec{v}_\lambda^\infty \coloneqq
 \sum_{j \in J_\ppp} a_{\lambda, j}^\infty \vec{w}_{\lambda, j}^\infty$.

Since $\Theta$~is closed, for all $j \in [m]$ we have $a_{\lambda, j}^\infty = \|\vec{w}_{\lambda, j}^\infty\|$.

Recalling $\lspn\{\vec{x}_1, \ldots, \vec{x}_n\} = \mathbb{R}^d$ and \autoref{l:S}~\ref{l:S.v}, we have $\vec{v}_\lambda^\infty = \vec{v}^*$.

By \autoref{l:S}~\ref{l:S.w}, for all $j \in J_\ppp$ we have
${\overline{\vec{w}}_{\lambda, j}^\infty\!\!}^\top
 \vec{v}^*
 > 1 - 4 \lambda^\varepsilon$,
so
\[1 \leq
  \sum_{j \in J_\ppp} \|\vec{w}_{\lambda, j}^\infty\|^2 <
  \frac{1}{1 - 4 \lambda^\varepsilon} \;,\]
and thus
${\displaystyle \lim_{\lambda \to 0^+}}
 \sum_{j \in J_\ppp} \|\vec{w}_{\lambda, j}^\infty\|^2
 = 1$.

By \autoref{l:lim}, for all $j, j' \in J_\ppp$,
${\displaystyle \lim_{\lambda \to 0^+}}
 \frac{\|\vec{w}_{\lambda, j}^\infty\|}
      {\|\vec{w}_{\lambda, j'}^\infty\|}$
exists.

Hence, for all $j \in J_\ppp$, we have that
$a_j^\infty \coloneqq
 {\displaystyle \lim_{\lambda \to 0^+}}
 \|\vec{w}_{\lambda, j}^\infty\|$
and
${\displaystyle \lim_{\lambda \to 0^+}}
 \overline{\vec{w}}_{\lambda, j}^\infty$
exist, and so also
$\vec{w}_j^\infty \coloneqq
 {\displaystyle \lim_{\lambda \to 0^+}}
 \vec{w}_{\lambda, j}^\infty$
exists.  Moreover, we have
$a_j^\infty = \|\vec{w}_j^\infty\|$
and
$\overline{\vec{w}}_j^\infty = \vec{v}^*$
for all $j \in J_\ppp$, and we have
$\sum_{j \in J_\ppp} \|\vec{w}_j^\infty\|^2 = 1$.

By \autoref{pr:T0.T1}, \autoref{l:len.w}~\ref{l:len.w.der}, \autoref{l:0}, and \autoref{ass:enum}~\ref{ass:enum.death}, for all $j \notin J_\ppp$, we have $\|\vec{w}_{\lambda, j}^t\| \leq \lambda \|\vec{z}_j\|$ for all $t \in [0, \infty)$, so
$a_j^\infty \coloneqq
 {\displaystyle \lim_{\lambda \to 0^+}}
 a_{\lambda, j}^\infty
 = 0$
and
$\vec{w}_j^\infty \coloneqq
 {\displaystyle \lim_{\lambda \to 0^+}}
 \vec{w}_{\lambda, j}^\infty
 = \vec{0}$.

Therefore
$\vec{\theta}^\infty \coloneqq
 \bigl([a_1^\infty, \ldots, a_m^\infty],
       [\vec{w}_1^\infty, \ldots, \vec{w}_m^\infty]^\top\bigr)
 \in \Theta_{\vec{v}^*}$.
\end{proof}

\begin{example}
Continuing with the single run from \autoref{ex:running}, we illustrate in \autoref{f:convergence} how, although the loss converges to zero exponentially fast, for a fixed positive initialisation scale the angles between the hidden neurons in the aligned bundle do not in general decrease to zero.

\autoref{f:complexity} shows the course of the training from the point of view of the two measures of network complexity, namely the nuclear and square Euclidean norms: during the alignment phase they are both close to zero, they grow rapidly as the network departs from the saddle at the origin, and they converge towards~$1$ and~$2$ respectively as the loss converges to zero.
\end{example}

\begin{figure}
\centering
\includegraphics{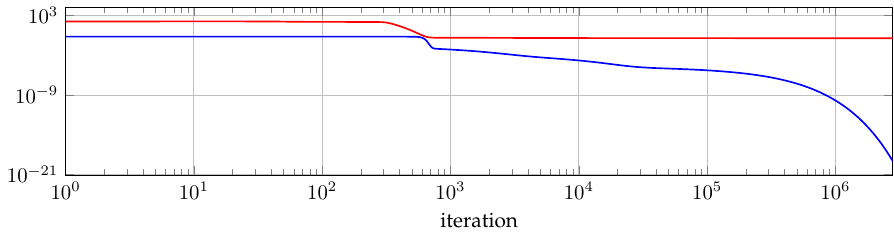}
\\[2.5ex]
\includegraphics{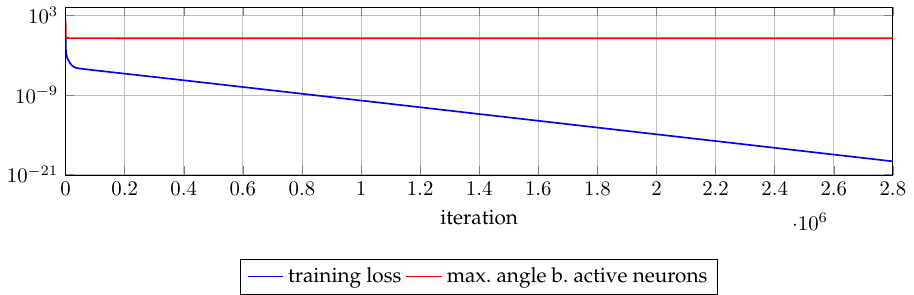}
\caption{The evolution of the training loss and the maximum angle between active hidden neurons during the training.  The two plots are of the same data, the vertical axes are logarithmic, the horizontal axis is logarithmic in the top plot and linear in the bottom plot, and the angles are in degrees.}
\label{f:convergence}
\end{figure}

\begin{figure}
\centering
\includegraphics{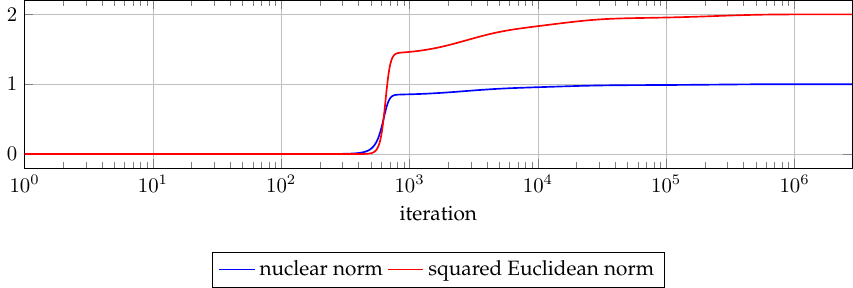}
\caption{The evolution of the nuclear and square Euclidean norms during the training.  The horizontal axis is logarithmic, and the vertical axis is linear.}
\label{f:complexity}
\end{figure}

\section{Proofs and examples for the interpolators}
\label{app:inter}

First we prove the following theorem, restated from \autoref{s:inter}.  For case $\mathcal{M} < 0$, the main part of our argument shows that, if a global minimiser of $\|\vec{\theta}\|^2$ was not a member of~$\Theta_{\vec{v}^*}$, then we could obtain from each hidden neuron~$\vec{w}_j$ a vector~$\vec{p}_j$ such that the inner products of the inputs~$\vec{x}_i$ with the vector $\sum_{j \in [m]} a_j \, \vec{p}_j$ coincide with the network outputs, the projection of each vector~$\vec{p}_j$ onto the teacher neuron has length at most~$\|\vec{w}_j\|$, and at least one of those inequalities is strict, leading to a contradiction.  For case $\mathcal{M} > 0$, we provide counterexample interpolator networks, the Euclidean norm of whose parameters is smaller than the Euclidean norm of the networks in~$\Theta_{\vec{v}^*}$.

\thmin*

\begin{proof}
For all $\vec{\theta} = (\vec{a}, \vec{W}) \in \Theta_{\vec{v}^*}$ we have $L(\vec{\theta}) = 0$ and $\|\vec{\theta}\|^2 = \sum_{j \in [m]} (a_j^2 + \|\vec{w}_j\|^2) = 2$.

To establish the case when $\mathcal{M} < 0$, supposing $\vec{\theta} = (\vec{a}, \vec{W}) \in \mathbb{R}^m \times \mathbb{R}^{m \times d}$ is a global minimiser of $\|\vec{\theta}\|^2$ subject to $L(\vec{\theta}) = 0$, it suffices to show $\vec{\theta} \in \Theta_{\vec{v}^*}$.

By the minimality of~$\|\vec{\theta}\|^2$ subject to $L(\vec{\theta}) = 0$, for all $j \in [m]$, if $a_j = 0$ then $\vec{w}_j = \vec{0}$, and also if $\forall i \in [d] \colon \sigma(\vec{w}_j^\top \vec{x}_i) = 0$ then $a_j = 0$.

For all $j \in [m]$, if $a_j = 0$ and $\vec{w}_j = \vec{0}$, then removing~$a_j$ and~$\vec{w}_j$ from~$\vec{\theta}$ preserves the values of~$\|\vec{\theta}\|^2$ and~$L(\vec{\theta})$, and the truth or falsity of $\vec{\theta} \in \Theta_{\vec{v}^*}$.
Hence we may assume for all $j \in [m]$ that $a_j \neq 0$ and $\exists i \in [d] \colon \sigma(\vec{w}_j^\top \vec{x}_i) \neq 0$.

For all $j \in [m]$, replacing~$a_j$ by $\sqrt{\|\vec{w}_j\| / |a_j|} \, a_j$ and~$\vec{w}_j$ by $\sqrt{|a_j| / \|\vec{w}_j\|} \, \vec{w}_j$ preserves~$L(\vec{\theta})$, and decreases~$\|\vec{\theta}\|^2$ unless $|a_j| = \|\vec{w}_j\|$.  Hence $\vec{\theta} \in \Theta$.

For all $j \in [m]$, let $K_j \coloneqq \{k \in [d] \,\,\vert\,\, \vec{w}_j^\top \vec{x}_k \geq 0\}$ and
\begin{align*}
\vec{p}_j & \coloneqq \sum_{k \in K_j} (\vec{w}_j^\top \vec{x}_k) \vec{\chi}_k &
\vec{q}_j & \coloneqq \sum_{k \notin K_j} -(\vec{w}_j^\top \vec{x}_k) \vec{\chi}_k \;,
\end{align*}
so that $\vec{w}_j = \vec{p}_j - \vec{q}_j$.  Observe also that since $\exists i \in [d] \colon \sigma(\vec{w}_j^\top \vec{x}_i) \neq 0$, we have $\vec{p}_j \neq \vec{0}$.

\begin{claim}
\label{cl:proj}
For all $j \in [m]$ we have $|\vec{p}_j^\top \vec{v}^*| \leq \|\vec{w}_j\|$, and if $\vec{q}_j \neq \vec{0}$ then the inequality is strict.
\end{claim}

\begin{proof}[Proof of claim]
Suppose $j \in [m]$.
If $\vec{q}_j = \vec{0}$ then $\vec{w}_j = \vec{p}_j$.
If $\vec{q}_j \neq \vec{0}$ then
\begin{align*}
\|\vec{w}_j\|^2
& =    \|\vec{p}_j\|^2 + \|\vec{q}_j\|^2
     - 2 \|\vec{p}_j\| \|\vec{q}_j\| \cos \angle(\vec{p}_j, \vec{q}_j) \\
& >    \|\vec{p}_j\|^2 + \|\vec{q}_j\|^2
     - 2 \|\vec{p}_j\| \|\vec{q}_j\| \sin \angle(\vec{p}_j, \vec{v}^*) \\
& =    \|\vec{p}_j\|^2 \cos^2 \angle(\vec{p}_j, \vec{v}^*)
     + (\|\vec{p}_j\| \sin \angle(\vec{p}_j, \vec{v}^*) - \|\vec{q}_j\|)^2 \\
& \geq \|\vec{p}_j\|^2 \cos^2 \angle(\vec{p}_j, \vec{v}^*) \\
& =    (\vec{p}_j^\top \vec{v}^*)^2 \;.
\qedhere
\end{align*}
\end{proof}

Now for all $i \in [d]$ we have
\[{\!\left(\sum_{j \in [m]} a_j \, \vec{p}_j\!\right)\!\!}^\top \vec{x}_i
= \!\sum_{j \in [m]} a_j \, \vec{p}_j^\top \vec{x}_i
= \!\sum_{j \in [m]} a_j \, \sigma(\vec{w}_j^\top \vec{x}_i)
= {\vec{v}^*}^\top \vec{x}_i \;.\]
Since $\lspn\{\vec{x}_1, \ldots, \vec{x}_d\} = \mathbb{R}^d$, we infer $\sum_{j \in [m]} a_j \, \vec{p}_j = \vec{v}^*$, so by \autoref{cl:proj} we have
\[1
=    \sum_{j \in [m]} a_j \, \vec{p}_j^\top \vec{v}^*
\leq \sum_{j \in [m]} |a_j \, \vec{p}_j^\top \vec{v}^*|
\leq \sum_{j \in [m]} \|a_j \, \vec{w}_j\|
=    \tfrac{1}{2} \!
     \sum_{j \in [m]} (a_j^2 + \|\vec{w}_j\|^2)
=    \tfrac{1}{2}
     \|\vec{\theta}\|^2
\leq 1 \;,\]
and if $\vec{q}_j \neq \vec{0}$ for some $j \in [m]$ then the second of the three inequalities is strict.  However, all three inequalities must be equalities, so also for all $j \in [m]$ we have $\vec{q}_j = \vec{0}$.  Hence $a_j \, \vec{w}_j^\top \vec{v}^* = a_j \, \vec{p}_j^\top \vec{v}^* = \|a_j \, \vec{w}_j\|$, and thus $\overline{a_j \, \vec{w}_j} = \vec{v}^*$.  Since $a_j < 0$ would imply $\overline{\vec{w}}_j = -\vec{v}^*$, which would contradict $\vec{q}_j = \vec{0}$, we have $a_j > 0$ and $\overline{\vec{w}}_j = \vec{v}^*$.  Therefore $\vec{\theta} \in \Theta_{\vec{v}^*}$.

To establish the case when $\mathcal{M} > 0$, it suffices to exhibit $\vec{\theta} = (\vec{a}, \vec{W}) \in \mathbb{R}^m \times \mathbb{R}^{m \times d}$ such that $L(\vec{\theta}) = 0$ and $\|\vec{\theta}\|^2 < 2$.

Let
$\emptyset \subsetneq K \subsetneq [d]$,
$\vec{0} \neq \vec{p} \in \cone\{\vec{\chi}_k \:\vert\; k \in K\}$, and
$\vec{0} \neq \vec{q} \in \cone\{\vec{\chi}_k \:\vert\; k \notin K\}$
be such that
$\cos \angle(\vec{p}, \vec{q}) >
 \sin \angle(\vec{p}, \vec{v}^*)$.
We have $\overline{\vec{p}} = \sum_{k \in K} b_k \vec{\chi}_k$ for some $b_k \geq 0$,
and $\overline{\vec{q}} = \sum_{k \notin K} c_k \vec{\chi}_k$ for some $c_k \geq 0$.
Since $\lspn\{\vec{\chi}_1, \ldots, \vec{\chi}_d\} = \mathbb{R}^d$,
we have $\cos \angle(\vec{p}, \vec{q}) < 1$.

\subparagraph{Case $\angle(\vec{p}, \vec{v}^*) \leq \pi / 2$.}

Then
$\cos \angle(\vec{p}, \vec{v}^*) >
 \sin \angle(\vec{p}, \vec{q})$.

Let
$\xi \coloneqq
 \min\bigl\{
 \min\{y_k / b_k \,\,\vert\:\, k \in K \wedge\, b_k \neq 0\},
 \cos \angle(\vec{p}, \vec{v}^*) -
 \sin \angle(\vec{p}, \vec{q})
 \bigr\}$,
$\vec{r} \coloneqq
 \overline{\vec{p}} - \overline{\vec{q}} \,
                      \overline{\vec{q}}^\top \overline{\vec{p}}$,
\begin{align*}
a_1 & \coloneqq
1 &
\vec{w}_1 & \coloneqq
\vec{v}^* - \xi \, \overline{\vec{p}} \\
a_2 & \coloneqq
\sqrt{\xi \|\vec{r}\|} &
\vec{w}_2 & \coloneqq
\sqrt{\xi / \|\vec{r}\|} \, \vec{r} \;,
\end{align*}
and $a_j \coloneqq 0$ and $\vec{w}_j \coloneqq \vec{0}$ for all $j > 2$.

From
\[\vec{r}^\top \vec{x}_i =
  \begin{cases}
  b_i
  & \text{if $i \in K$,} \\
  -c_i \cos \angle(\vec{p}, \vec{q})
  & \text{if $i \notin K$,}
  \end{cases}\]
it follows that $h_{\vec{\theta}}(\vec{x}_i) = y_i$ for all $i \in [d]$, i.e.~$L(\vec{\theta}) = 0$.

We have
\begin{align*}
\|\vec{\theta}\|^2
& =    a_1^2 + \|\vec{w}_1\|^2 + a_2^2 + \|\vec{w}_2\|^2 \\
& =    2 + \xi^2
     - 2 \xi \, \overline{\vec{p}}^\top \vec{v}^*
     + 2 \xi \|\vec{r}\| \\
& =    2 - \xi
     \bigl[2 (\cos \angle(\vec{p}, \vec{v}^*) -
              \sin \angle(\vec{p}, \vec{q}))
         - \xi\bigr] \\
& \leq 2 - \xi^2 \;.
\end{align*}

\subparagraph{Case $\angle(\vec{p}, \vec{v}^*) > \pi / 2$.}

Then
$-\cos \angle(\vec{p}, \vec{v}^*) >
 \sin \angle(\vec{p}, \vec{q})$.

Let
$\xi \coloneqq
 -\cos \angle(\vec{p}, \vec{v}^*) -
 \sin \angle(\vec{p}, \vec{q})$,
$\vec{r} \coloneqq
 \overline{\vec{p}} - \overline{\vec{q}} \,
                      \overline{\vec{q}}^\top \overline{\vec{p}}$,
\begin{align*}
a_1 & \coloneqq
1 &
\vec{w}_1 & \coloneqq
\vec{v}^* + \xi \, \overline{\vec{p}} \\
a_2 & \coloneqq
-\sqrt{\xi \|\vec{r}\|} &
\vec{w}_2 & \coloneqq
\sqrt{\xi / \|\vec{r}\|} \, \vec{r} \;,
\end{align*}
and $a_j \coloneqq 0$ and $\vec{w}_j \coloneqq \vec{0}$ for all $j > 2$.

From
\[\vec{r}^\top \vec{x}_i =
  \begin{cases}
  b_i
  & \text{if $i \in K$,} \\
  -c_i \cos \angle(\vec{p}, \vec{q})
  & \text{if $i \notin K$,}
  \end{cases}\]
it follows that $h_{\vec{\theta}}(\vec{x}_i) = y_i$ for all $i \in [d]$, i.e.~$L(\vec{\theta}) = 0$.

We have
\begin{align*}
\|\vec{\theta}\|^2
& = a_1^2 + \|\vec{w}_1\|^2 + a_2^2 + \|\vec{w}_2\|^2 \\
& = 2 + \xi^2
  + 2 \xi \, \overline{\vec{p}}^\top \vec{v}^*
  + 2 \xi \|\vec{r}\| \\
& = 2 - \xi
  \bigl[2 (-\cos \angle(\vec{p}, \vec{v}^*) -
           \sin \angle(\vec{p}, \vec{q}))
      - \xi\bigr] \\
& = 2 - \xi^2 \;.
\qedhere
\end{align*}
\end{proof}

\begin{example}
Now we present two families of examples of a teacher neuron and training points that respectively satisfy: $\mathcal{M} < 0$ for any $d > 1$, and $\mathcal{M} > 0$ for any $d > 2$.

Let $\{\vec{e}_i\}_{i = 1}^d$ denote the standard basis of~$\mathbb{R}^d$.

\subparagraph{$\mathcal{M} < 0$.}

Let $\xi \in (0, 1)$ and consider, for all $i \in [d]$, vectors
\[\vec{x}_i \coloneqq
  \!\left(\!1 - \frac{d - 1}{d} (1 - \xi)\!\right)\! \vec{e}_i
+ \frac{1 - \xi}{d} \sum_{k \neq i} \vec{e}_k \;.\]
Take $\vec{s} \coloneqq (1, \ldots, 1) \in \mathbb{R}^d$ and $\vec{v}^* \coloneqq \overline{\vec{s}}$.

It can be checked that, for vectors $\vec{\chi}_1, \ldots, \vec{\chi}_d$ defined by
\[\vec{\chi}_k \coloneqq
  \frac{1}{\xi} \!\left(\!\vec{e}_k - \frac{1 - \xi}{d} \vec{s}\!\right)\! \;,\]
we have $[\vec{\chi}_1, \ldots, \vec{\chi}_d]^\top = \vec{X}^{-1}$.

Notice that whenever $k \neq i$ we have $\vec{\chi}_k^\top \vec{\chi}_i = \frac{1}{\xi^2} \!\left(\!-2 \frac{1 - \xi}{d} + {\!\left(\frac{1 - \xi}{d}\!\right)\!}^2 d\right)\! = \frac{1 - \xi}{\xi^2 d} (-2 + 1 - \xi) < 0$.  Hence for all $\emptyset \subsetneq K \subsetneq [d]$, all $\vec{0} \neq \vec{p} \in \cone\{\vec{\chi}_k \:\vert\; k \in K\}$, and all $\vec{0} \neq \vec{q} \in \cone\{\vec{\chi}_i \:\vert\; i \notin K\}$ we have $\cos \angle(\vec{p}, \vec{q}) < 0$.  Thus $\mathcal{M} < 0$.

It remains to verify that $\angle(\vec{v}^*, \vec{x}_i) < \pi / 4$ for all~$i$.  Indeed \begin{align*}
\|\vec{x}_i\|^2
& = {\!\left(1 - \tfrac{d - 1}{d} (1 - \xi)\right)\!}^2
  + (d - 1) {\left(\tfrac{1 - \xi}{d}\right)\!}^2 \\
& = {\!\left(\tfrac{1}{d} + \xi \!\left(1 - \tfrac{1}{d}\right)\right)\!}^2
  + \tfrac{d - 1}{d^2} - \tfrac{d - 1}{d^2} 2 \xi + \tfrac{d - 1}{d^2} \xi^2 \\
& = \tfrac{1}{d^2}
  + 2 \xi \tfrac{1}{d} \!\left(1 - \tfrac{1}{d}\right)\!
  + {\!\left(1 - \tfrac{1}{d}\right)\!}^2 \xi^2
  + \tfrac{d - 1}{d^2} - \tfrac{d - 1}{d^2} 2 \xi + \tfrac{d - 1}{d^2} \xi^2 \\
& = \tfrac{1}{d} + \tfrac{(d - 1)^2 + (d - 1)}{d^2} \xi^2 \\
& = \tfrac{1}{d} + \tfrac{d - 1}{d} \xi^2 \;,
\end{align*}
so in particular $\|\vec{s}\|^2 \, \|\vec{x}_i\|^2 = 1 + (d - 1) \xi^2$.  Therefore
\begin{align*}
\cos \angle(\vec{v}^*, \vec{x}_i)
& = \frac{\vec{s}^\top \vec{x}_i}{\|\vec{s}\| \|\vec{x}_i\|} \\
& > \frac{\!\left(1 - \tfrac{d - 1}{d} (1 - \xi)\right)\! + (d - 1) \tfrac{1 - \xi}{d}}
         {1 + \tfrac{1}{2} (d - 1) \xi^2} \\
& = \frac{1}{1 + \tfrac{1}{2} (d - 1) \xi^2} \;,
\end{align*}
so it suffices to take $\xi \leq \sqrt{\frac{2 (\sqrt{2} - 1)}{d - 1}}$.

\subparagraph{$\mathcal{M} > 0$.}

For $d > 2$, let $b \geq 11$, and consider the data points
\begin{align*}
\vec{x}_1 & \coloneqq b \, \vec{e}_1 \\
\vec{x}_2 & \coloneqq b \, \vec{e}_1 - \sqrt{b} \, \vec{e}_2 + \vec{e}_3 \\
\vec{x}_3 & \coloneqq b \, \vec{e}_1 + \sqrt{b} \, \vec{e}_2 + \vec{e}_3 \\
\vec{x}_i & \coloneqq b \, \vec{e}_1 + \vec{e}_i
\quad\text{for all } 4 \leq i \leq d
\end{align*}
and the teacher neuron
$\vec{v}^* \coloneqq \frac{4}{5} \vec{e}_1 + \frac{3}{5} \vec{e}_3$.

For all~$i$ we have
\[\cos \angle(\vec{v}^*, \vec{x}_i)
>    \frac{4 b}{5 \sqrt{b^2 + b + 1}}
>    \frac{4}{5} \frac{b}{b + 1}
\geq \frac{4}{5} \frac{11}{12}
=    \frac{11}{15}
>    \frac{1}{\sqrt{2}} \;.\]

Straightforward calculation shows that, for $[\vec{\chi}_1, \ldots, \vec{\chi}_d]^\top \coloneqq \vec{X}^{-1}$, we have
\begin{align*}
\vec{\chi}_2 & =
         - \tfrac{1}{2 \sqrt{b}} \vec{e}_2
         + \tfrac{1}{2} \vec{e}_3 \\
\vec{\chi}_3 & =
\phantom{-}\tfrac{1}{2 \sqrt{b}} \vec{e}_2
         + \tfrac{1}{2} \vec{e}_3
\end{align*}
and hence
\begin{align*}
\cos \angle(\vec{\chi}_2, \vec{\chi}_3)
- \sin \angle(\vec{\chi}_2, \vec{v}^*)
& =    \frac{\frac{1}{4} - \frac{1}{4 b}}{\frac{1}{4} + \frac{1}{4 b}}
     - \sqrt{1 - \frac{9}{100 \!\left(\frac{1}{4} + \frac{1}{4 b}\right)\!}} \\
& =    \frac{b - 1}{b + 1}
     - \sqrt{1 - \frac{9}{25} \frac{b}{b + 1}} \\
& \geq \frac{5}{6}
     - \sqrt{\frac{67}{100}} \\
& >    0 \;.
\qedhere
\end{align*}
\end{example}

\begin{remark}
\begin{enumerate}[(i),itemsep=0ex,leftmargin=3em]
\item
For any $\vec{\theta}$ such that $L(\vec{\theta}) = 0$, we have $\|\vec{\theta}\|^2 \geq 2 \, h_{\vec{\theta}}(\vec{x}_1) / \|\vec{x}_1\| = 2 \cos \angle(\vec{v}^*, \vec{x}_1) > \sqrt{2}$.
\item
For $d = 2$, since $\angle(\vec{x}_1, \vec{x}_2) < \pi / 2$, we have $\angle(\vec{\chi}_1, \vec{\chi}_2) > \pi / 2$, so necessarily $\mathcal{M} < 0$.
\item
As its proof above shows, Theorem~\ref{th:min} remains true if we relax the correlation between the teacher neuron and the training points to $\angle(\vec{v}^*, \vec{x}_i) < \pi / 2$ for all~$i$.
\end{enumerate}
\end{remark}

\section{Additional information about the experiments}

For both the centred and the uncentred schemes of generating the training dataset (defined in \autoref{s:exp}), we train a one-hidden layer ReLU network by gradient descent with learning rate~$0.01$, from a balanced initialisation such that $\vec{z}_j \!\iid \mathcal{N}(\vec{0}, \frac{1}{d \, m} \vec{I}_d)$ and $s_j \!\iid \mathcal{U}\{\pm 1\}$, for a range of initialisation scales~$\lambda$, and for several combinations of input dimensions~$d$ and network widths~$m$.

The plots in \autoref{f:w}, which extends \autoref{f:exp} in the main, are obtained by varying the input dimension as $d = 4, 16, 64, 256, 1024$ while keeping the network width at $m = 200$.  The plots in \autoref{f:d} are obtained with input dimension $d = 1024$ by varying the network width as $m = 25, 50, 200$.  For all twelve plots, we vary the initialisation scale as $\lambda = 4^2, 4^1, \ldots, 4^{-12}, 4^{-13}$, and we train the network until the number of iterations reaches~$2 \cdot 10^7$ or the loss drops below~$10^{-9}$.  The plots are in line with \autoref{th:bias}, showing how the three different proxies of rank decrease as $\lambda$~decreases.

\autoref{f:loss.uncentred} complements \autoref{f:loss.centred} in the main, illustrating exponential convergence of the training loss (cf.~\autoref{th:2}) and reduction of the outside distribution test loss as $\lambda$~decreases, for the uncentred scheme of generating the training dataset.

The medians plotted in \autoref{f:w} and \autoref{f:d}, as well as the corresponding standard deviations, can be found in Tables~\ref{t:w-top-left}--\ref{t:w-bottom-right} and Tables~\ref{t:d-top-left}--\ref{t:d-bottom-right} respectively.

The experiments were run using Python~3.10.4 and Pytorch~1.12.1 with CUDA~11.7 on a cluster utilising Intel Xeon Platinum~8268 processors. Some experiments for dimension~$1024$ also used NVIDIA RTX~6000 GPUs. The time taken per iteration greatly depends on the dimension and the width. For dimension~$1024$ and width~$200$, about $300$~iterations per second could be performed on the CPU. The GPU was about $20\%$ faster in this setting. The total number of iterations performed for dimension~$1024$ was about $1.6$~billion. Experiments for lower dimensions or smaller widths are less demanding.

Overall, these numerical results correspond to our theoretical predictions, and suggest that the training dynamics and the implicit bias we established theoretically occur in practical settings in which some of our assumptions are relaxed.

\begin{figure}
\centering
\includegraphics{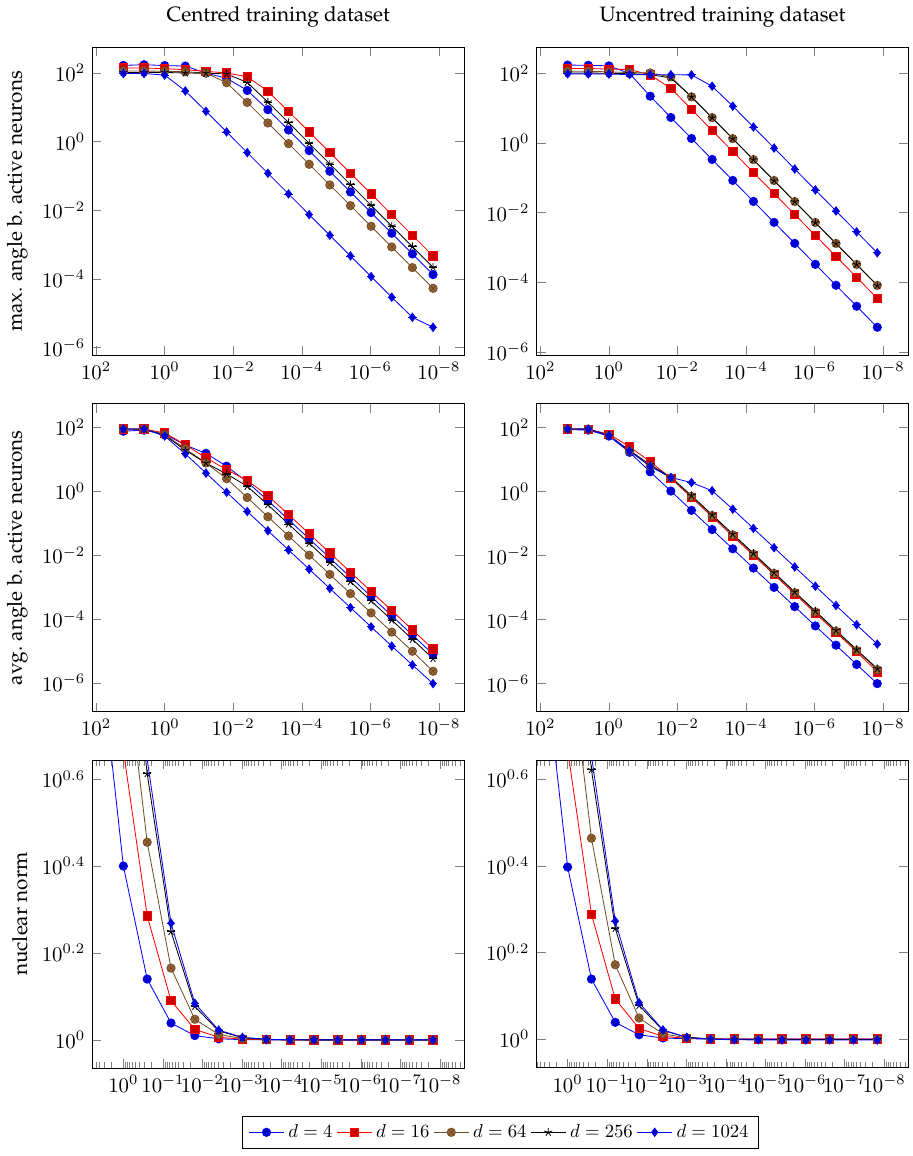}
\caption{Dependence of the maximum angle between active hidden neurons, of the average angle between active hidden neurons, and of the nuclear norm of the hidden-layer weights on the initialisation scale~$\lambda$, for the two generation schemes of the training dataset, the five different input dimensions, and network width~$200$, at the end of the training.  Both axes are logarithmic, and each point plotted shows the median over five trials.}
\label{f:w}
\end{figure}

\newcommand{\foreachw}[2]{
\foreach \w in {25,50,200} {
\addplot+ [error bars/.cd, y dir=none, y explicit] table [
  col sep=comma,
  x expr={and(\thisrow{init scale}!=9.31322574615478E-10,
                  \thisrow{init scale}!=3.72529029846191E-09)==1?
          \thisrow{init scale}:nan},
  y=#2-median-\w,
  y error=#2-std-\w
] {table_d_#1.csv}; }
}

\begin{figure}
\centering
\includegraphics{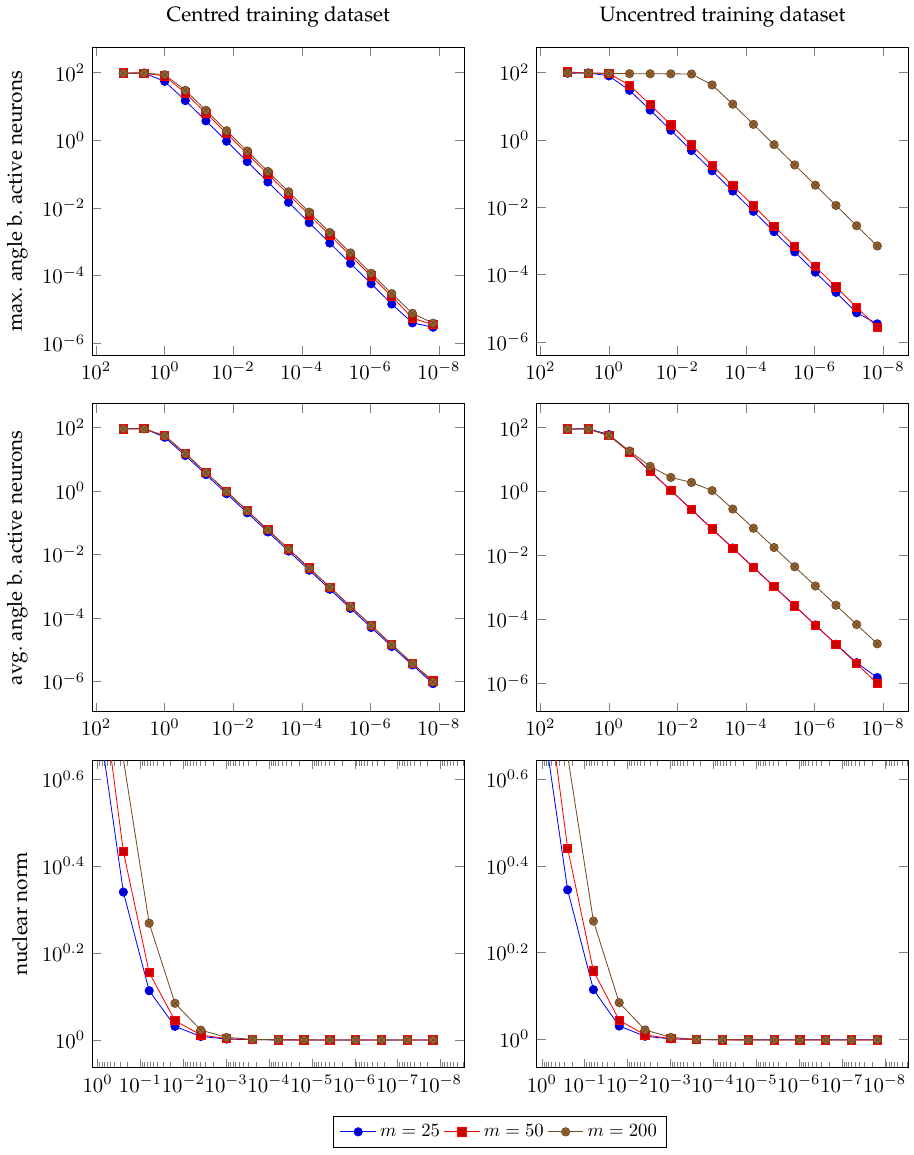}
\caption{Dependence of the maximum angle between active hidden neurons, of the average angle between active hidden neurons, and of the nuclear norm of the hidden-layer weights on the initialisation scale~$\lambda$, for the two generation schemes of the training dataset, the three different network widths, and input dimension~$1024$, at the end of the training.  Both axes are logarithmic, and each point plotted shows the median over five trials.}
\label{f:d}
\end{figure}

\begin{figure}
\centering
\includegraphics{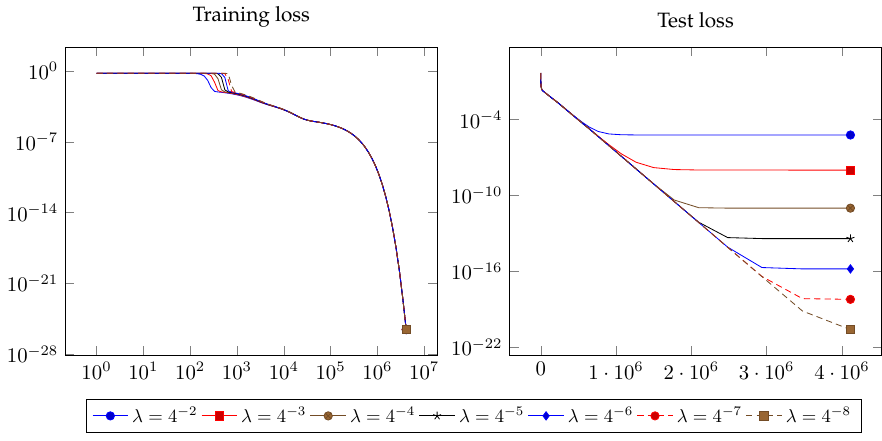}
\caption{Evolution of the training loss, and of an outside distribution test loss, during training for an example uncentred training dataset in dimension~$16$ and with $m=25$.  The horizontal axes show iterations; they are logarithmic for the training loss, and linear for the test loss.  The vertical axes are logarithmic.}
\label{f:loss.uncentred}
\end{figure}

\clearpage

\newlength{\wsep}
\setlength{\wsep}{1.5em}
\setlength{\tabcolsep}{.3em}

\sisetup{
round-mode=places,
round-precision=3,
round-pad=true,
exponent-mode=scientific,
exponent-product=\cdot,
print-zero-exponent=true}
\begin{table}
\caption{The medians over five trials plotted in \autoref{f:w} on the top left, with the standard deviations shown in parentheses, both rounded to four-digit mantissas.}
\label{t:w-top-left}\vspace{3ex}
\centering
\small
\begin{tabular}{l@{\hspace{\wsep}}ll@{\hspace{\wsep}}ll@{\hspace{\wsep}}ll} \toprule
$\lambda$ & $d=4$ & & $d=16$ & & $d=64$ & \\ \midrule
$4^{2}$ & $\num{166.47702327254854}$ & ($\num{10.445029732228928}$) & $\num{140.4779991452489}$ & ($\num{3.256486843444842}$) & $\num{113.62454273870704}$ & ($\num{1.4661950328257751}$) \\
$4^{1}$ & $\num{174.30794345807178}$ & ($\num{3.576954334089988}$) & $\num{139.87213755416485}$ & ($\num{3.8175630953203448}$) & $\num{113.96847165224048}$ & ($\num{1.2524195623461023}$) \\
$4^{0}$ & $\num{164.39533394159554}$ & ($\num{7.950705965644424}$) & $\num{133.61356893009403}$ & ($\num{2.044738352367734}$) & $\num{112.21229402955112}$ & ($\num{3.772133974405759}$) \\
$4^{-1}$ & $\num{159.1852918182555}$ & ($\num{26.407463225199148}$) & $\num{124.6009028572661}$ & ($\num{8.588880923934063}$) & $\num{108.9111747315418}$ & ($\num{23.594849183766375}$) \\
$4^{-2}$ & $\num{99.43094283186625}$ & ($\num{53.93312055238522}$) & $\num{111.4250693939442}$ & ($\num{18.9219585672569}$) & $\num{99.7627384922606}$ & ($\num{38.99682465682472}$) \\
$4^{-3}$ & $\num{70.3748031402944}$ & ($\num{59.23528239972}$) & $\num{100.25766624470148}$ & ($\num{38.96593067652093}$) & $\num{52.04661950024217}$ & ($\num{42.63013623094306}$) \\
$4^{-4}$ & $\num{31.05623128800411}$ & ($\num{61.12586773383981}$) & $\num{77.97062040973525}$ & ($\num{47.00196562164631}$) & $\num{13.90969731460497}$ & ($\num{48.76563124471318}$) \\
\end{tabular}
\begin{tabular}{l@{\hspace{\wsep}}ll@{\hspace{\wsep}}ll} \toprule
$\lambda$ & $d=256$ & & $d=1024$ & \\ \midrule
$4^{2}$ & $\num{102.97194051129902}$ & ($\num{0.7356611943810162}$) & $\num{96.54077293581696}$ & ($\num{0.2816062460339826}$) \\
$4^{1}$ & $\num{102.8276016775762}$ & ($\num{0.6488885488254467}$) & $\num{96.49657556000317}$ & ($\num{0.4394451257734195}$) \\
$4^{0}$ & $\num{103.72227286191472}$ & ($\num{1.0916793477637845}$) & $\num{86.94646114211247}$ & ($\num{5.651242316313105}$) \\
$4^{-1}$ & $\num{99.90273339093063}$ & ($\num{4.215807261914331}$) & $\num{29.968730376828624}$ & ($\num{10.0291850581342}$) \\
$4^{-2}$ & $\num{98.40662811822038}$ & ($\num{28.3654002148452}$) & $\num{7.662027478569353}$ & ($\num{2.819612945154185}$) \\
$4^{-3}$ & $\num{94.04192068388107}$ & ($\num{42.515244366457296}$) & $\num{1.918415386592474}$ & ($\num{0.7109637222527215}$) \\
$4^{-4}$ & $\num{53.03816717779589}$ & ($\num{45.23717409721861}$) & $\num{0.4796169288563967}$ & ($\num{0.1778677493892327}$) \\
\bottomrule
\end{tabular}
\end{table}
\sisetup{
round-mode=places,
round-precision=3,
round-pad=true,
exponent-mode=scientific,
exponent-product=\cdot,
print-zero-exponent=true}
\begin{table}
\caption{The medians over five trials plotted in \autoref{f:w} on the top right, with the standard deviations shown in parentheses, both rounded to four-digit mantissas.}
\label{t:w-top-right}\vspace{3ex}
\centering
\small
\begin{tabular}{l@{\hspace{\wsep}}ll@{\hspace{\wsep}}ll@{\hspace{\wsep}}ll} \toprule
$\lambda$ & $d=4$ & & $d=16$ & & $d=64$ & \\ \midrule
$4^{2}$ & $\num{172.83382694564852}$ & ($\num{8.40243862470096}$) & $\num{137.1559367578413}$ & ($\num{4.366183821903727}$) & $\num{115.40887403635016}$ & ($\num{1.4538600754363955}$) \\
$4^{1}$ & $\num{170.44995147066976}$ & ($\num{5.072294044953224}$) & $\num{138.10298163753626}$ & ($\num{3.1954467254392176}$) & $\num{115.2101607858866}$ & ($\num{1.5641400389420446}$) \\
$4^{0}$ & $\num{164.09882019339017}$ & ($\num{26.86817396793964}$) & $\num{132.41584587418134}$ & ($\num{6.701533086882289}$) & $\num{110.11955494358018}$ & ($\num{3.60748042873431}$) \\
$4^{-1}$ & $\num{99.16378081004255}$ & ($\num{66.01218205211093}$) & $\num{128.69872955814677}$ & ($\num{19.696788258092223}$) & $\num{105.61226689106812}$ & ($\num{8.201239256389918}$) \\
$4^{-2}$ & $\num{22.102109787850463}$ & ($\num{67.12265073661897}$) & $\num{90.46380691946834}$ & ($\num{28.852295309225763}$) & $\num{102.83438957641924}$ & ($\num{34.210698366387454}$) \\
$4^{-3}$ & $\num{5.42062522666801}$ & ($\num{66.9713003348158}$) & $\num{37.26530213221832}$ & ($\num{43.45620290935182}$) & $\num{79.43709462964061}$ & ($\num{41.058953492010325}$) \\
$4^{-4}$ & $\num{1.349607111158829}$ & ($\num{66.09480640021387}$) & $\num{9.249927486391536}$ & ($\num{50.93892232835058}$) & $\num{21.294048600501828}$ & ($\num{21.719970536545414}$) \\
\end{tabular}
\begin{tabular}{l@{\hspace{\wsep}}ll@{\hspace{\wsep}}ll} \toprule
$\lambda$ & $d=256$ & & $d=1024$ & \\ \midrule
$4^{2}$ & $\num{102.45055186305244}$ & ($\num{1.6762175966124304}$) & $\num{96.93348742540414}$ & ($\num{0.1512932716696836}$) \\
$4^{1}$ & $\num{102.17121141248325}$ & ($\num{1.1224546819893435}$) & $\num{96.23001368079224}$ & ($\num{0.4889204905562679}$) \\
$4^{0}$ & $\num{100.38783290029872}$ & ($\num{2.333019736986943}$) & $\num{95.07341194961587}$ & ($\num{0.6628303303805583}$) \\
$4^{-1}$ & $\num{97.99727065848025}$ & ($\num{12.142683141316176}$) & $\num{92.53136839588883}$ & ($\num{6.296318740705577}$) \\
$4^{-2}$ & $\num{92.12712802235627}$ & ($\num{34.335651733066165}$) & $\num{91.42959616178452}$ & ($\num{28.295105655943345}$) \\
$4^{-3}$ & $\num{74.03589574876166}$ & ($\num{41.39675340467375}$) & $\num{91.13184724154104}$ & ($\num{44.01408116235844}$) \\
$4^{-4}$ & $\num{21.01682654552712}$ & ($\num{38.175412617432976}$) & $\num{90.31208212640666}$ & ($\num{48.36871585865931}$) \\
\bottomrule
\end{tabular}
\end{table}
\sisetup{
round-mode=places,
round-precision=3,
round-pad=true,
exponent-mode=scientific,
exponent-product=\cdot,
print-zero-exponent=true}
\begin{table}
\caption{The medians over five trials plotted in \autoref{f:w} on the middle left, with the standard deviations shown in parentheses, both rounded to four-digit mantissas.}
\label{t:w-middle-left}\vspace{3ex}
\centering
\small
\begin{tabular}{l@{\hspace{\wsep}}ll@{\hspace{\wsep}}ll@{\hspace{\wsep}}ll} \toprule
$\lambda$ & $d=4$ & & $d=16$ & & $d=64$ & \\ \midrule
$4^{2}$ & $\num{76.83132661617012}$ & ($\num{9.817340798474}$) & $\num{90.15674643784249}$ & ($\num{0.5402344111271852}$) & $\num{90.00925318423293}$ & ($\num{0.1066706890061668}$) \\
$4^{1}$ & $\num{83.1073876130784}$ & ($\num{3.064012258683029}$) & $\num{89.78289181626587}$ & ($\num{0.343674978485003}$) & $\num{89.54459116984181}$ & ($\num{0.14286854117769}$) \\
$4^{0}$ & $\num{61.96004197892498}$ & ($\num{8.026331676125235}$) & $\num{67.79270820945786}$ & ($\num{3.165558377773201}$) & $\num{60.11706202349859}$ & ($\num{1.436483819807606}$) \\
$4^{-1}$ & $\num{28.31413546285967}$ & ($\num{8.054003012000225}$) & $\num{29.16790052987867}$ & ($\num{2.550483674742705}$) & $\num{21.699056867104805}$ & ($\num{2.630636844944759}$) \\
$4^{-2}$ & $\num{15.554473968494973}$ & ($\num{5.9222228286923775}$) & $\num{11.248674621044032}$ & ($\num{1.7268792907374002}$) & $\num{7.7691941367161395}$ & ($\num{2.284933535051491}$) \\
$4^{-3}$ & $\num{6.157081544961352}$ & ($\num{4.770377088054219}$) & $\num{4.808436379639381}$ & ($\num{1.412155416168899}$) & $\num{2.508912256916044}$ & ($\num{1.7188432307588457}$) \\
$4^{-4}$ & $\num{1.9657303029659643}$ & ($\num{3.648623093537536}$) & $\num{2.203229384276756}$ & ($\num{1.188157158656623}$) & $\num{0.6430029255297299}$ & ($\num{1.3125193592723263}$) \\
\end{tabular}
\begin{tabular}{l@{\hspace{\wsep}}ll@{\hspace{\wsep}}ll} \toprule
$\lambda$ & $d=256$ & & $d=1024$ & \\ \midrule
$4^{2}$ & $\num{90.11093836742134}$ & ($\num{0.0371607533393562}$) & $\num{90.1258476766475}$ & ($\num{0.0156346519362206}$) \\
$4^{1}$ & $\num{89.64722105050036}$ & ($\num{0.0750594408796089}$) & $\num{89.60074725539084}$ & ($\num{0.0985664419296974}$) \\
$4^{0}$ & $\num{57.60219016841489}$ & ($\num{1.8330800804441605}$) & $\num{54.27379633981375}$ & ($\num{1.1269832699850109}$) \\
$4^{-1}$ & $\num{19.52341865228195}$ & ($\num{2.336792178424912}$) & $\num{14.941601434919567}$ & ($\num{0.4246025202023312}$) \\
$4^{-2}$ & $\num{7.717791571762702}$ & ($\num{2.3522789765513146}$) & $\num{3.758367690486108}$ & ($\num{0.1109652596786216}$) \\
$4^{-3}$ & $\num{3.5199911254689327}$ & ($\num{2.1408745252308363}$) & $\num{0.9399637498248246}$ & ($\num{0.0278376340286798}$) \\
$4^{-4}$ & $\num{1.4499000084457447}$ & ($\num{1.7303485160339298}$) & $\num{0.2349975664747688}$ & ($\num{0.0069616652118935}$) \\
\bottomrule
\end{tabular}
\end{table}
\sisetup{
round-mode=places,
round-precision=3,
round-pad=true,
exponent-mode=scientific,
exponent-product=\cdot,
print-zero-exponent=true}
\begin{table}
\caption{The medians over five trials plotted in \autoref{f:w} on the middle right, with the standard deviations shown in parentheses, both rounded to four-digit mantissas.}
\label{t:w-middle-right}\vspace{3ex}
\centering
\small
\begin{tabular}{l@{\hspace{\wsep}}ll@{\hspace{\wsep}}ll@{\hspace{\wsep}}ll} \toprule
$\lambda$ & $d=4$ & & $d=16$ & & $d=64$ & \\ \midrule
$4^{2}$ & $\num{86.92639759398051}$ & ($\num{6.307255068716143}$) & $\num{90.06560427154069}$ & ($\num{0.1911647010044565}$) & $\num{90.20049965335248}$ & ($\num{0.1050766979415226}$) \\
$4^{1}$ & $\num{83.32131189076851}$ & ($\num{3.5784293967108405}$) & $\num{89.18645006742962}$ & ($\num{0.3285530861140784}$) & $\num{89.65029693965013}$ & ($\num{0.1668707889166789}$) \\
$4^{0}$ & $\num{54.27046589256961}$ & ($\num{13.945112370175226}$) & $\num{62.69792843444077}$ & ($\num{2.96743290164555}$) & $\num{57.55166369396056}$ & ($\num{2.2591912328713946}$) \\
$4^{-1}$ & $\num{16.899080292054084}$ & ($\num{18.78526667736841}$) & $\num{25.28116772333673}$ & ($\num{3.053620895288682}$) & $\num{18.34377473180483}$ & ($\num{2.898650196980795}$) \\
$4^{-2}$ & $\num{4.13836648799324}$ & ($\num{17.15082425471492}$) & $\num{8.631960398261286}$ & ($\num{1.6205165673306745}$) & $\num{6.325580118046817}$ & ($\num{2.342892011718917}$) \\
$4^{-3}$ & $\num{1.0406954897390368}$ & ($\num{14.707358179273076}$) & $\num{2.620560255527115}$ & ($\num{1.2056271689796108}$) & $\num{2.697468663630926}$ & ($\num{1.608444140806064}$) \\
$4^{-4}$ & $\num{0.2605717176464858}$ & ($\num{13.869477848529431}$) & $\num{0.6546920065516261}$ & ($\num{1.2576763680983132}$) & $\num{0.7072095836768977}$ & ($\num{0.6715472388850328}$) \\
\end{tabular}
\begin{tabular}{l@{\hspace{\wsep}}ll@{\hspace{\wsep}}ll} \toprule
$\lambda$ & $d=256$ & & $d=1024$ & \\ \midrule
$4^{2}$ & $\num{90.07727295698902}$ & ($\num{0.0508474690702683}$) & $\num{90.11726412495062}$ & ($\num{0.0270835076494818}$) \\
$4^{1}$ & $\num{89.6519740568585}$ & ($\num{0.0934360235740389}$) & $\num{89.73075584699784}$ & ($\num{0.0838394804907562}$) \\
$4^{0}$ & $\num{56.97700389357853}$ & ($\num{1.6859210616887177}$) & $\num{56.449683418660655}$ & ($\num{1.837349809393564}$) \\
$4^{-1}$ & $\num{19.34856361281548}$ & ($\num{3.0073405664370685}$) & $\num{18.549599057475778}$ & ($\num{2.0994328065499177}$) \\
$4^{-2}$ & $\num{6.90285738289776}$ & ($\num{2.8536807970517537}$) & $\num{6.110694624725481}$ & ($\num{2.3812936234060524}$) \\
$4^{-3}$ & $\num{2.819560904949839}$ & ($\num{1.8585219504692685}$) & $\num{2.7657454246816457}$ & ($\num{2.317647622223396}$) \\
$4^{-4}$ & $\num{0.7529542066684543}$ & ($\num{1.5207916285180707}$) & $\num{1.9215933981206097}$ & ($\num{2.149054599976675}$) \\
\bottomrule
\end{tabular}
\end{table}
\sisetup{
round-mode=places,
round-precision=3,
round-pad=true,
exponent-mode=scientific,
exponent-product=\cdot,
print-zero-exponent=true}
\begin{table}
\caption{The medians over five trials plotted in \autoref{f:w} on the bottom left, with the standard deviations shown in parentheses, both rounded to four-digit mantissas.}
\label{t:w-bottom-left}\vspace{3ex}
\centering
\small
\begin{tabular}{l@{\hspace{\wsep}}ll@{\hspace{\wsep}}ll@{\hspace{\wsep}}ll} \toprule
$\lambda$ & $d=4$ & & $d=16$ & & $d=64$ & \\ \midrule
$4^{2}$ & $\num{32.87858596064641}$ & ($\num{1.4118686901314952}$) & $\num{63.13858829944506}$ & ($\num{0.5679790352075759}$) & $\num{121.6086794809763}$ & ($\num{1.2957188262237302}$) \\
$4^{1}$ & $\num{7.962005931394954}$ & ($\num{0.2011508080715666}$) & $\num{15.881988961334969}$ & ($\num{0.1276533405577884}$) & $\num{30.56328194062708}$ & ($\num{0.3182154796778044}$) \\
$4^{0}$ & $\num{2.511814976339884}$ & ($\num{0.0749100622270681}$) & $\num{4.628464700830729}$ & ($\num{0.0492169592011731}$) & $\num{8.341564610150137}$ & ($\num{0.0788703821776683}$) \\
$4^{-1}$ & $\num{1.380187461522007}$ & ($\num{0.0200452668216777}$) & $\num{1.9292851059111704}$ & ($\num{0.0172324813869293}$) & $\num{2.849091441247715}$ & ($\num{0.0205539608896905}$) \\
$4^{-2}$ & $\num{1.0936235795413498}$ & ($\num{0.0072938192558342}$) & $\num{1.2307635592732216}$ & ($\num{0.0076938671873727}$) & $\num{1.4626790596718784}$ & ($\num{0.0053702764231742}$) \\
$4^{-3}$ & $\num{1.0230469678653902}$ & ($\num{0.002199803870309}$) & $\num{1.0571308715278642}$ & ($\num{0.0010748472646882}$) & $\num{1.1155690815281525}$ & ($\num{0.0013409697798027}$) \\
$4^{-4}$ & $\num{1.0057373735666597}$ & ($\num{0.0009144773710166}$) & $\num{1.014203480780022}$ & ($\num{0.0005712029228878}$) & $\num{1.0288270566136597}$ & ($\num{0.0003567512310775}$) \\
\end{tabular}
\begin{tabular}{l@{\hspace{\wsep}}ll@{\hspace{\wsep}}ll} \toprule
$\lambda$ & $d=256$ & & $d=1024$ & \\ \midrule
$4^{2}$ & $\num{200.71261581993755}$ & ($\num{0.7875505783462655}$) & $\num{220.53510905783412}$ & ($\num{0.2389379497649081}$) \\
$4^{1}$ & $\num{50.38379483830828}$ & ($\num{0.2023939321588538}$) & $\num{55.34492436282002}$ & ($\num{0.0510300935768091}$) \\
$4^{0}$ & $\num{13.335575305601392}$ & ($\num{0.0500166796792792}$) & $\num{14.60887431099044}$ & ($\num{0.0109403174275352}$) \\
$4^{-1}$ & $\num{4.098262777921313}$ & ($\num{0.0125274003438473}$) & $\num{4.420620426517898}$ & ($\num{0.0041023325665847}$) \\
$4^{-2}$ & $\num{1.7755155732820396}$ & ($\num{0.003138205645438}$) & $\num{1.8563275835140007}$ & ($\num{0.0011320896127661}$) \\
$4^{-3}$ & $\num{1.193867343094184}$ & ($\num{0.0007825164254873}$) & $\num{1.2141134628954409}$ & ($\num{0.0002890741908316}$) \\
$4^{-4}$ & $\num{1.048428098218026}$ & ($\num{0.0001922709542839}$) & $\num{1.053500587371317}$ & ($\num{7.163907592131925e-05}$) \\
\bottomrule
\end{tabular}
\end{table}
\sisetup{
round-mode=places,
round-precision=3,
round-pad=true,
exponent-mode=scientific,
exponent-product=\cdot,
print-zero-exponent=true}
\begin{table}
\caption{The medians over five trials plotted in \autoref{f:w} on the bottom right, with the standard deviations shown in parentheses, both rounded to four-digit mantissas.}
\label{t:w-bottom-right}\vspace{3ex}
\centering
\small
\begin{tabular}{l@{\hspace{\wsep}}ll@{\hspace{\wsep}}ll@{\hspace{\wsep}}ll} \toprule
$\lambda$ & $d=4$ & & $d=16$ & & $d=64$ & \\ \midrule
$4^{2}$ & $\num{32.14860141865184}$ & ($\num{1.3545293718292917}$) & $\num{63.1010165800693}$ & ($\num{0.5531174667249836}$) & $\num{122.77185695504964}$ & ($\num{0.7578138895885803}$) \\
$4^{1}$ & $\num{7.710350786109528}$ & ($\num{0.2874900075344113}$) & $\num{15.85240718451917}$ & ($\num{0.1381159436488722}$) & $\num{30.812307305058763}$ & ($\num{0.1893721582684458}$) \\
$4^{0}$ & $\num{2.4947396986097656}$ & ($\num{0.061910143033689}$) & $\num{4.658554656340824}$ & ($\num{0.0514075126053004}$) & $\num{8.493011393014568}$ & ($\num{0.0422846782656073}$) \\
$4^{-1}$ & $\num{1.3751332763884372}$ & ($\num{0.0295987807562428}$) & $\num{1.941064068080958}$ & ($\num{0.0250914966149209}$) & $\num{2.9071352882124204}$ & ($\num{0.0330901320668915}$) \\
$4^{-2}$ & $\num{1.0926331920252876}$ & ($\num{0.0041725294638866}$) & $\num{1.236689294911108}$ & ($\num{0.0087986465086302}$) & $\num{1.4837469677668085}$ & ($\num{0.0056473708150418}$) \\
$4^{-3}$ & $\num{1.0225620466753538}$ & ($\num{0.0010055683740369}$) & $\num{1.057594101559302}$ & ($\num{0.0014313149156774}$) & $\num{1.1175955923685756}$ & ($\num{0.0049617931343736}$) \\
$4^{-4}$ & $\num{1.005390779564762}$ & ($\num{0.000522612456913}$) & $\num{1.0142217670844045}$ & ($\num{0.0008382403737151}$) & $\num{1.0280852586569216}$ & ($\num{0.0050804348518423}$) \\
\end{tabular}
\begin{tabular}{l@{\hspace{\wsep}}ll@{\hspace{\wsep}}ll} \toprule
$\lambda$ & $d=256$ & & $d=1024$ & \\ \midrule
$4^{2}$ & $\num{199.9328088544393}$ & ($\num{0.967403390885134}$) & $\num{220.2870536478345}$ & ($\num{0.208104443923833}$) \\
$4^{1}$ & $\num{50.19033777690935}$ & ($\num{0.2478675300968638}$) & $\num{55.29740745347625}$ & ($\num{0.0543448381911283}$) \\
$4^{0}$ & $\num{13.399460512495748}$ & ($\num{0.0728682968609342}$) & $\num{14.688928683650117}$ & ($\num{0.0245997944820843}$) \\
$4^{-1}$ & $\num{4.185170165281864}$ & ($\num{0.0380510122485693}$) & $\num{4.507471829631518}$ & ($\num{0.0218931729091235}$) \\
$4^{-2}$ & $\num{1.8004588371140886}$ & ($\num{0.0127289669376699}$) & $\num{1.8717704316723367}$ & ($\num{0.0056611197133084}$) \\
$4^{-3}$ & $\num{1.1947585280108046}$ & ($\num{0.0023200513473539}$) & $\num{1.213023684485996}$ & ($\num{0.0005792345389632}$) \\
$4^{-4}$ & $\num{1.0469251852729315}$ & ($\num{0.0008068099507485}$) & $\num{1.049415267921469}$ & ($\num{0.0007904455109653}$) \\
\bottomrule
\end{tabular}
\end{table}

\newlength{\dsep}
\setlength{\dsep}{1.5em}
\setlength{\tabcolsep}{.3em}

\sisetup{
round-mode=places,
round-precision=3,
round-pad=true,
exponent-mode=scientific,
exponent-product=\cdot,
print-zero-exponent=true}
\begin{table}
\caption{The medians over five trials plotted in \autoref{f:d} on the top left, with the standard deviations shown in parentheses, both rounded to four-digit mantissas.}
\label{t:d-top-left}\vspace{3ex}
\centering
\small
\begin{tabular}{l@{\hspace{\dsep}}ll@{\hspace{\dsep}}ll@{\hspace{\dsep}}ll} \toprule
$\lambda$ & $m=25$ & & $m=50$ & & $m=200$ & \\ \midrule
$4^{2}$ & $\num{95.7485816268635}$ & ($\num{1.0974705691956368}$) & $\num{97.17676764278352}$ & ($\num{1.0863323008354795}$) & $\num{96.54077293581696}$ & ($\num{0.2816062460339826}$) \\
$4^{1}$ & $\num{97.13382682852269}$ & ($\num{1.8505467596554737}$) & $\num{96.043809922143}$ & ($\num{0.7676437653181986}$) & $\num{96.49657556000317}$ & ($\num{0.4394451257734195}$) \\
$4^{0}$ & $\num{55.23806853492907}$ & ($\num{18.692964439988312}$) & $\num{79.69471362864941}$ & ($\num{12.72075532053608}$) & $\num{86.94646114211247}$ & ($\num{5.651242316313105}$) \\
$4^{-1}$ & $\num{14.905444141196302}$ & ($\num{8.640953257600906}$) & $\num{25.01032183261542}$ & ($\num{18.305167504614964}$) & $\num{29.968730376828624}$ & ($\num{10.0291850581342}$) \\
$4^{-2}$ & $\num{3.7455253627857337}$ & ($\num{2.266669650309409}$) & $\num{6.35108502381482}$ & ($\num{5.362735353473715}$) & $\num{7.662027478569353}$ & ($\num{2.819612945154185}$) \\
$4^{-3}$ & $\num{0.936674643981734}$ & ($\num{0.5692671743002072}$) & $\num{1.5895348582115305}$ & ($\num{1.360604966257251}$) & $\num{1.918415386592474}$ & ($\num{0.7109637222527215}$) \\
$4^{-4}$ & $\num{0.2341727808470942}$ & ($\num{0.1423682770109511}$) & $\num{0.3974116771489795}$ & ($\num{0.3405017191707278}$) & $\num{0.4796169288563967}$ & ($\num{0.1778677493892327}$) \\
\bottomrule
\end{tabular}
\end{table}
\sisetup{
round-mode=places,
round-precision=3,
round-pad=true,
exponent-mode=scientific,
exponent-product=\cdot,
print-zero-exponent=true}
\begin{table}
\caption{The medians over five trials plotted in \autoref{f:d} on the top right, with the standard deviations shown in parentheses, both rounded to four-digit mantissas.}
\label{t:d-top-right}\vspace{3ex}
\centering
\small
\begin{tabular}{l@{\hspace{\dsep}}ll@{\hspace{\dsep}}ll@{\hspace{\dsep}}ll} \toprule
$\lambda$ & $m=25$ & & $m=50$ & & $m=200$ & \\ \midrule
$4^{2}$ & $\num{95.28797799974502}$ & ($\num{3.500844423466204}$) & $\num{103.63749224360645}$ & ($\num{9.43494643635476}$) & $\num{96.93348742540414}$ & ($\num{0.1512932716696836}$) \\
$4^{1}$ & $\num{97.29736897571416}$ & ($\num{1.0431222242335965}$) & $\num{96.0744936430012}$ & ($\num{0.9753286420519268}$) & $\num{96.23001368079224}$ & ($\num{0.4889204905562679}$) \\
$4^{0}$ & $\num{79.20707127209823}$ & ($\num{15.10070352107431}$) & $\num{93.46770117661195}$ & ($\num{7.892326558163569}$) & $\num{95.07341194961587}$ & ($\num{0.6628303303805583}$) \\
$4^{-1}$ & $\num{29.49319584321992}$ & ($\num{38.9060469175496}$) & $\num{41.838174205526016}$ & ($\num{31.326111694847445}$) & $\num{92.53136839588883}$ & ($\num{6.296318740705577}$) \\
$4^{-2}$ & $\num{7.645015181941102}$ & ($\num{39.786346000225535}$) & $\num{11.139028715788}$ & ($\num{29.113907082660432}$) & $\num{91.42959616178452}$ & ($\num{28.295105655943345}$) \\
$4^{-3}$ & $\num{1.9174198683963208}$ & ($\num{40.27933733876404}$) & $\num{2.7983813217520828}$ & ($\num{11.39002953676953}$) & $\num{91.13184724154104}$ & ($\num{44.01408116235844}$) \\
$4^{-4}$ & $\num{0.4795716432183913}$ & ($\num{41.48423916065044}$) & $\num{0.6999033851071742}$ & ($\num{2.956722095134165}$) & $\num{90.31208212640666}$ & ($\num{48.36871585865931}$) \\
\bottomrule
\end{tabular}
\end{table}
\sisetup{
round-mode=places,
round-precision=3,
round-pad=true,
exponent-mode=scientific,
exponent-product=\cdot,
print-zero-exponent=true}
\begin{table}
\caption{The medians over five trials plotted in \autoref{f:d} on the middle left, with the standard deviations shown in parentheses, both rounded to four-digit mantissas.}
\label{t:d-middle-left}\vspace{3ex}
\centering
\small
\begin{tabular}{l@{\hspace{\dsep}}ll@{\hspace{\dsep}}ll@{\hspace{\dsep}}ll} \toprule
$\lambda$ & $m=25$ & & $m=50$ & & $m=200$ & \\ \midrule
$4^{2}$ & $\num{90.12954046139846}$ & ($\num{0.2193101424925883}$) & $\num{90.18605609262092}$ & ($\num{0.0588748293977232}$) & $\num{90.1258476766475}$ & ($\num{0.0156346519362206}$) \\
$4^{1}$ & $\num{92.92190761038472}$ & ($\num{0.5591564263619232}$) & $\num{90.81490633968056}$ & ($\num{0.06436821400396}$) & $\num{89.60074725539084}$ & ($\num{0.0985664419296974}$) \\
$4^{0}$ & $\num{48.98331947987131}$ & ($\num{3.902621845304928}$) & $\num{55.329831979338536}$ & ($\num{2.030490494425289}$) & $\num{54.27379633981375}$ & ($\num{1.1269832699850109}$) \\
$4^{-1}$ & $\num{13.009550814881733}$ & ($\num{1.6709094577876964}$) & $\num{15.213519255516756}$ & ($\num{1.6718956845017678}$) & $\num{14.941601434919567}$ & ($\num{0.4246025202023312}$) \\
$4^{-2}$ & $\num{3.2647874620756143}$ & ($\num{0.4356155534252004}$) & $\num{3.826075303253432}$ & ($\num{0.4738710344568028}$) & $\num{3.758367690486108}$ & ($\num{0.1109652596786216}$) \\
$4^{-3}$ & $\num{0.8163994698065886}$ & ($\num{0.1093346107420026}$) & $\num{0.9568799169074936}$ & ($\num{0.1199075772604664}$) & $\num{0.9399637498248246}$ & ($\num{0.0278376340286798}$) \\
$4^{-4}$ & $\num{0.2041031572793811}$ & ($\num{0.0273421336695934}$) & $\num{0.239226471593743}$ & ($\num{0.0300025215384142}$) & $\num{0.2349975664747688}$ & ($\num{0.0069616652118935}$) \\
\bottomrule
\end{tabular}
\end{table}
\sisetup{
round-mode=places,
round-precision=3,
round-pad=true,
exponent-mode=scientific,
exponent-product=\cdot,
print-zero-exponent=true}
\begin{table}
\caption{The medians over five trials plotted in \autoref{f:d} on the middle right, with the standard deviations shown in parentheses, both rounded to four-digit mantissas.}
\label{t:d-middle-right}\vspace{3ex}
\centering
\small
\begin{tabular}{l@{\hspace{\dsep}}ll@{\hspace{\dsep}}ll@{\hspace{\dsep}}ll} \toprule
$\lambda$ & $m=25$ & & $m=50$ & & $m=200$ & \\ \midrule
$4^{2}$ & $\num{90.0018427387892}$ & ($\num{0.275174888576079}$) & $\num{90.22592442689304}$ & ($\num{0.208829726585622}$) & $\num{90.11726412495062}$ & ($\num{0.0270835076494818}$) \\
$4^{1}$ & $\num{92.19838208679036}$ & ($\num{0.443842995700443}$) & $\num{90.80765964483798}$ & ($\num{0.067290462625198}$) & $\num{89.73075584699784}$ & ($\num{0.0838394804907562}$) \\
$4^{0}$ & $\num{60.4030772864804}$ & ($\num{5.290642722932022}$) & $\num{55.75856259103461}$ & ($\num{2.49488655606391}$) & $\num{56.449683418660655}$ & ($\num{1.837349809393564}$) \\
$4^{-1}$ & $\num{17.291186024525057}$ & ($\num{9.636860781442616}$) & $\num{17.00947623598286}$ & ($\num{3.628527008390166}$) & $\num{18.549599057475778}$ & ($\num{2.0994328065499177}$) \\
$4^{-2}$ & $\num{4.364169912385957}$ & ($\num{6.058903337946872}$) & $\num{4.305886414426823}$ & ($\num{3.047203568036127}$) & $\num{6.110694624725481}$ & ($\num{2.3812936234060524}$) \\
$4^{-3}$ & $\num{1.091918324827922}$ & ($\num{4.760759218059402}$) & $\num{1.077516782493486}$ & ($\num{1.1743496944367784}$) & $\num{2.7657454246816457}$ & ($\num{2.317647622223396}$) \\
$4^{-4}$ & $\num{0.2730045685169296}$ & ($\num{4.860104202134356}$) & $\num{0.2694006431332851}$ & ($\num{0.304495611254773}$) & $\num{1.9215933981206097}$ & ($\num{2.149054599976675}$) \\
\bottomrule
\end{tabular}
\end{table}
\sisetup{
round-mode=places,
round-precision=3,
round-pad=true,
exponent-mode=scientific,
exponent-product=\cdot,
print-zero-exponent=true}
\begin{table}
\caption{The medians over five trials plotted in \autoref{f:d} on the bottom left, with the standard deviations shown in parentheses, both rounded to four-digit mantissas.}
\label{t:d-bottom-left}\vspace{3ex}
\centering
\small
\begin{tabular}{l@{\hspace{\dsep}}ll@{\hspace{\dsep}}ll@{\hspace{\dsep}}ll} \toprule
$\lambda$ & $m=25$ & & $m=50$ & & $m=200$ & \\ \midrule
$4^{2}$ & $\num{78.23224537885686}$ & ($\num{0.2779646701835433}$) & $\num{111.82827732350796}$ & ($\num{0.4887620546346026}$) & $\num{220.53510905783412}$ & ($\num{0.2389379497649081}$) \\
$4^{1}$ & $\num{19.74088564938595}$ & ($\num{0.0840115551448909}$) & $\num{28.03018313365765}$ & ($\num{0.1203937129806467}$) & $\num{55.34492436282002}$ & ($\num{0.0510300935768091}$) \\
$4^{0}$ & $\num{5.695088502020861}$ & ($\num{0.0162855856090612}$) & $\num{7.770575590436834}$ & ($\num{0.025919432751148}$) & $\num{14.60887431099044}$ & ($\num{0.0109403174275352}$) \\
$4^{-1}$ & $\num{2.1885883255657803}$ & ($\num{0.0042164520480554}$) & $\num{2.711951600410454}$ & ($\num{0.0069113399050982}$) & $\num{4.420620426517898}$ & ($\num{0.0041023325665847}$) \\
$4^{-2}$ & $\num{1.298026877940666}$ & ($\num{0.0010425706791508}$) & $\num{1.4292664994652668}$ & ($\num{0.001780121927483}$) & $\num{1.8563275835140007}$ & ($\num{0.0011320896127661}$) \\
$4^{-3}$ & $\num{1.0745258116168506}$ & ($\num{0.0002594659582517}$) & $\num{1.1073569849890024}$ & ($\num{0.0004487105763886}$) & $\num{1.2141134628954409}$ & ($\num{0.0002890741908316}$) \\
$4^{-4}$ & $\num{1.018604123203792}$ & ($\num{6.385866516974998e-05}$) & $\num{1.0268112067644162}$ & ($\num{0.000112169221458}$) & $\num{1.053500587371317}$ & ($\num{7.163907592131925e-05}$) \\
\bottomrule
\end{tabular}
\end{table}
\sisetup{
round-mode=places,
round-precision=3,
round-pad=true,
exponent-mode=scientific,
exponent-product=\cdot,
print-zero-exponent=true}
\begin{table}
\caption{The medians over five trials plotted in \autoref{f:d} on the bottom right, with the standard deviations shown in parentheses, both rounded to four-digit mantissas.}
\label{t:d-bottom-right}\vspace{3ex}
\centering
\small
\begin{tabular}{l@{\hspace{\dsep}}ll@{\hspace{\dsep}}ll@{\hspace{\dsep}}ll} \toprule
$\lambda$ & $m=25$ & & $m=50$ & & $m=200$ & \\ \midrule
$4^{2}$ & $\num{77.89981161387547}$ & ($\num{0.4055080977939265}$) & $\num{111.42021544104423}$ & ($\num{0.1843704098896348}$) & $\num{220.2870536478345}$ & ($\num{0.208104443923833}$) \\
$4^{1}$ & $\num{19.59951271542564}$ & ($\num{0.093652122678881}$) & $\num{27.99456713540297}$ & ($\num{0.0512160503797194}$) & $\num{55.29740745347625}$ & ($\num{0.0543448381911283}$) \\
$4^{0}$ & $\num{5.686692662463214}$ & ($\num{0.0239755482489022}$) & $\num{7.818364752031973}$ & ($\num{0.0209171101707492}$) & $\num{14.688928683650117}$ & ($\num{0.0245997944820843}$) \\
$4^{-1}$ & $\num{2.2094616075395845}$ & ($\num{0.0064573396619471}$) & $\num{2.757760175261293}$ & ($\num{0.0112715924974419}$) & $\num{4.507471829631518}$ & ($\num{0.0218931729091235}$) \\
$4^{-2}$ & $\num{1.299745364626865}$ & ($\num{0.0034295090096349}$) & $\num{1.4356308149908166}$ & ($\num{0.0027351753595372}$) & $\num{1.8717704316723367}$ & ($\num{0.0056611197133084}$) \\
$4^{-3}$ & $\num{1.0711328314027115}$ & ($\num{0.000984214946783}$) & $\num{1.1045571706532875}$ & ($\num{0.0006527857304464}$) & $\num{1.213023684485996}$ & ($\num{0.0005792345389632}$) \\
$4^{-4}$ & $\num{1.0147819805902072}$ & ($\num{0.0006092758834032}$) & $\num{1.0220843428974764}$ & ($\num{0.0006193638411725}$) & $\num{1.049415267921469}$ & ($\num{0.0007904455109653}$) \\
\bottomrule
\end{tabular}
\end{table}

\clearpage

\section{Further experiments}
\label{app:further}

Here we report on experiments in which we explore the effects of adding a second teacher neuron whose direction is opposite to that of the first, and of increasing the scale~$\rho$ of the noise used to generate the synthetic datasets (cf.~\autoref{s:exp}) so that quickly most of the data points exceed the $\pi / 4$ angle with their corresponding teacher neuron.

In \autoref{fig:various-rho}, the growing maximum angles between neurons at the end of the training indicate that we no longer have a single (or one per teacher neuron) aligned bundle of neurons forming and sticking together for the rest of the training.

In the bottom two plots of \autoref{fig:various-rho} and in \autoref{fig:one-run}, for small scales~$\rho$ (where the smallest values are such that the angles between the data points and the corresponding teacher neuron concentrate around $\pi / 4$), the phenomena we identified theoretically still seem to hold, where the training passes near a second saddle point as we outlined in \autoref{s.concl}.

\begin{figure}
\centering
\includegraphics{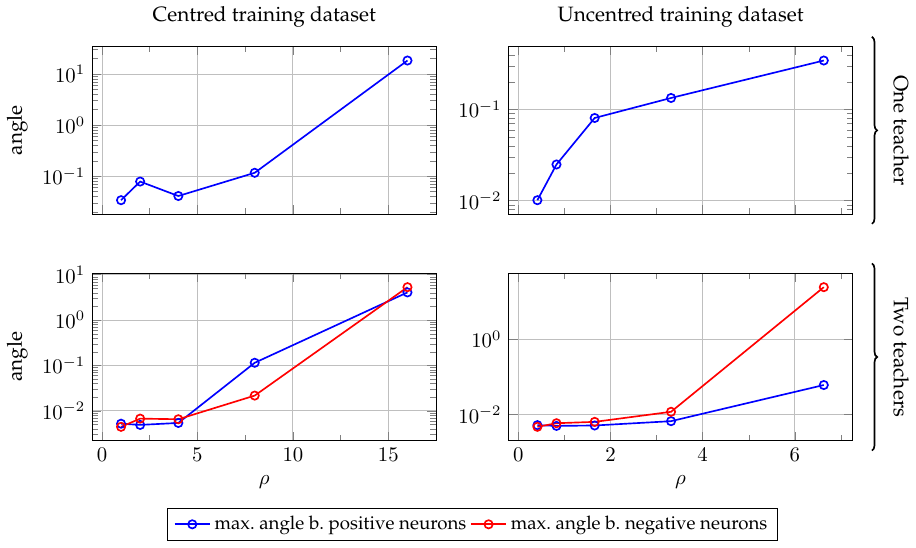}
\caption{The maximum angle between hidden neurons that start with a positive (blue) and negative (red) inner product with the first teacher neuron. The first teacher neuron has norm~$1$ and the second teacher neuron has norm~$3$. The vertical axes are logarithmic and the angles are in degrees. The horizontal axes show different multipliers~$\rho$ for the variance of the distribution of the data points (cf.~\autoref{s:exp}). The input dimension is $d=16$ and, for each teacher neuron, we sample $d$~data points from the distribution specified in the main. Each point in the plot shows the median over $15$~trials of the angle in degrees at the end of training. The training runs for $2\cdot 10^7$ iterations or until the loss reaches $10^{-9}$. The width of the network is $m=25$, and the initialisation scale is $\lambda=4^{-7}$.}
\label{fig:various-rho}
\end{figure}

\begin{figure}
\centering
\includegraphics{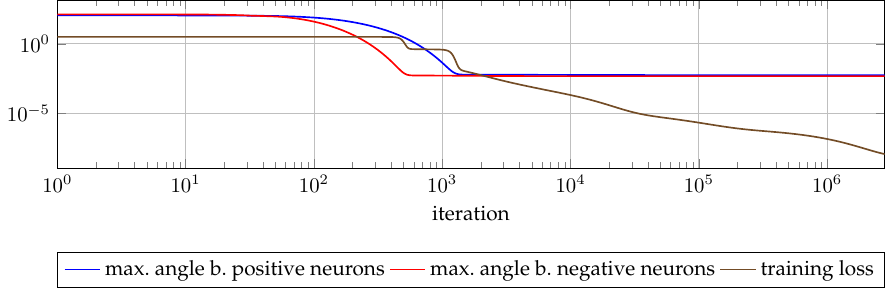}
\caption{The evolution of the training loss and the maximum angle between positive and negative hidden neurons during the training in dimension~$16$. The vertical axes are logarithmic and the angles are in degrees. This is one example of a run contributing to \autoref{fig:various-rho}. Specifically, in this run the training dataset is uncentered and $\rho=\sqrt{2}-1$. The two fast drops in loss (after passing of the first and then the second saddle point) coincide with the times at which the respective group of hidden neurons aligns.}
\label{fig:one-run}
\end{figure}
\else\fi
\end{document}